\documentclass[a4 paper,12pt]{article}
\usepackage[margin=1in,footskip=0.5in]{geometry}
\usepackage{amsmath,amssymb,amsthm}
\usepackage{bbm}
\usepackage{enumitem}
\usepackage{amsfonts}
\usepackage{authblk}
\usepackage{tikz}
\usetikzlibrary{positioning}
\usepackage{graphicx}
\usetikzlibrary{decorations.pathmorphing}
\usetikzlibrary{patterns}
\usetikzlibrary{patterns.meta}
\usepackage{subcaption}
\usepackage{booktabs} 
\usepackage{tabularx} 
\usepackage{makecell}
\usepackage{threeparttable}
\usepackage{hyperref}

\newtheorem{theorem}{Theorem}
\newtheorem{lemma}{Lemma}
\newtheorem{proposition}{Proposition}

\newtheorem{remark}{Remark}

\newtheorem*{reproposition1}{Proposition 1}
\newtheorem*{reproposition2}{Proposition 2}
\newtheorem*{retheorem1}{Theorem 1}
\newtheorem*{retheorem}{Theorem 2}

\allowdisplaybreaks[4]

\title{Approximation of Analytic Functions by ReLU Neural Networks with Adjustable Depth and Width}
\author[a]{Yanming Lai\thanks{Corresponding Author (yanming.lai@polyu.edu.hk)}}
\author[b]{Defeng Sun}
\author[c]{Yang Wang}
\affil[a,b]{Department of Applied Mathematics, The Hong Kong Polytechnic University, Hung Hom, Hong Kong, China}
\affil[c]{Department of Mathematics, The University of Hong Kong, Pokfulam, Hong Kong, China}
\date{}

\begin{document}

\maketitle

\begin{abstract}

In contrast to most studies on neural network approximation theory that characterize results through a single parameter, such as the total number of network parameters, \cite{shen2020deep} pioneered the characterization of approximation rates as a joint function of the width parameter $N$ and the depth parameter $L$, thereby granting greater architectural flexibility. Existing works using the $(N,L)$-characterization focus on function classes with finite smoothness $s$, establishing a typical approximation rate of $\mathcal{O}\left(N^{-2s/d}L^{-2s/d}\right)$ with $d$ denoting the input dimension, which indicates that network depth and width play symmetric roles for these classes. In contrast, this paper establishes upper bounds for the approximation of analytic functions, which possess infinite smoothness, via ReLU networks under the $(N,L)$-characterization. Specifically, we derive approximation rates of $\mathcal{O}\left(N^{-C L^{\tau}}\right)$, where $C>0$ is some constant and $\tau>0$ is a parameter influenced by the relation between $L$ and $N$. In particular, $\tau=1$ if $N$ scales roughly as $L^d$. Our findings reveal that depth plays a more critical role than width in the context of analytic function approximation.
The main technical difficulty of obtaining such upper bounds lies in the trade-off between the smoothness parameters and the approximation accuracy. To overcome this difficulty, we employ refined constructions of several ReLU networks to approximate power functions, multivariate multiplication, and polynomials, which may be of independent interest.

\end{abstract}

\section{Introduction}\label{Introduction}

\subsection{Background}
Over the past decade, neural networks have achieved spectacular success across a vast array of practical applications, ranging from computer vision and natural language processing to scientific computing. This empirical triumph is largely attributed to their extraordinary capacity to represent complex, high-dimensional target functions from data. To unravel the underlying mathematical mechanisms driving this success, a rigorous analysis of their structural properties is highly demanded.
Among various neural network architectures, the foundational model is the feedforward neural network (FNN). A typical FNN $\boldsymbol{g}:\mathbb{R}^d\to\mathbb{R}^{n_{out}}$ is mathematically formulated as:
\begin{align}
\boldsymbol{g}_0(\boldsymbol{x}) &= \boldsymbol{x}, \nonumber\\
\boldsymbol{g}_{\ell+1}(\boldsymbol{x}) &= \sigma(\boldsymbol{A}_\ell \boldsymbol{g}_\ell(\boldsymbol{x}) + \boldsymbol{b}_\ell), \quad \ell = 0, 1, \dots, \bar{L}-1, \nonumber\\
\boldsymbol{g}(\boldsymbol{x}) &= \boldsymbol{A}_{\bar{L}} \boldsymbol{g}_{\bar{L}}(\boldsymbol{x}) + \boldsymbol{b}_{\bar{L}},\nonumber
\end{align}
where $\boldsymbol{A}_\ell\in\mathbb{R}^{n_{\ell+1}\times n_{\ell}},\boldsymbol{b}_{\ell}\in\mathbb{R}^{n_{\ell+1}}$ with $n_{0}=d,n_{\bar{L}+1}=n_{out}$ denote the weight matrices and bias vectors, respectively. The numbers $\bar{N}:=\max_{\ell\in\{1, \dots, \bar{L}\}} n_{\ell}$ and $\bar{L}$ are called the width and depth of $\boldsymbol{g}$, respectively. The function $\sigma$ is called the activation of $\boldsymbol{g}$, which acts on vectors componentwise. In this work, we will focus on the widely used rectified linear unit (ReLU)  activation $\sigma_R(x):=\max\{x,0\}$, which is the simplest piecewise linear function.

The study on the approximation capabilities of FNNs dates back to the late 1980s and 1990s. During that era, research primarily focused on two-layer networks with smooth activation functions. Representative works from this period include \cite{cybenko1989approximation,hornik1991approximation,barron1993universal,mhaskar1993approximation,leshno1993multilayer,mhaskar1996neural,pinkus1999approximation} and the references therein, with the last being a comprehensive review paper. In the late 2010s, driven by the rapid development of machine learning, the approximation theory of neural networks witnessed a major resurgence of interest. The scope of research has since expanded significantly to encompass both shallow and deep architectures with a diverse array of activation functions, among which ReLU networks have received particular attention. Landmark contributions in this modern wave include \cite{yarotsky2017error,yarotsky2018optimal,liang2017why,lu2017expressive,petersen2018optimal,montanelli2019new,bolcskei2019optimal,guhring2020error,guhring2021approximation,devore2021neural,blanchard2021shallow,jiao2023deep,siegel2020approximation,siegel2022high,siegel2023optimal,mao2023rates,suzuki2021deep,suzukiadaptivity,yang2025optimalb,jiao2023approximation} and the references therein. Notably, all the aforementioned works characterize their approximation results through a single parameter, such as the target approximation accuracy, the total number of network parameters, or the network width (in the case of two-layer networks).

In contrast to the aforementioned literature, \cite{shen2020deep} pioneered a novel paradigm: the authors creatively proposed characterizing the network approximation rate as a joint function of both the width parameter $N$ and the depth parameter $L$, which are the two most representative parameters of a network architecture. Leveraging a technique known as ``bit extraction", they proved that for the Hölder space $C^s([0,1]^d)$ with smoothness $s\in(0,1]$, networks of width $\mathcal{O}(N)$ and depth $\mathcal{O}(L)$ can achieve an approximation rate of $\mathcal{O}\left(N^{-2/d}L^{-2/d}\right)$. Compared to the single-parameter approach, the primary advantage of employing the $(N,L)$-characterization is that it grants greater architectural flexibility; researchers can tune the network configuration within a certain range to accommodate other practical constraints. Furthermore, under this characterization, the approximation rate explicitly illuminates the respective roles of depth and width. On the other hand, due to the involvement of two parameters, this approach often demands more intricate constructive techniques.

\begin{table}[htbp]
\centering
\begin{threeparttable}
\caption{Comparison of works using the $(N,L)$-characterization.}
\label{related works}
\small

\renewcommand{\tabularxcolumn}[1]{m{#1}} 
\newcolumntype{Y}{>{\centering\arraybackslash}X}

\begin{tabularx}{\textwidth}{c >{\hsize=1.0\hsize}Y >{\hsize=0.6\hsize}Y !{\quad} >{\hsize=1.4\hsize}Y} 
\toprule
\textbf{Reference} & \textbf{Function Class} & \textbf{Norm} & \textbf{Approximation Rate} \\
\midrule
\multicolumn{4}{c}{\textbf{ReLU Activation}}\\
\addlinespace
\cite{shen2020deep,shen2022optimal} & H\"older space $C^{s}$, $s\in(0,1]$ & $L^p$, $p\in[1,\infty]$ & $\mathcal{O}\left(N^{-2/d}L^{-2/d}\right) $ \\
\addlinespace 
\cite{lu2021deep}&H\"older space $C^{s}$, $s\in\mathbb{N}_{\geq1}$&$L^{\infty}$& $\mathcal{O}\left(N^{-2s/d}L^{-2s/d}\right)$\\
\addlinespace
\cite{hon2022simultaneous}&H\"older space $C^{s}$, $s\in\mathbb{N}_{\geq2}$&$W^{1,p}$, $p\in[1,\infty)$&$\mathcal{O}\left(N^{-2(s-1)/d}L^{-2(s-1)/d}\right)$\\
\addlinespace
\cite{yang2025optimal}&
    \makecell[c]{$\bullet$ Sobolev space $W^{s,q}$,\\$s\in(0,\infty),q\in[1,\infty]$\\
    $\bullet$ Besov space $B_{q,r}^s$,\\
    $s\in(0,\infty),r,q\in[1,\infty]$}
&\makecell[c]{$L^{p}$, $p\in[1,\infty]$,\\$1/p>1/q-s/d$}&$\mathcal{O}\left(N^{-2s/d}L^{-2s/d}\right)$\\
\addlinespace
\cite{yang2023nearly}&Sobolev space $W^{s,\infty}$,$s\in[2,\infty)$&$W^{1,\infty}$&$\mathcal{O}\left(N^{-2(s-1)/d}L^{-2(s-1)/d}\right)$\\
\addlinespace
\cite{yang2024near}&Korobov space $X^{2,\infty}$&\makecell[c]{$L^p$, $p\in[1,\infty]$\\$H^1$}&\makecell[c]{$\widetilde{\mathcal{O}}\left(N^{-4}L^{-4}\right)$\\$\widetilde{\mathcal{O}}\left(N^{-2}L^{-2}\right)$}\\
\addlinespace
\cite{li2025some}&\makecell[c]{Korobov space $X^{s,p}$,\\ $s\in[2,\infty),p\in[1,\infty)$}&\makecell[c]{$L^p$\\$W^{1,p}$}&\makecell[c]{$\widetilde{\mathcal{O}}\left(N^{-2s}L^{-2s}\right)$\\$\widetilde{\mathcal{O}}\left(N^{-2(s-1)}L^{-2(s-1)}\right)$}\\
\addlinespace

\cite{yang2026approximation}&\makecell[c]{$\bullet$ anisotropic Besov \\space\\ $\bullet$ mixed smooth Besov \\space} &$L^p$, $p\in[1,\infty]$&$\widetilde{\mathcal{O}}\left(N^{-2s}L^{-2s}\right)$\tnote{*}\\
\addlinespace
\textbf{Ours} & Analytic function class $\mathcal{A}(\boldsymbol{\rho},M)$ & $L^{\infty}$ & $\mathcal{O}\Big(N^{-C(d,\boldsymbol{\rho},\kappa,\alpha) L^{\tau(\kappa)}}\Big)$ \\
\addlinespace
\multicolumn{4}{c}{\textbf{ReLU$\mathbf{^2}$ Activation}}\\
\addlinespace
\cite{hon2022simultaneous}&H\"older space $C^{s}$, $s\in\mathbb{N}_{\geq2}$&\makecell[c]{$W^{m,p}$,\\ $p\in[1,\infty),m\in[s-1]$}&$\mathcal{O}\left(N^{-2(s-m)/d}L^{-2(s-m)/d}\right)$\\
\addlinespace
\cite{yang2023nearly}&Sobolev space $W^{s,\infty}$,$s\in\mathbb{N}_{\geq3}$&\makecell[c]{$W^{m,\infty}$,\\$m\in\{2,3,\cdots,s-1\}$}&$\mathcal{O}\left(N^{-2(s-m)/d}L^{-2(s-m)/d}\right)$\\
\addlinespace
\multicolumn{4}{c}{\textbf{Other Common Activations}}\\
\addlinespace
\cite{zhang2024deep}&H\"older space $C^{s}$, $s\in\mathbb{N}_{\geq1}$&$L^{\infty}$& $\mathcal{O}\left(N^{-2s/d}L^{-2s/d}\right)$\\
\addlinespace
\cite{yang2025deep}&Sobolev space $W^{s,\infty}$,$s\in\mathbb{N}_{\geq1}$&\makecell[c]{$W^{m,\infty}$,\\$m\in\{0,1,\cdots,s-1\}$}&$\mathcal{O}\left(N^{-2(s-m)/d}L^{-2(s-m)/d}\right)$\\
\bottomrule
\end{tabularx}

\begin{tablenotes}
    \footnotesize
    \item[*] For the mathematical definition of the smoothness parameter $s$ used here, the readers are referred to the original paper \cite{yang2026approximation}.
\end{tablenotes}

\end{threeparttable}
\end{table}

Table \ref{related works} summarizes existing literature that investigates network approximation capacities using the $(N,L)$-characterization, focusing on networks with width at most $\widetilde{\mathcal{O}}(N)$ and depth at most $\widetilde{\mathcal{O}}(L)$. Within this line of research, the majority of works focus on ReLU networks \cite{shen2020deep,lu2021deep,shen2022optimal,hon2022simultaneous,yang2023nearly,yang2024near,yang2025optimal,li2025some}. Among them, \cite{hon2022simultaneous,yang2023nearly} also extended their study to $\text{ReLU}^2$ networks. The defining feature of $\text{ReLU}^2$ networks, distinguishing them from ReLU networks, is that their composition inherits higher smoothness. Consequently, while ReLU networks can at most approximate a target function and its first-order derivative, $\text{ReLU}^2$ networks are capable of simultaneously approximating the target function along with its higher-order derivatives. Beyond ReLU and $\text{ReLU}^2$, other widely used activation functions include some non-piecewise-polynomial functions such as sigmoid, tanh, and softplus. In this context, \cite{zhang2024deep} demonstrated that a ReLU network of depth $L$ and width $N$ can be approximated by a broad class of networks utilizing these alternative activations with depth $2L$ and width $3N$, thereby establishing identical upper bounds for these activations. However, their results cannot be directly extended to settings involving approximation of derivatives. Taking a different route, \cite{yang2025deep} avoided reducing the problem to the ReLU case and instead directly constructed networks with certain activations to achieve (higher-order) derivative approximation.

As illustrated in Table \ref{related works}, among the existing works that employ the $(N,L)$-characterizion, the target functions chosen for approximation all belong to function classes with finite smoothness $s$, such as Hölder, Sobolev, Besov, and Korobov spaces. For these classes, the resulting approximation rates indicate that the network depth and width play symmetric roles. In contrast to these functions classes with finite smoothness, analytic functions possess infinite differentiability. However, research regarding the neural network approximation of analytic functions remains limited \cite{mhaskar1993approximation,mhaskar1996neural,weinan2018exponential,opschoor2022exponential}. Specifically, the former two studies investigate sigmoid networks, whereas the latter two focus on ReLU networks. The approximation rates obtained in these works are all characterized by a single parameter, leaving an absence of bounds under the $(N,L)$-characterization. In this paper, we aim to close this gap and elucidate the distinct roles of neural network depth and width in this setting. The main technical difficulty lies in the trade-off between the smoothness parameter and the approximation accuracy, which does not arise when studying the approximation of function classes with finite smoothness.

\subsection{Main Results}\label{Main Results}
We begin by formalizing the mathematical definitions of the function classes considered in this study. We investigate the approximation of a real analytic function $f(\boldsymbol{x})$ on $[-1,1]^d$, that is, $f(\boldsymbol{x})$ is infinitely differentiable on $[-1,1]^d$ and for any $\boldsymbol{x}_0\in[-1,1]^d$, its Taylor series at $\boldsymbol{x}_0$ converges to itself in a neighborhood of $\boldsymbol{x}_0$.
It is well known that any real analytic function on a bounded subset of $\mathbb{R}^d$ can be holomorphically extended to a complex analytic function on a domain in $\mathbb{C}^d$. By convention, this extended domain is typically chosen as a Bernstein polyellipse of parameter $\boldsymbol{\rho}\in(1,\infty)^d$, which is defined as
\begin{align*}
\mathcal{E}_{\boldsymbol{\rho}} := \prod_{i=1}^d \mathcal{E}_{\rho_i} \subset \mathbb{C}^d.
\end{align*}
Here, each $\mathcal{E}_{\rho_i}$ is a Bernstein ellipse of parameter $\rho_i>1$ in the complex plane, defined as
\begin{align*}
\mathcal{E}_{\rho_i} := \left\{ \frac{z + z^{-1}}{2} \ \middle|\  z \in \mathbb{C}, \, 1 \leq |z| < \rho_i \right\} \subset \mathbb{C}.
\end{align*}
The foci of $\mathcal{E}_{\rho_i}$ are $(\pm1,0)$.  The lengths of its major semi-axis and minor semi-axis are $\frac{\rho_i+\rho_i^{-1}}{2}$ and $\frac{\rho_i-\rho_i^{-1}}{2}$, respectively. 
For $\mathcal{E}_{\boldsymbol{\rho}}$, $\boldsymbol{\rho}$ serves as a metric characterizing the regularity of the function; specifically, a larger component $\rho_i$ indicates that $f$ possesses a higher degree of regularity with respect to the variable $x_i$. As all components of $\boldsymbol{\rho}$ approach 1, $\mathcal{E}_{\boldsymbol{\rho}}$ collapses to $[-1,1]^d$.
Throughout the remainder of this paper, it is a standing assumption that the target function $f$ belongs to the set 
\begin{align*}
\mathcal{A}(\boldsymbol{\rho},M):=\{f\ |\ f:[-1,1]^d\to\mathbb{R}\text{ admits a holomorphic extension to }\mathcal{E}_{\boldsymbol{\rho}},\  \|f\|_{L^\infty(\mathcal{E}_{\boldsymbol{\rho}})}\leq M \}.
\end{align*}
Our main result establishes the upper bound for the approximation of analytic functions by ReLU networks under the $(N,L)$-characterization.

\begin{theorem}\label{approximation}
Suppose $L,N$ are sufficiently large and there exist $\kappa\in\left[0,d\right],\beta>0$ such that
\begin{align*}
L^{\kappa+\alpha}\leq N\leq e^{L^{\beta}},
\end{align*}
where $\alpha>0$ can be arbitrarily small; when $\kappa=0$, $\alpha$ can be $0$. For any $f\in\mathcal{A}(\boldsymbol{\rho},M)$, there exists a neural network $f_{NN}:\mathbb{R}^d\to\mathbb{R}$ with width $C(d,\boldsymbol{\rho})N$ and depth $C(d,\boldsymbol{\rho},\beta)L$ (when $\kappa=d$, a larger depth $C(d,\boldsymbol{\rho},\beta)L(\log L)^2$ is needed) such that for any $\boldsymbol{x}\in\left[-1,1\right]^d$, there holds

\begin{align*}
\left| f(\boldsymbol{x}) - f_{NN}(\boldsymbol{x}) \right|
&\leq C(d,\boldsymbol{\rho},M)N^{-C(d,\boldsymbol{\rho},\kappa,\alpha)L^{\tau(\kappa)}}
\end{align*}
with
\begin{align*}
\tau(\kappa)=\left\{
\begin{matrix}
\frac{\kappa+1}{d+2},&\kappa\in\left[0,\frac{1}{d+1}\right];\\
\kappa,&\kappa\in\left(\frac{1}{d+1},\frac{1}{d}\right);\\
\frac{\kappa+1}{d+1},&\kappa\in\left[\frac{1}{d},d\right].
\end{matrix}
\right.
\end{align*}
\end{theorem}

Through Theorem \ref{approximation}, we observe a phenomenon that stands in contrast to function classes with finite smoothness. Prior works on approximating function classes with finite smoothness indicate that a uniform approximation rate is achieved across different regions on the $L$-$N$ plane. Conversely, Theorem \ref{approximation} reveals that when approximating analytic functions, different regions on the $L$-$N$ plane may exhibit different convergence rates. This discrepancy stems precisely from the aforementioned core challenge of approximating analytic functions: the intrinsic trade-off between the smoothness parameters and the approximation accuracy. Figure \ref{exponent} shows the graph of the exponent $\tau(\kappa)$. A larger $\tau$ corresponds to a faster convergence rate. 
Figure \ref{applicability region} illustrates the applicability region of Theorem \ref{approximation} on the $L$-$N$ plane. When $\kappa=0$ and $\beta$ is sufficiently large, this region covers almost the entire $L$-$N$ plane, leaving only a minor gap uncovered.

\begin{figure}[htbp]
    \centering

    \begin{subfigure}[b]{0.48\textwidth}
    \centering

    \resizebox{\textwidth}{!}{
        \begin{tikzpicture}[>=latex, scale=1.2]

            \draw[->, thick] (-0.5, 0) -- (11, 0) node[right, font=\large] {$\kappa$};
            \draw[->, thick] (0, -0.5) -- (0, 8.5) node[above, font=\large] {$\tau(\kappa)$};
            \node[below left] at (0,0) {$0$};

            \coordinate (A) at (0, 1.5);
            \coordinate (B) at (2.5, 2.5);
            \coordinate (C) at (4.5, 4.5);
            \coordinate (D) at (9.5, 7.5);

            \draw[ultra thick, blue] (A) -- (B) node[midway, above left, font=\small] {};
            \draw[ultra thick, blue] (B) -- (C);
            \draw[ultra thick, blue] (D) -- (C);

            \node[left] at (0, 1.5) {$\frac{1}{d+2}$};
            \draw[fill=blue] (A) circle (2pt);

            \draw[dashed, gray!80, thick] (2.5, 0) -- (B) -- (0, 2.5);
            \node[below, yshift=-2pt] at (2.5, 0) {$\frac{1}{d+1}$};
            \node[left] at (0, 2.5) {$\frac{1}{d+1}$};
            \draw[fill=blue] (B) circle (2pt);

            \draw[dashed, gray!80, thick] (4.5, 0) -- (C) -- (0, 4.5);
            \node[below, yshift=-2pt] at (4.5, 0) {$\frac{1}{d}$};
            \node[left] at (0, 4.5) {$\frac{1}{d}$};
            \draw[fill=blue] (C) circle (2pt);

            \draw[dashed, gray!80, thick] (9.5, 0) -- (D) -- (0, 7.5);
            \node[below, yshift=-2pt] at (9.5, 0) {$d$};
            \node[left] at (0, 7.5) {$1$};
            \draw[fill=blue] (D) circle (2pt);

            \node[blue, above, rotate=15] at (1.25, 2.0) {$\frac{\kappa+1}{d+2}$};
            \node[blue, above, rotate=45] at (3.5, 3.5) {$\kappa$};
            \node[blue, above, rotate=25] at (7.0, 6.0) {$\frac{\kappa+1}{d+1}$};

        \end{tikzpicture}
    }

    \caption{An illustration of $\kappa(\tau)$.}
    \label{exponent}
    \end{subfigure}
    \hfill
    \begin{subfigure}[b]{0.48\textwidth}
    \centering
\resizebox{\textwidth}{!}{
\begin{tikzpicture}[>=latex, scale=1.3]

            \fill[blue!20] (0,0) -- (0,1.2)
                -- plot[domain=0:1.5, samples=100] (\x, {1.2 + 4.3 * (\x/1.5)*(\x/1.5)}) 
                .. controls (4.5, 5.5) and (5.0, 0.9) .. (6.5, 0.9)
                -- plot[domain=6.5:0, samples=100] (\x, {0.9 * sqrt(\x/6.5)}) 
                -- cycle;

            \draw[->, ultra thick] (0, 0) -- (7.5, 0) node[below right, font=\large] {$L$};
            \draw[->, ultra thick] (0, 0) -- (0, 6.5) node[left, font=\large] {$N$};
            \node[below left, font=\large] at (0,0) {$O$};

            \draw[ultra thick, black, domain=0:1.5, samples=100] plot (\x, {1.2 + 4.3 * (\x/1.5)*(\x/1.5)}) node[above, font=\Large] {$e^{L^\beta}$};
            \draw[ultra thick, black, domain=0:6.5, samples=100] plot (\x, {0.9 * sqrt(\x/6.5)}) node[right, font=\Large] {$L^\kappa$};

        \end{tikzpicture}
    }
    \caption{Applicability region on the $L$-$N$ plane.}
    \label{applicability region}
\end{subfigure}
    \caption{Visualization of Theorem \ref{approximation}}
\end{figure}
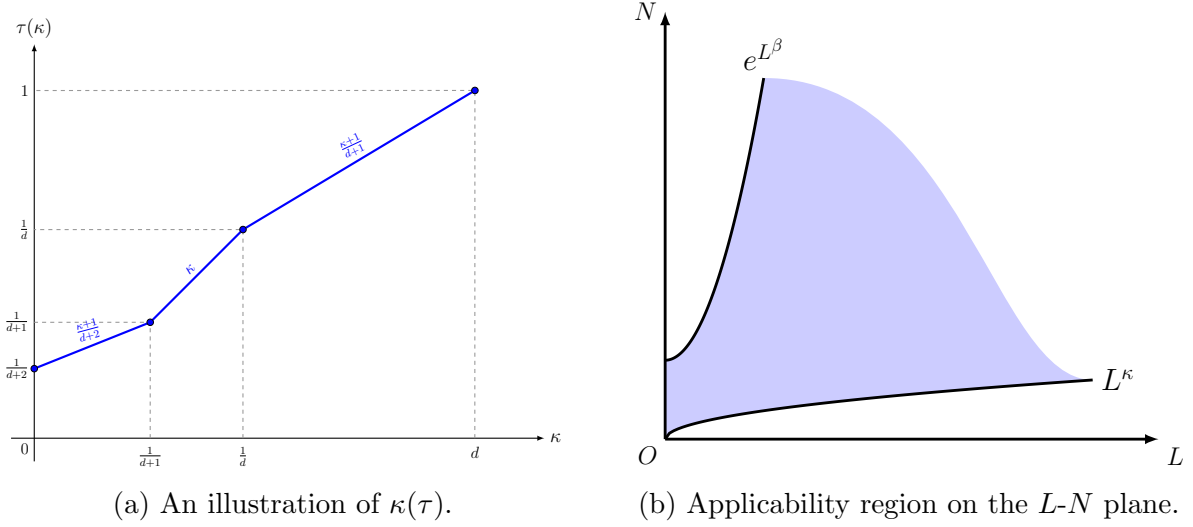

Theorem \ref{approximation} can be viewed as a generalization of prior works on single-parameter ReLU FNN approximations for analytic functions \cite{weinan2018exponential,opschoor2022exponential}. By specializing the parameters in Theorem \ref{approximation}, we can directly compare our results with these existing bounds. Specifically, \cite{weinan2018exponential} leveraged Taylor polynomial approximations to establish an approximation rate of $\mathcal{O}\left(e^{-C(d)L^{1/(2d)}}\right)$ for analytic functions with absolutely convergent power series, utilizing networks of depth $L$ and constant width. By setting $\kappa=0$, $\beta=1$, $\alpha=0$ and $N=C$ in Theorem \ref{approximation}, we recover a network architecture identical to that of \cite{weinan2018exponential}; however, our derived convergence rate is $\mathcal{O}\left(e^{-C(d,\boldsymbol{\rho})L^{1/(d+2)}}\right)$, which strictly improves upon their bound when $d>2$. \cite{opschoor2022exponential} relaxed the assumption on the target function in \cite{weinan2018exponential} by employing Legendre polynomial approximation techniques, thereby removing the requirement of absolute power series convergence. They established an approximation rate of $\mathcal{O}\left(e^{-C(d,\boldsymbol{\rho})L}\right)$ using networks with depth  $\widetilde{\mathcal{O}}(L)$ and $\mathcal{O}\left(L^{d+1}\right)$ parameters. By choosing $\kappa=d$, $\beta=1$, and $N=L^{d+\alpha}$ in Theorem \ref{approximation}, a straightforward estimation of the parameter number yields a network with depth $\widetilde{\mathcal{O}}(L)$ and  $\widetilde{\mathcal{O}}\left(L^{d+1+\alpha}\right)$ parameters, while achieving the identical approximation rate. Since $\alpha$ can be chosen arbitrarily small, our result is nearly identical to theirs.

Compared to the typical approximation rate of $\mathcal{O}\left(N^{-2\widetilde{s}}L^{-2\widetilde{s}}\right)$ for function classes with finite smoothness ($\widetilde{s}$ is some quantity related to the function smoothness), Theorem \ref{approximation} demonstrates that analytic functions can be approximated by neural networks at a significantly faster rate. Furthermore, a prominent distinction from the finite-smoothness setting is that depth plays a more critical role than width in the approximation of analytic functions, constituting the most crucial finding of the present work. This conclusion is further corroborated by the approximation lower bound presented below, which precludes the possibility of an $e^{-NL}$-type approximation rate.

\begin{theorem}\label{app lb}
For any $N$ and $L$, let $\mathcal{F}_{NN}\left({N},{L}\right)$ be the ReLU neural network function class with width ${N}$ and depth ${L}$. There holds
\begin{align*}
\inf_{f_{NN}\in\mathcal{F}_{NN}(N,L)}\sup_{f\in\mathcal{A}(\boldsymbol{\rho},M)}\|f-f_{NN}\|_{L^{\infty}([-1,1]^d)}\geq C(d,\boldsymbol{\rho},M)N^{-2L}.
\end{align*}
\end{theorem}

To directly compare Theorem \ref{approximation} and Theorem \ref{app lb}, replacing $N$ by $C(d,\boldsymbol{\rho})N$ and $L$ by $C(d,\boldsymbol{\rho},\beta)L$ in Theorem \ref{app lb} (when $\kappa=d$, replacing $L$ by $C(d,\boldsymbol{\rho},\beta)L(\log L)^2$) yields a lower bound of ${\Omega}(N^{-C(d,\boldsymbol{\rho},\beta)L})$ (or $\Omega(N^{-C(d,\boldsymbol{\rho},\beta)L(\log L)^2})$ when $\kappa=d$). Hence, our upper bound aligns with the lower bound in the regime of $\kappa=d$ if we ignore the logarithmic factor.

In Section \ref{Nonparametric Regression}, we will explore a classical application of our approximation theory by analyzing the statistical performance of the ReLU network estimator in nonparametric regression. In summary, the contributions of this work are highlighted as follows:

\begin{itemize}
\item 
We derive upper bounds for the approximation of analytic functions by ReLU networks under the $(N,L)$-characterization (Theorem \ref{approximation}). Distinct from previous studies on functions with finite smoothness, our findings reveal that depth plays a more critical role than width in the context of analytic function approximation, further highlighting the advantage of the $(N,L)$-characterization.

\item
The main difficulty in establishing Theorem \ref{approximation} lies in the trade-off between the smoothness parameters and the approximation accuracy. To overcome this difficulty, we employ refined technical constructions of several ReLU networks to approximate power functions, multivariate multiplication, and polynomials (see Section \ref{proof of app} for details). These intermediary results may be of independent interest.

\item
Furthermore, we derive an approximation lower bound (Theorem \ref{app lb}), which demonstrates that our upper bound is nearly optimal in the $\kappa=d$ regime.

\item
Based on the approximation results, we establish upper bounds on the convergence rate of a ReLU network estimator in nonparametric regression with analytic targets (Theorem \ref{regression ub}). Comparison with the minimax rate (Theorem \ref{regression minimax}) confirms that these upper bounds are nearly optimal.

\end{itemize}

\subsection{Notations}

The notations used in this paper are listed below.

\begin{itemize}
\item
The $L^{\infty}$ norm $\|\cdot\|_{L^{\infty}([-1,1]^d)}$ of a function $h$ is defined as the maximum of $h$ over $[-1,1]^d$:
    \begin{align*}
    \|h(\boldsymbol{x})\|_{L^{\infty}([-1,1]^d)}:=\max_{\boldsymbol{x}\in[-1,1]^d}|h(\boldsymbol{x})|.
    \end{align*}

\item 
$\mathbb{N}_{\geq1}$, $\mathbb{R}$ and $\mathbb{C}$ denote the set of positive integers, real numbers and complex numbers, respectively. 

\item 
Throughout this paper, all matrices and vectors are denoted in boldface.

\item 
The $l_1$ norm $\|\cdot\|_1$ of a $d$-dimensional vector $\boldsymbol{a}$ is defined as the absoulte sum of its components:
\begin{align*}
\|\boldsymbol{a}\|_1:=\sum_{i=1}^d|a_i|.
\end{align*}

\item For any $n\in\mathbb{N}_{\geq1}$, $[n]$ denotes the set $\{1,2,\cdots,n\}$.

\item 
Throughout this paper, the width parameter $N$ and the depth parameter $L$ are always positive integers.

\item 
By convention, any sum equals zero if its upper index is strictly less than the lower index.

\item
$\lceil\cdot\rceil$ denotes the standard ceiling function.

\item 
$\vee$ denotes the maximum operator.

\item 
For any set $\Gamma\subset\mathbb{R}^d$, $|\Gamma|$ denotes its Lebesgue measure.

\item
By convention, we use $C>0$ to denote constants which may change from line to line. This is standard and convenient in analysis.

\item 
The asymptotic notations used throughout this paper are defined as follows. 
For positive quantities $a$ and $b$, we write $a=\mathcal{O}(b)$ or $a\lesssim b$ if $a\leq Cb$ for some constant $C>0$; $a=\Omega(b)$ or $a\gtrsim b$ if $b=\mathcal{O}(a)$; and $a\asymp b$ or $a=\Theta(b)$ if both $a=\mathcal{O}(b)$ and $b=\mathcal{O}(a)$. We write $a=\widetilde{\mathcal{O}}(b)$ if $a\leq Cb\log^{\gamma} b$ for some positive constants $C$ and $\gamma$.

\end{itemize}

\subsection{Organization of This Paper}

The remainder of this paper is organized as follows. In Section \ref{proof of app}, we prove the approximation results presented in this section. Section \ref{Nonparametric Regression} is devoted to investigate the convergence rate of the ReLU network estimator in nonparametric regression. Finally, in Section \ref{Conclusions} we conclude the paper and outline several promising directions for future research.

\section{Proofs of Approximation Results}\label{proof of app}

A series of prior works have utilized Legendre polynomials to study analytic functions \cite{cohen2011analytic,chkifa2015breaking,adcock2022sparse}. As a class of orthogonal polynomials, the Legendre polynomial approximation is a global approximation, fundamentally distinguishing itself from the localized nature of Taylor polynomial approximations. In the context of constructing neural networks to approximate analytic functions, \cite{opschoor2022exponential} took advantage of this property to remove the assumption in \cite{weinan2018exponential} regarding the absolute convergence of the power series of the target function, whereas the latter employs Taylor polynomial approximation. Similar to \cite{opschoor2022exponential}, our method is also based on the approximation of analytic functions by multivariate Legendre polynomials (Proposition \ref{Legendre approximation}).
Therefore, our primary task is to construct desired neural networks that approximate these multivariate Legendre polynomials. To this end, our basic strategy is to separately construct networks for approximating univariate Legendre polynomials and $d$-variate multiplication, thereby yielding the approximation of multivariate Legendre polynomials. This path fully exploits the inherent structure of multivariate Legendre polynomials, enabling highly efficient approximation. In particular, we present two distinct concrete construction methods, leading to Proposition \ref{I} and Proposition \ref{II}, which apply to the regimes of smaller and larger $\kappa$, respectively. The approximation upper bound (Theorem \ref{approximation}) follows as a direct corollary of these two propositions.

The structure of this section is as follows. In Subsection \ref{proof of app ub}, we introduce Proposition \ref{I} and Proposition \ref{II}, through which we complete the proof of Theorem \ref{approximation}. In Subsection \ref{Legendre polynomials}, we briefly introduce Legendre polynomials and the necessary preliminary results. In Subsection \ref{Two basic neural networks}, we construct two basic networks needed for both methods. Based on these two basic networks, we construct the networks for Proposition \ref{I} and Proposition \ref{II} in Subsections \ref{proof of I} and \ref{proof of II}, respectively. Beyond the upper bounds, the lower bound (Theorem \ref{app lb}) is proved in Subsection \ref{Proof of lb}.

\subsection{Proof of Theorem \ref{approximation}}\label{proof of app ub}

Theorem \ref{approximation} is a direct corollary of the following two propositions, which established approximation results for the regime of  smaller and larger $\kappa$, respectively.

\begin{proposition}\label{I}
Suppose $L,N$ are sufficiently large and there exist $\kappa\in\left[0,\frac{d}{2}\right],\beta>0$ such that
\begin{align*}
L^{\kappa+\alpha}\leq N\leq e^{L^{\beta}},
\end{align*}
where $\alpha>0$ can be arbitrarily small; when $\kappa=0$, $\alpha$ can be $0$. For any $f\in\mathcal{A}(\boldsymbol{\rho},M)$, there exists a neural network $f_{NN}:\mathbb{R}^d\to\mathbb{R}$ with width $C(d,\boldsymbol{\rho})N$ and depth $C(d,\boldsymbol{\rho},\beta)L$ such that for any $\boldsymbol{x}\in\left[-1,1\right]^d$, there holds
\begin{align*}
\left| f(\boldsymbol{x}) - f_{NN}(\boldsymbol{x}) \right|
&\leq C(d,\boldsymbol{\rho},M)N^{-C(d,\boldsymbol{\rho},\kappa,\alpha)L^{\frac{\kappa+1}{d+2}}}.
\end{align*}
\end{proposition}

\begin{proposition}\label{II}
Suppose $L,N$ are sufficiently large and there exist $\kappa\in\left[\frac{1}{d+1},d\right],\beta>0$ such that
\begin{align*}
L^{\kappa+\alpha}\leq N\leq e^{L^{\beta}},
\end{align*}
where $\alpha>0$ can be arbitrarily small. For any $f\in\mathcal{A}(\boldsymbol{\rho},M)$, there exists a neural network $f_{NN}:\mathbb{R}^d\to\mathbb{R}$ with width $C(d,\boldsymbol{\rho})N$ and depth $C(d,\boldsymbol{\rho},\beta)L$ (when $\kappa=d$, a larger depth $C(d,\boldsymbol{\rho},\beta)L(\log L)^2$ is needed) such that for any $\boldsymbol{x}\in\left[-1,1\right]^d$, there holds
\begin{align*}
\left| f(\boldsymbol{x}) - f_{NN}(\boldsymbol{x}) \right|
&\leq\left\{\begin{matrix}
C(d,\boldsymbol{\rho},M)N^{-C(d,\boldsymbol{\rho},\kappa,\alpha)L^{\kappa}},  & \kappa\in\left[\frac{1}{d+1},\frac{1}{d}\right);\\
C(d,\boldsymbol{\rho},M)N^{-C(d,\boldsymbol{\rho},\kappa,\alpha)L^{\frac{\kappa+1}{d+1}}},  & \kappa\in\left[\frac{1}{d},d\right].
\end{matrix}\right.
\end{align*}
\end{proposition}

\begin{retheorem1}[restated]
Suppose $L,N$ are sufficiently large and there exist $\kappa\in\left[0,d\right],\beta>0$ such that
\begin{align*}
L^{\kappa+\alpha}\leq N\leq e^{L^{\beta}},
\end{align*}
where $\alpha>0$ can be arbitrarily small; when $\kappa=0$, $\alpha$ can be $0$. For any $f\in\mathcal{A}(\boldsymbol{\rho},M)$, there exists a neural network $f_{NN}:\mathbb{R}^d\to\mathbb{R}$ with width $C(d,\boldsymbol{\rho})N$ and depth $C(d,\boldsymbol{\rho},\beta)L$ (when $\kappa=d$, a larger depth $C(d,\boldsymbol{\rho},\beta)L(\log L)^2$ is needed) such that for any $\boldsymbol{x}\in\left[-1,1\right]^d$, there holds

\begin{align*}
\left| f(\boldsymbol{x}) - f_{NN}(\boldsymbol{x}) \right|
&\leq C(d,\boldsymbol{\rho},M)N^{-C(d,\boldsymbol{\rho},\kappa,\alpha)L^{\tau(\kappa)}}
\end{align*}
with
\begin{align*}
\tau(\kappa)=\left\{
\begin{matrix}
\frac{\kappa+1}{d+2},&\kappa\in\left[0,\frac{1}{d+1}\right];\\
\kappa,&\kappa\in\left(\frac{1}{d+1},\frac{1}{d}\right);\\
\frac{\kappa+1}{d+1},&\kappa\in\left[\frac{1}{d},d\right].
\end{matrix}
\right.
\end{align*}
\end{retheorem1}
\begin{proof}
~\newline
\begin{itemize}
\item For the regime $\kappa\in\left[0,\frac{1}{d+1}\right)$, Proposition \ref{I} is applicable while Proposition \ref{II} is not. 
\item For the regime $\kappa\in\left[\frac{1}{d+1},\frac{d}{2}\right]$, both propositions are applicable, with Proposition \ref{II} yielding a sharper approximation bound than Proposition \ref{I} (when $\kappa=\frac{1}{d+1}$, the two bounds are identical). 
\item For the regime $\kappa\in\left(\frac{d}{2},d\right]$, Proposition \ref{II} is applicable while Proposition \ref{I} is not.
\end{itemize}

\end{proof}

\subsection{Legendre Polynomials}\label{Legendre polynomials}

The univariate Legendre polynomial of degree $\nu\in\mathbb{N}_{\geq0}$ is defined as
\begin{align*}
L_\nu(x) = \frac{\sqrt{2\nu+1}}{2^\nu} \sum_{k=0}^{\lfloor \nu/2 \rfloor} (-1)^k \binom{\nu}{k} \binom{2\nu - 2k}{\nu} x^{\nu-2k}.
\end{align*}
Let $\mu_1$ be the uniform probability measure on $[-1,1]$. $\{L_{\nu}\}_{\nu\in\mathbb{N}_{\geq0}}$ forms an orthonormal basis in $L^2(\mu_1)$: 
\begin{align*}
\int_{-1}^{1} L_\nu(x) L_{\nu'}(x) \, d\mu_1(x) = \delta_{\nu\nu'},\quad\forall\nu,\nu'\in\mathbb{N}_{\geq0}.
\end{align*}
Rewrite the univariate Legendre polynomial as
\begin{align*}
L_\nu(x) = \sum_{\ell=0}^{\nu} c_{\ell}^{(\nu)} x^{\ell}
\end{align*}
with
\begin{align*}
c_{\ell}^{(\nu)} = 
\begin{cases} 
0, & \nu - \ell \text{ is odd}; \\ 
(-1)^{(\nu-\ell)/2} 2^{-\nu} \binom{\nu}{(\nu-\ell)/2} \binom{\nu+\ell}{\nu} \sqrt{2\nu+1}, & \nu - \ell \text{ is even}. 
\end{cases}
\end{align*}
The following lemma provides an upper bound of the sum of $\left|c_{\ell}^{(\nu)}\right|$.
\begin{lemma}\label{coefficient1}
For any $\nu\in\mathbb{N}_{\geq0}$, there holds
\begin{align*}
\sum_{\ell=0}^\nu\left|c_{\ell}^{(\nu)}\right|\leq 4^\nu.
\end{align*}
\end{lemma}
\begin{proof}
See the proof of \cite[Proposition 4.6]{opschoor2020deep}.
\end{proof}

The multivariate Legendre polynomial of $\boldsymbol{\nu}\in\mathbb{N}_{\geq0}^d$ is defined as the product of univariate Legendre polynomial of degree $\nu_i$:
\begin{align*}
L_{\boldsymbol{\nu}}(\boldsymbol{x})=\prod_{i=1}^{d}L_{\nu_i}(x_i).
\end{align*}
$\{L_{\boldsymbol{\nu}}\}_{\boldsymbol{\nu}\in\mathbb{N}_{\geq0}^d}$ forms an orthonormal basis in $L^2(\mu_d)$ with $\mu_d$ being the uniform probability measure on $[-1,1]^d$. Let $l_{\boldsymbol{\nu}}(f)$ denote the expansion coefficient of $f$ associated with $L_{\boldsymbol{\nu}}$ under the basis $\{L_{\boldsymbol{\nu}}\}_{\boldsymbol{\nu}\in\mathbb{N}_{\geq0}^d}$:
\begin{align*}
l_{\boldsymbol{\nu}}(f) := \int_{[-1, 1]^d} f({\boldsymbol{x}})L_{\boldsymbol{\nu}}({\boldsymbol{x}}) \, d\mu_d({\boldsymbol{x}}).
\end{align*}
The following two properties of $\{l_{\boldsymbol{\nu}}(f)\}$ are required for our subsequent analysis.
\begin{lemma}\label{coefficient1.5}
For any $\boldsymbol{\nu}\in\mathbb{N}_{\geq0}^d$ and $f\in\mathcal{A}(\boldsymbol{\rho},M)$, there holds
\[
\left|l_{\boldsymbol{\nu}}(f)\right| \leq C(M,\boldsymbol{\rho}) \left( \prod_{i=1}^d \sqrt{2\nu_i + 1}  \rho_i^{-\nu_i}\right). 
\]
\end{lemma}
\begin{proof}
See, for example, \cite[Theorem 3.2]{adcock2022sparse}.
\end{proof}

\begin{lemma}\label{coefficient2}
For any $f\in\mathcal{A}(\boldsymbol{\rho},M)$, there holds
\[
\sum_{\boldsymbol{\nu}\in\mathbb{N}_{\geq0}^d}\left|l_{\boldsymbol{\nu}}(f)\right| \leq C(M,\boldsymbol{\rho}).
\]
\end{lemma}
\begin{proof}
See, for example, \cite[Theorem 3.6]{adcock2022sparse}.
\end{proof}
To characterize the approximation rate of multivariate Legendre polynomials, we introduce a multi-index set with a tunable parameter $\epsilon \in (0, 1)$:
\begin{align*}
\Lambda_\epsilon := \left\{\boldsymbol{\nu} \in \mathbb{N}_{\geq0}^d : \prod_{i=1}^d{\rho_i}^{-\nu_i} \geq \epsilon\right\}.
\end{align*}
One can observe that $\Lambda_\epsilon$ grows as $\epsilon$ decreases.
The following lemma shows that the degree of Legendre polynomials associated with $\Lambda_{\epsilon}$ can be upper-bounded by the cardinality of $\Lambda_{\epsilon}$.
\begin{lemma}\label{m estimate}
There holds
\begin{align*}
\max_{\boldsymbol{\nu} \in \Lambda_{\epsilon}} \|\boldsymbol{\nu}\|_1\leq Cd|\Lambda_{\epsilon}|^{1/d}.
\end{align*}
\end{lemma}
\begin{proof}
See equation (3.6) in \cite{opschoor2022exponential}.
\end{proof}

As the cornerstone of this section, the following result establishes the approximation of analytic functions using multivariate Legendre polynomials.
\begin{proposition}[\cite{opschoor2022exponential}, Theorem 3.5]\label{Legendre approximation}
For any $f\in\mathcal{A}(\boldsymbol{\rho},M)$, there holds
\begin{align*}
\left\| f - \sum_{{\boldsymbol{\nu}} \in \Lambda_\epsilon} l_{\boldsymbol{\nu}}(f) L_{\boldsymbol{\nu}} \right\|_{L^{\infty}([-1, 1]^d)} \leq C(d,\boldsymbol{\rho},M) e^{-C(d,\boldsymbol{\rho}) |\Lambda_\epsilon|^{1/d}}.
\end{align*}
\end{proposition}

\subsection{Two Basic Neural Networks}\label{Two basic neural networks}

In this subsection, we construct two basic networks $f_{NN,n}^{(pow)}$ and $f_{NN,d}^{(mtp)}$: $f_{NN,n}^{(pow)}$ is used to approximate $x^n$, which will subsequently be utilized in different ways to construct networks for approximating univariate Legendre polynomials (Lemma \ref{newLegendre} and Lemma \ref{Legendre}); $f_{NN,d}^{(mtp)}$ serves to approximate $d$-variate multiplication. The construction of these two basic networks relies on the following approximation result for bivariate multiplication.

\begin{lemma}[\cite{lu2021deep}, Lemma 4.2]\label{multiplication 2}
Let \( a, b \in \mathbb{R} \) with \( a < b \). There exists a neural network \( f_{NN,2}^{(mtp)}:\mathbb{R}^2\to\mathbb{R} \) with width \( 10N  \) and depth \( L \) such that for any $x, y \in [a, b]$, there holds
\[
|f_{NN,2}^{(mtp)}(x, y) - xy| \leq 6(b - a)^2 N^{-L}.
\] 
\end{lemma}

We begin by describing the construction of $f_{NN,n}^{(pow)}$. Here, we introduce a tunable parameter $r$ to modulate the width and depth of the network; this parameter will take distinct values in the construction of univariate Legendre polynomials in Lemmas \ref{newLegendre} and \ref{Legendre}.

\begin{lemma}\label{power approximation}
Let $n,r\in\mathbb{N}_{\geq1}$. Suppose $2\leq r\leq n$. There exists a neural network $f_{NN,n}^{(pow)}:\mathbb{R}^{n}\to\mathbb{R}$ with width $5r2^{\lceil\log_2n\rceil}N$ and depth $5\lceil\log_2n\rceil\left\lceil\frac{\lceil\log_2n\rceil}{\log_2r}\right\rceil L$ such that for any $x\in[-1,1]$, there holds
\begin{align*}
\left|f_{NN,n}^{(pow)}(x)-x^{n}\right|&\leq 30N^{-5L}.
\end{align*}
\end{lemma}
\begin{remark}
This result improves upon \cite[Lemma 5.3]{lu2021deep}. The dependence of our network's depth on the input variable number $n$ is significantly weaker than theirs: since $r \ge 2$, the growth of our network depth with respect to $n$ does not exceed $(\log n)^2$, whereas theirs scales as $n^2$. Given that the smoothness parameter is related to the approximation error in the approximation of analytic functions, this improvement is pivotal for establishing analytic function approximation. As shown in the proof below, this improvement stems from our adoption of a binary tree construction.
\end{remark}
\begin{proof}
Let $K:=\lceil\log_2n\rceil$. We will construct a neural network $\phi_K$ implementing the multivariate multiplication. Specifically, we are going to show that there exists a neural network $\phi_K:\mathbb{R}^{2^K}\to\mathbb{R}$ with width $2^{K-1}\cdot10rN$ and depth $5K\left\lceil\frac{K}{\log_2r}\right\rceil L$ such that for any $x_1,\cdots,x_{2^K}\in[-1,1]$, there holds
\begin{align}\label{power approximation1}
\left|\phi_K(x_1,\cdots,x_{2^K})-x_1\cdots x_{2^K}\right|&\leq 30N^{-5L}.
\end{align}
Based on \eqref{power approximation1}, the power approximation network $f_{NN,n}^{(pow)}$ is defined as
\begin{align*}
f_{NN,n}^{(pow)}(x):=\phi_K(\mathcal{H}(x))
\end{align*}
with $\mathcal{H}$ being an affine mapping:
\begin{align*}
\mathcal{H}(x):=\begin{pmatrix}
x\cdot\boldsymbol{1}_{n\times 1}\\
\boldsymbol{1}_{\left(2^K-n\right)\times 1}
\end{pmatrix}\in\mathbb{R}^{2^K}.
\end{align*}

It remains to prove \eqref{power approximation1}. To this end, we inductively show that for $k\in[K]$, there exists a neural network $\phi_k:\mathbb{R}^{2^k}\to\mathbb{R}$ with width $2^{k-1}\cdot10rN$ and depth $5k\left\lceil\frac{K}{\log_2r}\right\rceil L$ such that for any $x_1,\cdots,x_{2^k}\in[-1,1]$, there holds
\begin{align}\label{power approximation2}
&\left|\phi_{k}(x_1,\cdots,x_{2^{k}})-x_1\cdots x_{2^{k}}\right|\leq30\cdot4^{k-1}(rN)^{-5\left\lceil\frac{K}{\log_2r}\right\rceil L}.
\end{align}
The case of $k=1$ is guaranteed by Lemma \ref{multiplication 2}: setting $b=1.1,a=-1.1$ in Lemma \ref{multiplication 2} and letting $\phi_{1}:=f_{NN,2}^{(mtp)}$ with width $10rN$ and depth $5\left\lceil\frac{K}{\log_2r}\right\rceil L$, we have for any $x_1,x_2\in[-1.1,1.1]$, there holds
\begin{align}\label{power approximation3}
\left|\phi_{1}(x_1,x_{2})-x_1 x_{2}\right|\leq30(rN)^{-5\left\lceil\frac{K}{\log_2r}\right\rceil L}.
\end{align}
The reason for choosing interval $[-1.1,1.1]$ rather than $[-1,1]$ becomes apparent in the proof below. Next, we assume \eqref{power approximation2} is valid for the case of $k-1$ and prove the case of $k$.
Based on the constructed $\phi_{k-1}$, $\phi_{k}$ is defined as
\begin{align*}
\phi_k(x_1,\cdots,x_{2^k}):=\phi_1(\phi_{k-1}(x_1,\cdots,x_{2^{k-1}}),\phi_{k-1}(x_{2^{k-1}+1},\cdots,x_{2^k})).
\end{align*}
By the induction assumption, the width of $\phi_k$ is 
\begin{align*}
\max\left\{2\cdot2^{k-2}\cdot10rN,10rN\right\}=2^{k-1}\cdot10rN
\end{align*}
and the depth is 
\begin{align*}
5(k-1)\left\lceil\frac{K}{\log_2r}\right\rceil L+5\left\lceil\frac{K}{\log_2r}\right\rceil L=5k\left\lceil\frac{K}{\log_2r}\right\rceil L.
\end{align*}
By the definition of $\phi_k$, we have
\begin{align}
&\left|\phi_k(x_1,\cdots,x_{2^k})-x_1\cdots x_{2^k}\right|\nonumber\\
&\leq\left|\phi_1(\phi_{k-1}(x_1,\cdots,x_{2^{k-1}}),\phi_{k-1}(x_{2^{k-1}+1},\cdots,x_{2^k}))\right.\nonumber\\
&\left.\quad\ -\phi_{k-1}(x_1,\cdots,x_{2^{k-1}})\phi_{k-1}(x_{2^{k-1}+1},\cdots,x_{2^k})\right|\nonumber\\
&\quad+\left|\phi_{k-1}(x_1,\cdots,x_{2^{k-1}})\phi_{k-1}(x_{2^{k-1}+1},\cdots,x_{2^k})-x_1\cdots x_{2^k}\right|.\label{power approximation4}
\end{align}
In order to apply \eqref{power approximation3} to bound the first term on the right hand side, we need to verify
\begin{align}\label{power approximation5}
|\phi_{k-1}(x_1,\cdots,x_{2^{k-1}})|,|\phi_{k-1}(x_{2^{k-1}+1},\cdots,x_{2^k})|\leq 1.1.
\end{align}
In fact, by the induction assumption,
\begin{align*}
|\phi_{k-1}(x_1,\cdots,x_{2^{k-1}})|&\leq\left|\phi_{k-1}(x_1,\cdots,x_{2^{k-1}})-x_1\cdots x_{2^{k-1}}\right|+\left|x_1\cdots x_{2^{k-1}}\right|\\
&\leq30\cdot4^{k-2}(rN)^{-5\left\lceil\frac{K}{\log_2r}\right\rceil L}+1\\
&\leq30\cdot4^{K-2}2^{-5KL}N^{-5\left\lceil\frac{K}{\log_2r}\right\rceil L}+1\\
&\leq\frac{15}{8}N^{-5}+1\\
&\leq\frac{15}{256}+1\\
&\leq0.1+1=1.1.
\end{align*}
$|\phi_{k-1}(x_{2^{k-1}+1},\cdots,x_{2^k})|$ can be estimated in the same way. Therefore, \eqref{power approximation3} leads to
\begin{align}
&\left|\phi_1(\phi_{k-1}(x_1,\cdots,x_{2^{k-1}}),\phi_{k-1}(x_{2^{k-1}+1},\cdots,x_{2^k}))\right.\nonumber\\
&\left.\ -\phi_{k-1}(x_1,\cdots,x_{2^{k-1}})\phi_{k-1}(x_{2^{k-1}+1},\cdots,x_{2^k})\right|\leq30(rN)^{-5\left\lceil\frac{K}{\log_2r}\right\rceil L}.\label{power approximation6}
\end{align}
Next we estimate the second term on the right hand side of \eqref{power approximation4}:
\begin{align}
&\left|\phi_{k-1}(x_1,\cdots,x_{2^{k-1}})\phi_{k-1}(x_{2^{k-1}+1},\cdots,x_{2^k})-x_1\cdots x_{2^k}\right|\nonumber\\
&\leq\left|\phi_{k-1}(x_1,\cdots,x_{2^{k-1}})\phi_{k-1}(x_{2^{k-1}+1},\cdots,x_{2^k})-x_1\cdots x_{2^{k-1}}\phi_{k-1}(x_{2^{k-1}+1},\cdots,x_{2^k})\right|\nonumber\\
&\quad+\left|x_1\cdots x_{2^{k-1}}\phi_{k-1}(x_{2^{k-1}+1},\cdots,x_{2^k})-x_1\cdots x_{2^k}\right|\nonumber\\
&=\left|\phi_{k-1}(x_1,\cdots,x_{2^{k-1}})-x_1\cdots x_{2^{k-1}}\right|\left|\phi_{k-1}(x_{2^{k-1}+1},\cdots,x_{2^k})\right|\nonumber\\
&\quad+\left|x_1\cdots x_{2^{k-1}}\right|\left|\phi_{k-1}(x_{2^{k-1}+1},\cdots,x_{2^k})-x_{2^{k-1}+1}\cdots x_{2^k}\right|\nonumber\\
&\leq 1.1\cdot30\cdot4^{k-2}(rN)^{-5\left\lceil\frac{K}{\log_2r}\right\rceil L}+30\cdot4^{k-2}(rN)^{-5\left\lceil\frac{K}{\log_2r}\right\rceil L}=2.1\cdot30\cdot4^{k-2}(rN)^{-5\left\lceil\frac{K}{\log_2r}\right\rceil L},\label{power approximation7}
\end{align}
where in the third step we make use of \eqref{power approximation5} and the induction assumption. Plugging \eqref{power approximation6} and \eqref{power approximation7} into \eqref{power approximation4}, we obtain
\begin{align*}
\left|\phi_k(x_1,\cdots,x_{2^k})-x_1\cdots x_{2^k}\right|&\leq30(rN)^{-5\left\lceil\frac{K}{\log_2r}\right\rceil L}+2.1\cdot30\cdot4^{k-2}(rN)^{-5\left\lceil\frac{K}{\log_2r}\right\rceil L}\\
&=30(1+2.1\cdot4^{k-2})(rN)^{-5\left\lceil\frac{K}{\log_2r}\right\rceil L}\leq30\cdot4^{k-1}(rN)^{-5\left\lceil\frac{K}{\log_2r}\right\rceil L}.
\end{align*}
Hence the induction is completed and we derive that
\begin{align*}
\left|\phi_K(x_1,\cdots,x_{2^K})-x_1\cdots x_{2^K}\right|&\leq30\cdot4^{K-1}(rN)^{-5\left\lceil\frac{K}{\log_2r}\right\rceil L}\\
&=30\cdot4^{K-1}2^{-5KL}N^{-5\left\lceil\frac{K}{\log_2r}\right\rceil L}\leq 30N^{-5L},
\end{align*}
which is exactly \eqref{power approximation1}.

\end{proof}

Analogous to $f_{NN,n}^{(pow)}$, the construction of $f_{NN,d}^{(mtp)}$ also adopts a binary tree approach. However, the emphasis is different from the former. Since the number of input variables here is the spatial dimension $d$—rather than a smoothness parameter that is profoundly tied to the approximation rate as in the former case—we do not need to dedicate excessive effort to reducing the dependence of the depth on $d$. Instead, the domain of the input variables is larger than $[-1,1]$, which impacts the error amplification during layer-wise propagation. Here, we introduce two parameters $D$ and $H$, respectively relating to the width and depth of the network, to control the size of the domain. Their specific values will be determined later in Lemmas \ref{new multivariate Legendre} and \ref{final} according to concrete requirements.

\begin{lemma}\label{multiplication d}
Let $D,H\in\mathbb{N}_{\geq1}$. Suppose $D^H\geq\frac{7}{2}$. There exists a neural network $f_{NN,d}^{(mtp)}:\mathbb{R}^{d}\to\mathbb{R}$ with width $5\cdot2^{\lceil\log_2d\rceil}DN$ and depth $2d\lceil\log_2d\rceil^2H$ such that for any $\boldsymbol{x}\in\left[-D^H,D^H\right]^d$, there holds
\begin{align*}
\left|f_{NN,d}^{(mtp)}(\boldsymbol{x})-x_1\cdots x_{d}\right|\leq 96N^{-2d\lceil\log_2d\rceil H }.
\end{align*}
\end{lemma}
\begin{proof}
Let $K:=\lceil\log_2d\rceil$. We inductively show that for $k\in[K]$, there exists a neural network $\phi_k:\mathbb{R}^{2^k}\to\mathbb{R}$ with width $2^{k-1}\cdot10DN$ and depth $2kKdH$ such that for any $x_1,\cdots,x_{2^k}\in\left[-D^{H},D^{H}\right]$, there holds
\begin{align}\label{multiplication d1}
&\left|\phi_{k}(x_1,\cdots,x_{2^{k}})-x_1\cdots x_{2^{k}}\right|\leq96D^{2kdH}(DN)^{-2KdH}.
\end{align}
The case of $k=1$ is guaranteed by Lemma \ref{multiplication 2}: setting $b=2D^{dH},a=-2D^{dH}$ in Lemma \ref{multiplication 2} and letting $\phi_{1}:=f_{NN,2}^{(mtp)}$ with width $10DN$ and depth $2KdH$, we have for any $x_1,x_2\in\left[-2D^{dH},2D^{dH}\right]$, there holds
\begin{align}\label{multiplication d2}
\left|\phi_{1}(x_1,x_{2})-x_1 x_{2}\right|\leq96D^{2dH}(DN)^{-2KdH}.
\end{align}
The reason for choosing interval $\left[-2D^{dH},2D^{dH}\right]$ rather than $\left[-D^{H},D^{H}\right]$ becomes apparent in the proof below. Next, we assume \eqref{multiplication d1} is valid for the case of $k-1$ and prove the case of $k$. Based on the constructed $\phi_{k-1}$, $\phi_{k}$ is defined as
\begin{align*}
\phi_k(x_1,\cdots,x_{2^k}):=\phi_1(\phi_{k-1}(x_1,\cdots,x_{2^{k-1}}),\phi_{k-1}(x_{2^{k-1}+1},\cdots,x_{2^k})).
\end{align*}
By the induction assumption, the width of $\phi_k$ is 
\begin{align*}
\max\left\{2\cdot2^{k-2}\cdot10DN,10DN\right\}=2^{k-1}\cdot10DN
\end{align*}
and the depth is 
\begin{align*}
2(k-1)KdH+2KdH=2kKdH.
\end{align*}
By the definition of $\phi_k$, we have
\begin{align}
&\left|\phi_k(x_1,\cdots,x_{2^k})-x_1\cdots x_{2^k}\right|\nonumber\\
&\leq\left|\phi_1(\phi_{k-1}(x_1,\cdots,x_{2^{k-1}}),\phi_{k-1}(x_{2^{k-1}+1},\cdots,x_{2^k}))\right.\nonumber\\
&\left.\quad\ -\phi_{k-1}(x_1,\cdots,x_{2^{k-1}})\phi_{k-1}(x_{2^{k-1}+1},\cdots,x_{2^k})\right|\nonumber\\
&\quad+\left|\phi_{k-1}(x_1,\cdots,x_{2^{k-1}})\phi_{k-1}(x_{2^{k-1}+1},\cdots,x_{2^k})-x_1\cdots x_{2^k}\right|.\label{multiplication d3}
\end{align}
In order to apply \eqref{multiplication d2} to bound the first term on the right hand side, we need to verify
\begin{align}\label{multiplication d4}
|\phi_{k-1}(x_1,\cdots,x_{2^{k-1}})|,|\phi_{k-1}(x_{2^{k-1}+1},\cdots,x_{2^k})|\leq 2D^{dH}.
\end{align}
In fact, by the induction assumption,
\begin{align*}
|\phi_{k-1}(x_1,\cdots,x_{2^{k-1}})|&\leq\left|\phi_{k-1}(x_1,\cdots,x_{2^{k-1}})-x_1\cdots x_{2^{k-1}}\right|+\left|x_1\cdots x_{2^{k-1}}\right|\\
&\leq96D^{2(k-1)dH}(DN)^{-2KdH}+D^{dH}\\
&\leq 96D^{2(K-1)dH}(DN)^{-2Kd H}+D^{dH}\\
&\leq D^{dH}\cdot96D^{-3dH}N^{-2KdH}+D^{dH}\\
&\leq D^{dH}\cdot96\left(\frac{2}{7}\right)^3\frac{1}{4}+D^{dH}\\
&\leq D^{dH}+D^{dH}=2D^{dH}.
\end{align*}
$|\phi_{k-1}(x_{2^{k-1}+1},\cdots,x_{2^k})|$ can be estimated in the same way. Therefore, \eqref{multiplication d2} leads to
\begin{align}\label{multiplication d5}
&\left|\phi_1(\phi_{k-1}(x_1,\cdots,x_{2^{k-1}}),\phi_{k-1}(x_{2^{k-1}+1},\cdots,x_{2^k}))\right.\nonumber\\
&\left.\ -\phi_{k-1}(x_1,\cdots,x_{2^{k-1}})\phi_{k-1}(x_{2^{k-1}+1},\cdots,x_{2^k})\right|\leq96D^{2dH}(DN)^{-2KdH}.
\end{align}
Next we estimate the second term on the right hand side of \eqref{multiplication d3}:
\begin{align}
&\left|\phi_{k-1}(x_1,\cdots,x_{2^{k-1}})\phi_{k-1}(x_{2^{k-1}+1},\cdots,x_{2^k})-x_1\cdots x_{2^k}\right|\nonumber\\
&\leq\left|\phi_{k-1}(x_1,\cdots,x_{2^{k-1}})\phi_{k-1}(x_{2^{k-1}+1},\cdots,x_{2^k})-x_1\cdots x_{2^{k-1}}\phi_{k-1}(x_{2^{k-1}+1},\cdots,x_{2^k})\right|\nonumber\\
&\quad+\left|x_1\cdots x_{2^{k-1}}\phi_{k-1}(x_{2^{k-1}+1},\cdots,x_{2^k})-x_1\cdots x_{2^k}\right|\nonumber\\
&=\left|\phi_{k-1}(x_1,\cdots,x_{2^{k-1}})-x_1\cdots x_{2^{k-1}}\right|\left|\phi_{k-1}(x_{2^{k-1}+1},\cdots,x_{2^k})\right|\nonumber\\
&\quad+\left|x_1\cdots x_{2^{k-1}}\right|\left|\phi_{k-1}(x_{2^{k-1}+1},\cdots,x_{2^k})-x_{2^{k-1}+1}\cdots x_{2^k}\right|\nonumber\\
&\leq 2D^{dH}\cdot96D^{2(k-1)dH}(DN)^{-2KdH }+D^{dH}\cdot96D^{2(k-1)dH}(DN)^{-2KdH}\nonumber\\
&=3D^{dH}\cdot96D^{2(k-1)dH}(DN)^{-2KdH},\label{multiplication d6}
\end{align}
where in the third step we make use of \eqref{multiplication d4} and the induction assumption. Plugging \eqref{multiplication d5} and \eqref{multiplication d6} into \eqref{multiplication d3}, we obtain
\begin{align*}
\left|\phi_k(x_1,\cdots,x_{2^k})-x_1\cdots x_{2^k}\right|&\leq96D^{2dH}(DN)^{-2KdH }+3D^{dH}\cdot96D^{2(k-1)dH}(DN)^{-2KdH}\\
&=96\left(D^{2dH}+3D^{dH}\cdot D^{2(k-1)dH}\right)(DN)^{-2KdH}\\
&\leq 96\left(D^{2(k-1)dH}+3D^{dH}\cdot D^{2(k-1)dH}\right)(DN)^{-2KdH}\\
&\leq 96\left(D^{2(k-1)dH}\cdot(D^{2dH}-3D^{dH})+3D^{dH}\cdot D^{2(k-1)dH}\right)(DN)^{-2KdH}\\
&=96D^{2kdH}(DN)^{-2KdH},
\end{align*}
where in the fourth step we use the inequality $D^{2dH}-3D^{dH}\geq1$ since $D^H\geq\frac{7}{2}$. By far, the induction is completed and we construct a neural network $\phi_K:\mathbb{R}^{2^K}\to\mathbb{R}$ with width $2^{K-1}\cdot10DN$ and depth $2K^2dH$ such that for any $x_1,\cdots,x_{2^K}\in\left[-D^{H},D^{H}\right]$, there holds 
\begin{align*}
\left|\phi_K(x_1,\cdots,x_{2^K})-x_1\cdots x_{2^K}\right|&\leq96D^{2KdH}(DN)^{-2KdH}\leq 96N^{-2KdH}.
\end{align*}
Now, $f_{NN,d}^{(mtp)}$ is defined as
\begin{align*}
f_{NN,d}^{(mtp)}(\boldsymbol{x}):=\phi_K(\mathcal{H}(\boldsymbol{x})),\quad\boldsymbol{x}\in\left[-D^{H},D^{H}\right]^d,
\end{align*}
with $\mathcal{H}$ being an affine mapping:
\begin{align*}
\mathcal{H}(\boldsymbol{x}):=\begin{pmatrix}
\boldsymbol{x}\\
\boldsymbol{1}_{\left(2^K-d\right)\times 1}
\end{pmatrix}\in\mathbb{R}^{2^K}.
\end{align*}

\end{proof}

\subsection{Proof of Proposition \ref{I}}\label{proof of I}

Leveraging Lemma \ref{power approximation}, we proceed to construct the network for approximating univariate Legendre polynomials, where a tunable parameter $k$ is introduced to modulate the width and depth of the network.
An illustration of the construction is provided in Figure \ref{figure prop1}, where the black lines represent the construction of the powers of $x$ (more precisely, their approximations), while the blue lines denote their linear combination. For simplicity, the depicted case assumes $k$ is divisible by $\nu$.

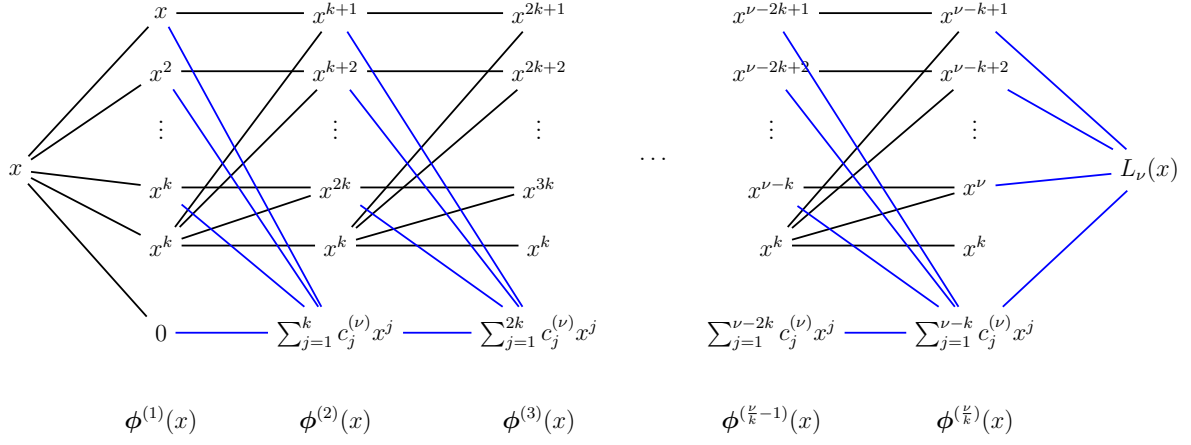
\begin{figure}
    \centering
\resizebox{\textwidth}{!}{
\begin{tikzpicture}[
    scale=0.9, every node/.style={transform shape},
    blackline/.style={black, thick},
    redline/.style={blue, thick}
]

\node (x_src) at (0, -1.2) {$x$};

\node (c1_1) at (2.5, 1.5) {$x$};
\node (c1_2) at (2.5, 0.5) {$x^2$};
\node (c1_dots) at (2.5, -0.4) {$\vdots$};
\node (c1_3) at (2.5, -1.5) {$x^k$};
\node (c1_4) at (2.5, -2.5) {$x^k$};
\node (c1_0) at (2.5, -4.0) {$0$};
\node (phi1) at (2.5, -5.5) {$\boldsymbol{\phi}^{(1)}(x)$};

\draw[blackline] (x_src) -- (c1_1);
\draw[blackline] (x_src) -- (c1_2);
\draw[blackline] (x_src) -- (c1_3);
\draw[blackline] (x_src) -- (c1_4);
\draw[blackline] (x_src) -- (c1_0);

\node (c2_1) at (5.5, 1.5) {$x^{k+1}$};
\node (c2_2) at (5.5, 0.5) {$x^{k+2}$};
\node (c2_dots) at (5.5, -0.4) {$\vdots$};
\node (c2_3) at (5.5, -1.5) {$x^{2k}$};
\node (c2_4) at (5.5, -2.5) {$x^k$};
\node (c2_sum) at (5.5, -4.0) {$\sum_{j=1}^k c_j^{(\nu)} x^j$};
\node (phi2) at (5.5, -5.5) {$\boldsymbol{\phi}^{(2)}(x)$};

\draw[blackline] (c1_1) -- (c2_1);
\draw[blackline] (c1_2) -- (c2_2);
\draw[blackline] (c1_3) -- (c2_3);
\draw[blackline] (c1_4) -- (c2_4);
\draw[blackline] (c1_4) -- (c2_1);
\draw[blackline] (c1_4) -- (c2_2);
\draw[blackline] (c1_4) -- (c2_3);

\draw[redline] (c1_1) -- (c2_sum);
\draw[redline] (c1_2) -- (c2_sum);
\draw[redline] (c1_3) -- (c2_sum);
\draw[redline] (c1_0) -- (c2_sum);

\node (c3_1) at (9.0, 1.5) {$x^{2k+1}$};
\node (c3_2) at (9.0, 0.5) {$x^{2k+2}$};
\node (c3_dots) at (9.0, -0.4) {$\vdots$};
\node (c3_3) at (9.0, -1.5) {$x^{3k}$};
\node (c3_4) at (9.0, -2.5) {$x^k$};
\node (c3_sum) at (9.0, -4.0) {$\sum_{j=1}^{2k} c_j^{(\nu)} x^j$};
\node (phi3) at (9.0, -5.5) {$\boldsymbol{\phi}^{(3)}(x)$};

\draw[blackline] (c2_1) -- (c3_1);
\draw[blackline] (c2_2) -- (c3_2);
\draw[blackline] (c2_3) -- (c3_3);
\draw[blackline] (c2_4) -- (c3_4);
\draw[blackline] (c2_4) -- (c3_1);
\draw[blackline] (c2_4) -- (c3_2);
\draw[blackline] (c2_4) -- (c3_3);

\draw[redline] (c2_1) -- (c3_sum);
\draw[redline] (c2_2) -- (c3_sum);
\draw[redline] (c2_3) -- (c3_sum);
\draw[redline] (c2_sum) -- (c3_sum);

\node (dots_center) at (11.0, -1.0) {$\dots$};

\node (c4_1) at (13.0, 1.5) {$x^{\nu-2k+1}$};
\node (c4_2) at (13.0, 0.5) {$x^{\nu-2k+2}$};
\node (c4_dots) at (13.0, -0.4) {$\vdots$};
\node (c4_3) at (13.0, -1.5) {$x^{\nu-k}$};
\node (c4_4) at (13.0, -2.5) {$x^k$};
\node (c4_sum) at (13.0, -4.0) {$\sum_{j=1}^{\nu-2k} c_j^{(\nu)} x^j$};
\node (phi4) at (13.0, -5.5) {$\boldsymbol{\phi}^{(\frac{\nu}{k}-1)}(x)$};

\node (c5_1) at (16.5, 1.5) {$x^{\nu-k+1}$};
\node (c5_2) at (16.5, 0.5) {$x^{\nu-k+2}$};
\node (c5_dots) at (16.5, -0.4) {$\vdots$};
\node (c5_3) at (16.5, -1.5) {$x^{\nu}$};
\node (c5_4) at (16.5, -2.5) {$x^k$};
\node (c5_sum) at (16.5, -4.0) {$\sum_{j=1}^{\nu-k} c_j^{(\nu)} x^j$};
\node (phi5) at (16.5, -5.5) {$\boldsymbol{\phi}^{(\frac{\nu}{k})}(x)$};

\draw[blackline] (c4_1) -- (c5_1);
\draw[blackline] (c4_2) -- (c5_2);
\draw[blackline] (c4_3) -- (c5_3);
\draw[blackline] (c4_4) -- (c5_4);
\draw[blackline] (c4_4) -- (c5_1);
\draw[blackline] (c4_4) -- (c5_2);
\draw[blackline] (c4_4) -- (c5_3);

\draw[redline] (c4_1) -- (c5_sum);
\draw[redline] (c4_2) -- (c5_sum);
\draw[redline] (c4_3) -- (c5_sum);
\draw[redline] (c4_sum) -- (c5_sum);

\node (L_nu) at (19.5, -1.2) {$L_{\nu}(x)$};
        
        \draw[redline] (c5_1) -- (L_nu);
        \draw[redline] (c5_2) -- (L_nu);
        \draw[redline] (c5_3) -- (L_nu);
        \draw[redline] (c5_sum) -- (L_nu);

\end{tikzpicture}
}
    \caption{An illustration of the proof of Lemma \ref{newLegendre}.}
    \label{figure prop1}
\end{figure}

\begin{lemma}\label{newLegendre}
Let $\nu,\bar{\nu},k\in\mathbb{N}_{\geq1}$. Suppose $k\leq \nu\leq\bar{\nu}$. Suppose $N^{\left\lceil\frac{\bar{\nu}}{k}\right\rceil}\geq\left(300\cdot4^{\left\lceil\frac{{\nu}}{k}\right\rceil-1}\right)^{1/5}$. There exists a neural network $f_{NN,\nu}^{(Lgd)}:\mathbb{R}\to\mathbb{R}$ with width $10k4^{\lceil\log_2k\rceil}N+4$ and depth $5\left\lceil\frac{\bar{\nu}}{k}\right\rceil\left(\lceil\log_2k\rceil+\left\lceil\frac{{\nu}}{k}\right\rceil-1\right)$ such that for any ${x}\in\left[-1,1\right]$, there holds
\begin{align*}
\left|f_{NN,\nu}^{(Lgd)}(x)-L_\nu(x)\right|
\leq\frac{15}{2}\cdot16^{\nu}N^{-5\left\lceil\frac{\bar{\nu}}{k}\right\rceil}.
\end{align*}
\end{lemma}
\begin{proof}
We inductively show that for $i\in[\left\lceil\frac{\nu}{k}\right\rceil]$, there exists a neural network $\boldsymbol{\phi}^{(i)}:\mathbb{R}\to\mathbb{R}^{k+2}$ with width $10k4^{\lceil\log_2k\rceil}N+4$ and depth $5\left\lceil\frac{\bar{\nu}}{k}\right\rceil\left(\lceil\log_2k\rceil+i-1\right)$ such that for any $x\in\left[-1,1\right]$, there holds
\begin{gather}
\left|\phi_j^{(i)}(x)-x^{(i-1)k+j}\right|\leq 30\cdot4^{i-1}N^{-5\left\lceil\frac{\bar{\nu}}{k}\right\rceil},\quad j\in[k],\label{newLegendre1}\\
\left|\phi_{k+1}^{(i)}(x)-x^{k}\right|\leq 30N^{-5\left\lceil\frac{\bar{\nu}}{k}\right\rceil},\label{newLegendre2}\\
\phi_{k+2}^{(i)}=\sum_{i'=1}^{i-1}\sum_{j=1}^{k}c_{(i'-1)k+j}^{(\nu)}\phi_{j}^{(i')}.\label{newLegendre3}
\end{gather}
For $j\in[k]$, by setting $r=2^{\left\lceil\log_2{j}\right\rceil},L=\left\lceil\frac{\bar{\nu}}{k}\right\rceil$ in Lemma \ref{power approximation}, we find a neural network $f_{NN,j}^{(pow)}:\mathbb{R}\to\mathbb{R}$ with width $5\cdot4^{\lceil\log_2j\rceil}N$ and depth $5\lceil\log_2j\rceil \left\lceil\frac{\bar{\nu}}{k}\right\rceil$ such that for any $x\in[-1,1]$, there holds
\begin{align*}
\left|f_{NN,j}^{(pow)}({x})-{x}^{j}\right|&\leq 30N^{-5\left\lceil\frac{\bar{\nu}}{k}\right\rceil}.
\end{align*}
Hence $\boldsymbol{\phi}^{(1)}:\mathbb{R}\to\mathbb{R}^{k+2}$ defined in the following way satisfies \eqref{newLegendre1}-\eqref{newLegendre3}:
\begin{align*}
\boldsymbol{\phi}^{(1)}(x):=\begin{pmatrix}
f_{NN,1}^{(pow)}(x)\\f_{NN,2}^{(pow)}(x)\\\vdots\\f_{NN,k}^{(pow)}(x)\\f_{NN,k}^{(pow)}(x)\\0 
\end{pmatrix}\in\mathbb{R}^{k+2}.
\end{align*}
From its definition, we find that the width of $\boldsymbol{\phi}^{(1)}$ is
\begin{align*}
\sum_{j=1}^k5\cdot4^{\lceil\log_2j\rceil}N+5\cdot4^{\lceil\log_2k\rceil}N+2\leq10k4^{\lceil\log_2k\rceil}N+4
\end{align*}
and the depth is 
\begin{align*}
\max_{j\in[k]}5\left\lceil\frac{\bar{\nu}}{k}\right\rceil\lceil\log_2j\rceil=5\left\lceil\frac{\bar{\nu}}{k}\right\rceil\lceil\log_2k\rceil.
\end{align*}
Next, we assume \eqref{newLegendre1}-\eqref{newLegendre3} is valid for the case of $i$ and prove the case of $i+1$. Let $f_{NN,2}^{(mtp)}:\mathbb{R}^2\to\mathbb{R}$ be the neural network in Lemma \ref{multiplication 2} with $b=1.1,a=-1.1,L=5\left\lceil\frac{\bar{\nu}}{k}\right\rceil$. Then $f_{NN,2}^{(mtp)}$ is of width $10N$ and depth $5\left\lceil\frac{\bar{\nu}}{k}\right\rceil$, and 
\begin{align}\label{newLegendre4}
|f_{NN,2}^{(mtp)}(x_1, x_2) - x_1x_2| \leq 30N^{-5\left\lceil\frac{\bar{\nu}}{k}\right\rceil},\quad\forall x_1,x_2\in[-1.1,1.1].
\end{align}
Define
\begin{align*}
\boldsymbol{\phi}^{(i+1)}(x):=
\begin{pmatrix}
f_{NN,2}^{(mtp)}\left(\phi_{1}^{(i)},\phi_{k+1}^{(i)}\right)\\
f_{NN,2}^{(mtp)}\left(\phi_{2}^{(i)},\phi_{k+1}^{(i)}\right)\\
\vdots\\
f_{NN,2}^{(mtp)}\left(\phi_{k}^{(i)},\phi_{k+1}^{(i)}\right)\\
\phi_{k+1}^{(i)}\\
\phi_{k+2}^{(i)}+\sum_{j=1}^{k} c_{(i-1)k+j}^{(\nu)} \phi_{j}^{(i)}
\end{pmatrix}\in\mathbb{R}^{k+2}.
\end{align*}
From its definition and the induction assumption, we find that the width of $\boldsymbol{\phi}^{(i+1)}$ is
\begin{align*}
\max\left\{10k4^{\lceil\log_2k\rceil}N+4,10kN+4\right\}=10k4^{\lceil\log_2k\rceil}N+4
\end{align*}
and the depth is 
\begin{align*}
5\left\lceil\frac{\bar{\nu}}{k}\right\rceil\left(\lceil\log_2k\rceil+i-1\right) +5\left\lceil\frac{\bar{\nu}}{k}\right\rceil=5\left\lceil\frac{\bar{\nu}}{k}\right\rceil\left(\lceil\log_2k\rceil+i\right).
\end{align*}
Based on the induction assumption, it is easy to check that $\boldsymbol{\phi}^{(i+1)}$ defined above satisfies \eqref{newLegendre2}-\eqref{newLegendre3}. In the following, we verify $\boldsymbol{\phi}^{(i+1)}$ also satisfies \eqref{newLegendre1}. By the triangle inequality, for $j\in[k]$,
\begin{align}
&\left|\phi_j^{(i+1)}(x)-x^{ik+j}\right|\nonumber\\
&\leq\left|f_{NN,2}^{(mtp)}\left(\phi_j^{(i)}(x),\phi_{k+1}^{(i)}(x)\right)-\phi_j^{(i)}(x)\phi_{k+1}^{(i)}(x)\right|+\left|\phi_j^{(i)}(x)\phi_{k+1}^{(i)}(x)-x^{ik+j}\right|.\label{newLegendre5}
\end{align}
We aim to employ \eqref{Legendre2} to bound the first term on the right-hand side. Thereby, we need to check  
\begin{align}\label{newLegendre6}
\left|\phi_j^{(i)}(x)\right|\leq1.1,\quad j\in[k+1].
\end{align}
In fact, this is guaranteed by the induction assumption and the condition $N^{\left\lceil\frac{\bar{\nu}}{k}\right\rceil}\geq\linebreak\left(300\cdot4^{\left\lceil\frac{{\nu}}{k}\right\rceil-1}\right)^{1/5}$:
\begin{align*}
\left|\phi_j^{(i)}(x)\right|&\leq\left|\phi_j^{(i)}(x)-x^{(i-1)k+j}\right|+\left|x^{(i-1)k+j}\right|\leq 30\cdot4^{i-1}N^{-5\left\lceil\frac{\bar{\nu}}{k}\right\rceil}+1\\
&\leq 30\cdot4^{\left\lceil\frac{{\nu}}{k}\right\rceil-1}N^{-5\left\lceil\frac{\bar{\nu}}{k}\right\rceil}+1\leq0.1+1=1.1,\quad j\in[k];\\
\left|\phi_{k+1}^{(i)}(x)\right|&\leq\left|\phi_{k+1}^{(i)}(x)-x^{k}\right|+\left|x^{k}\right|\leq 30N^{-5\left\lceil\frac{\bar{\nu}}{k}\right\rceil}+1\leq0.1+1=1.1.
\end{align*}
Therefore, \eqref{newLegendre4} leads to
\begin{align}\label{newLegendre7}
\left|f_{NN,2}^{(mtp)}\left(\phi_j^{(i)}(x),\phi_{k+1}^{(i)}(x)\right)-\phi_j^{(i)}(x)\phi_{k+1}^{(i)}(x)\right|\leq 30N^{-5\left\lceil\frac{\bar{\nu}}{k}\right\rceil}.
\end{align}
For the second term on the right-hand side of \eqref{newLegendre5}, 
\begin{align}
&\left|\phi_j^{(i)}(x)\phi_{k+1}^{(i)}(x)-x^{ik+j}\right|\nonumber\\
&\leq\left|\phi_j^{(i)}(x)\phi_{k+1}^{(i)}(x)-x^{(i-1)k+j}\phi_{k+1}^{(i)}(x)\right|+\left|x^{(i-1)k+j}\phi_{k+1}^{(i)}(x)-x^{ik+j}\right|\nonumber\\
&\leq\left|\phi_j^{(i)}(x)-x^{(i-1)k+j}\right|\left|\phi_{k+1}^{(i)}(x)\right|+\left|\phi_{k+1}^{(i)}(x)-x^{k}\right|\nonumber\\
&\leq 1.1\cdot30\cdot4^{i-1}N^{-5\left\lceil\frac{\bar{\nu}}{k}\right\rceil}+30N^{-5\left\lceil\frac{\bar{\nu}}{k}\right\rceil},\label{newLegendre8}
\end{align}
where we make use of the induction assumption and \eqref{newLegendre6} in the third step. Plugging \eqref{newLegendre7} and \eqref{newLegendre8} into \eqref{newLegendre5}, we get
\begin{align*}
\left|\phi_j^{(i+1)}(x)-x^{ik+j}\right|
&\leq30N^{-5\left\lceil\frac{\bar{\nu}}{k}\right\rceil}+1.1\cdot30\cdot4^{i-1}N^{-5\left\lceil\frac{\bar{\nu}}{k}\right\rceil}+30N^{-5\left\lceil\frac{\bar{\nu}}{k}\right\rceil}\\
&=30(1.1\cdot4^{i-1}+2)N^{-5\left\lceil\frac{\bar{\nu}}{k}\right\rceil}\leq30\cdot4^{i}N^{-5\left\lceil\frac{\bar{\nu}}{k}\right\rceil}.
\end{align*}
Hence, the induction is completed and we construct a neural network $\boldsymbol{\phi}^{\left(\left\lceil\frac{{\nu}}{k}\right\rceil\right)}(x):\mathbb{R}\to\mathbb{R}^{k+2}$ with width $10k4^{\lceil\log_2k\rceil}N+4$ and depth $5\left\lceil\frac{\bar{\nu}}{k}\right\rceil\left(\lceil\log_2k\rceil+\left\lceil\frac{{\nu}}{k}\right\rceil-1\right)$ such that for any $x\in\left[-1,1\right]$, there holds
\begin{gather*}
\left|\phi_j^{\left(\left\lceil\frac{{\nu}}{k}\right\rceil\right)}(x)-x^{\left(\left\lceil\frac{{\nu}}{k}\right\rceil-1\right)k+j}\right|\leq 30\cdot4^{\left\lceil\frac{{\nu}}{k}\right\rceil-1}N^{-5\left\lceil\frac{\bar{\nu}}{k}\right\rceil},\quad j\in[k],\\
\left|\phi_{k+1}^{\left(\left\lceil\frac{{\nu}}{k}\right\rceil\right)}(x)-x^{k}\right|\leq 30N^{-5\left\lceil\frac{\bar{\nu}}{k}\right\rceil},\\
\phi_{k+2}^{\left(\left\lceil\frac{{\nu}}{k}\right\rceil\right)}=\sum_{i=1}^{\left\lceil\frac{{\nu}}{k}\right\rceil-1}\sum_{j=1}^{k}c_{(i-1)k+j}^{(\nu)}\phi_{j}^{(i)}.
\end{gather*}
Now, $f_{NN,\nu}^{(Lgd)}$ is defined as
\begin{align*}
f_{NN,\nu}^{(Lgd)}(x):&=\begin{pmatrix}
c_{\left(\left\lceil\frac{{\nu}}{k}\right\rceil-1\right)k+1}^{(\nu)}&c_{\left(\left\lceil\frac{{\nu}}{k}\right\rceil-1\right)k+2}^{(\nu)}&\cdots&c_{\nu}^{(\nu)}&\boldsymbol{0}_{1\times(k\left\lceil\frac{{\nu}}{k}\right\rceil-\nu+1)}&1    
\end{pmatrix}\boldsymbol{\phi}^{\left(\left\lceil\frac{{\nu}}{k}\right\rceil\right)}(x)+c_0^{(\nu)}\\
&=\sum_{j=1}^{\nu-\left(\left\lceil\frac{{\nu}}{k}\right\rceil-1\right)k} c_{\left(\left\lceil\frac{{\nu}}{k}\right\rceil-1\right)k+j}^{(\nu)} {\phi}_{j}^{\left(\left\lceil\frac{{\nu}}{k}\right\rceil\right)}(x)+\sum_{i=1}^{\left\lceil\frac{{\nu}}{k}\right\rceil-1}\sum_{j=1}^{k}c_{(i-1)k+j}^{(\nu)}\phi_{j}^{(i)}(x)+c_0^{(\nu)}.
\end{align*}
It follows that
\begin{align*}
&\left|f_{NN,\nu}^{(Lgd)}(x)-L_\nu(x)\right|\\
&=\left|\sum_{j=1}^{\nu-\left(\left\lceil\frac{{\nu}}{k}\right\rceil-1\right)k} c_{\left(\left\lceil\frac{{\nu}}{k}\right\rceil-1\right)k+j}^{(\nu)} {\phi}_{j}^{\left(\left\lceil\frac{{\nu}}{k}\right\rceil\right)}(x)+\sum_{i=1}^{\left\lceil\frac{{\nu}}{k}\right\rceil-1}\sum_{j=1}^{k}c_{(i-1)k+j}^{(\nu)}\phi_{j}^{(i)}(x)+c_0^{(\nu)}-\sum_{\ell=0}^{\nu} c_{\ell}^{(\nu)} x^{\ell}\right|\\
&=\left|\sum_{j=1}^{\nu-\left(\left\lceil\frac{{\nu}}{k}\right\rceil-1\right)k} c_{\left(\left\lceil\frac{{\nu}}{k}\right\rceil-1\right)k+j}^{(\nu)} {\phi}_{j}^{\left(\left\lceil\frac{{\nu}}{k}\right\rceil\right)}(x)+\sum_{i=1}^{\left\lceil\frac{{\nu}}{k}\right\rceil-1}\sum_{j=1}^{k}c_{(i-1)k+j}^{(\nu)}\phi_{j}^{(i)}(x)\right.\\
&\quad\left.-\sum_{j=1}^{\nu-\left(\left\lceil\frac{{\nu}}{k}\right\rceil-1\right)k} c_{\left(\left\lceil\frac{{\nu}}{k}\right\rceil-1\right)k+j}^{(\nu)}x^{\left(\left\lceil\frac{{\nu}}{k}\right\rceil-1\right)k+j}-\sum_{i=1}^{\left\lceil\frac{{\nu}}{k}\right\rceil-1}\sum_{j=1}^{k}c_{(i-1)k+j}^{(\nu)}x^{(i-1)k+j}\right|\\
&\leq\sum_{j=1}^{\nu-\left(\left\lceil\frac{{\nu}}{k}\right\rceil-1\right)k} \left|c_{\left(\left\lceil\frac{{\nu}}{k}\right\rceil-1\right)k+j}^{(\nu)}\right|\left| {\phi}_{j}^{\left(\left\lceil\frac{{\nu}}{k}\right\rceil\right)}(x)-x^{\left(\left\lceil\frac{{\nu}}{k}\right\rceil-1\right)k+j}\right|+\sum_{i=1}^{\left\lceil\frac{{\nu}}{k}\right\rceil-1}\sum_{j=1}^{k}\left|c_{(i-1)k+j}^{(\nu)}\right|\left|\phi_{j}^{(i)}(x)-x^{(i-1)k+j}\right|\\
&\leq 30\cdot 4^{\left\lceil\frac{{\nu}}{k}\right\rceil-1}N^{-5\left\lceil\frac{\bar{\nu}}{k}\right\rceil}\left(\sum_{j=1}^{\nu-\left(\left\lceil\frac{{\nu}}{k}\right\rceil-1\right)k} \left|c_{\left(\left\lceil\frac{{\nu}}{k}\right\rceil-1\right)k+j}^{(\nu)}\right|+\sum_{i=1}^{\left\lceil\frac{{\nu}}{k}\right\rceil-1}\sum_{j=1}^{k}\left|c_{(i-1)k+j}^{(\nu)}\right|\right)\\
&=30\cdot 4^{\left\lceil\frac{{\nu}}{k}\right\rceil-1}N^{-5\left\lceil\frac{\bar{\nu}}{k}\right\rceil}\sum_{\ell=1}^{\nu}\left|c_{(\ell}^{(\nu)}\right|\\
&\leq30\cdot 4^{\nu+\left\lceil\frac{{\nu}}{k}\right\rceil-1}N^{-5\left\lceil\frac{\bar{\nu}}{k}\right\rceil}\\
&\leq\frac{15}{2}\cdot16^{\nu}N^{-5\left\lceil\frac{\bar{\nu}}{k}\right\rceil},
\end{align*}
where we apply \eqref{newLegendre1} in the fourth step and Lemma \ref{coefficient1} in the sixth step.

\end{proof}

Coupling Lemma \ref{newLegendre} with Lemma \ref{multiplication d} allows us to construct the network for approximating multivariate Legendre polynomials.

\begin{lemma}\label{new multivariate Legendre}
Let $k,\bar{\nu}\in\mathbb{N}_{\geq1},\boldsymbol{\nu}\in\mathbb{N}_{\geq0}^d,q\in\mathbb{N}_{\geq2}$. Suppose $1\leq\|\boldsymbol{\nu}\|_1\leq\bar{\nu}$ and $k\leq\bar{\nu}$. Suppose $N^{\left\lceil\frac{\bar{\nu}}{k}\right\rceil}\geq\left(300\cdot4^{\left\lceil\frac{\bar{\nu}}{k}\right\rceil-1}\right)^{1/5}$. There exists a neural network $f_{NN,\boldsymbol{\nu}}^{(d-Lgd)}:\mathbb{R}^d\to\mathbb{R}$ with width $\max\left\{5\cdot2^{\lceil\log_2d\rceil}qN,10dk4^{\lceil\log_2k\rceil}N+4d\right\}$ and depth $2d\lceil\log_2d\rceil^2\left\lceil\frac{4\bar{\nu}}{\log_2q}\right\rceil +5\left\lceil\frac{\bar{\nu}}{k}\right\rceil\left(\lceil\log_2k\rceil+\left\lceil\frac{\bar{\nu}}{k}\right\rceil-1\right)$ such that for any $\boldsymbol{x}\in\left[-1,1\right]^d$, there holds
\begin{align*}
\left|f_{NN,\boldsymbol{\nu}}^{(d-Lgd)}(\boldsymbol{x})-L_{\boldsymbol{\nu}}(\boldsymbol{x})\right|&\leq96N^{-2d\lceil\log_2d\rceil\left\lceil\frac{4\bar{\nu}}{\log_2q}\right\rceil }+\frac{15}{2}d16^{d\bar{\nu}}N^{-5\left\lceil\frac{\bar{\nu}}{k}\right\rceil}.
\end{align*}
\end{lemma}
\begin{proof}
For $i\in[d]$, if $\nu_i\geq1$, then by Lemma \ref{newLegendre}, there exists a neural network $f_{NN,\nu_i}^{(Lgd)}:\mathbb{R}\to\mathbb{R}$ with width $10k4^{\lceil\log_2k\rceil}N+4$ and depth $5\left\lceil\frac{\bar{\nu}}{k}\right\rceil\left(\lceil\log_2k\rceil+\left\lceil\frac{{\nu_i}}{k}\right\rceil-1\right)$ such that for any ${x}\in\left[-1,1\right]$, there holds
\begin{align}\label{new multivariate Legendre1}
\left|f_{NN,\nu_i}^{(Lgd)}(x)-L_{\nu_i}(x)\right|
\leq\frac{15}{2}\cdot16^{\nu_i}N^{-5\left\lceil\frac{\bar{\nu}}{k}\right\rceil}.
\end{align}
If $\nu_i=0$, then let 
\begin{align}\label{new multivariate Legendre2}
f_{NN,{\nu_i}}^{(Lgd)}(x):=L_{\nu_i}(x)=L_0(x)=1.
\end{align}
Let $f_{NN,d}^{(mtp)}:\mathbb{R}^{d}\to\mathbb{R}$ be the neural network in Lemma \ref{multiplication d} with $D=q,H=\left\lceil\frac{4\bar{\nu}}{\log_2q}\right\rceil$. Then $f_{NN,d}^{(mtp)}$ is of width $5\cdot2^{\lceil\log_2d\rceil}qN$ and depth $2d\lceil\log_2d\rceil^2\left\lceil\frac{4\bar{\nu}}{\log_2q}\right\rceil$, and
\begin{align}\label{new multivariate Legendre3}
\left|f_{NN,d}^{(mtp)}(\boldsymbol{x})-x_1\cdots x_{d}\right|\leq 96N^{-2d\lceil\log_2d\rceil \left\lceil\frac{4\bar{\nu}}{\log_2q}\right\rceil},\quad\forall \boldsymbol{x}\in\left[-q^{\left\lceil\frac{4\bar{\nu}}{\log_2q}\right\rceil},q^{\left\lceil\frac{4\bar{\nu}}{\log_2q}\right\rceil}\right]^d.
\end{align}
$f_{NN,\boldsymbol{\nu}}^{(d-Lgd)}$ is now defined as
\begin{align*}
f_{NN,\boldsymbol{\nu}}^{(d-Lgd)}(\boldsymbol{x}):=f_{NN,d}^{(mtp)}\left(f_{NN,{\nu_1}}^{(Lgd)}({x}_1),f_{NN,{\nu_2}}^{(Lgd)}({x}_2),\cdots,f_{NN,{\nu_d}}^{(Lgd)}({x}_d)\right).
\end{align*}
It follows that the width of $f_{NN,\boldsymbol{\nu}}^{(d-Lgd)}$ is 
\begin{align*}
\max\left\{5\cdot2^{\lceil\log_2d\rceil}qN,d\left(10k4^{\lceil\log_2k\rceil}N+4\right)\right\}
\end{align*}
and the depth is 
\begin{align*}
2d\lceil\log_2d\rceil^2\left\lceil\frac{4\bar{\nu}}{\log_2q}\right\rceil +5\left\lceil\frac{\bar{\nu}}{k}\right\rceil\left(\lceil\log_2k\rceil+\left\lceil\frac{\bar{\nu}}{k}\right\rceil-1\right).
\end{align*}
By the triangle inequality,
\begin{align}\label{new multivariate Legendre4}
\left|f_{NN,\boldsymbol{\nu}}^{(d-Lgd)}(\boldsymbol{x})-L_{\boldsymbol{\nu}}(\boldsymbol{x})\right|&\leq\left|f_{NN,\boldsymbol{\nu}}^{(d-Lgd)}(\boldsymbol{x})-\prod_{i=1}^{d}f_{NN,{\nu_i}}^{(Lgd)}({x}_i)\right|+\left|\prod_{i=1}^{d}f_{NN,{\nu_i}}^{(Lgd)}({x}_i)-\prod_{i=1}^dL_{\nu_i}(x_i)\right|.
\end{align}
We aim to employ \eqref{new multivariate Legendre3} to bound the first term on the right hand side. To this end, we need to check that for $i\in[d]$,
\begin{align*}
\left|f_{NN,{\nu}_i}^{(Lgd)}(x_i)\right|&\leq q^{\left\lceil\frac{4\bar{\nu}}{\log_2q}\right\rceil}.
\end{align*}
Since $|x_i|\leq1$,
\begin{align}\label{new multivariate Legendre4.5}
\left|L_{\nu_i}(x_i)\right|=\left|\sum_{\ell=0}^{\nu_i}c_{\ell}^{(\nu_i)}x_i^{\ell}\right|\leq\sum_{\ell=0}^{\nu_i}\left|c_{\ell}^{(\nu_i)}\right|\leq 4^{\nu_i}.
\end{align}
It follows that
\begin{align}
\left|f_{NN,{\nu}_i}^{(Lgd)}(x_i)\right|&\leq\left|f_{NN,{\nu}_i}^{(Lgd)}(x_i)-L_{\nu_i}(x_i)\right|+\left|L_{\nu_i}(x_i)\right|
\leq\frac{15}{2}\cdot16^{\nu_i}N^{-5\left\lceil\frac{\bar{\nu}}{k}\right\rceil}+4^{\nu_i}\nonumber\\
&\leq\frac{1}{2}\cdot16^{\nu_i}+\frac{1}{2}\cdot16^{\nu_i}=16^{\nu_i}\leq16^{\bar{\nu}}=q^{\frac{4\bar{\nu}}{\log_2q}}\leq q^{\left\lceil\frac{4\bar{\nu}}{\log_2q}\right\rceil},\label{new multivariate Legendre4.75}
\end{align}
where we employ \eqref{new multivariate Legendre1}\eqref{new multivariate Legendre2} in the second step. Therefore, \eqref{new multivariate Legendre3} leads to
\begin{align}\label{new multivariate Legendre5}
\left|f_{NN,\boldsymbol{\nu}}^{(d-Lgd)}(\boldsymbol{x})-\prod_{i=1}^{d}f_{NN,{\nu_i}}^{(Lgd)}({x}_i)\right|\leq
96N^{-2d\lceil\log_2d\rceil\left\lceil\frac{4\bar{\nu}}{\log_2q}\right\rceil }.
\end{align}
For the second term on the right-hand side of \eqref{new multivariate Legendre4}, 
\begin{align}
&\left|\prod_{i=1}^{d}f_{NN,{\nu_i}}^{(Lgd)}({x}_i)-\prod_{i=1}^dL_{\nu_i}(x_i)\right|\nonumber\\
&\leq\left|\prod_{i=1}^{d}f_{NN,{\nu_i}}^{(Lgd)}({x}_i)-L_{\nu_1}(x_1)\cdot\prod_{i=2}^{d}f_{NN,{\nu_i}}^{(Lgd)}({x}_i)\right|\nonumber\\
&\quad+\left|L_{\nu_1}(x_1)\cdot\prod_{i=2}^{d}f_{NN,{\nu_i}}^{(Lgd)}({x}_i)-\prod_{i=1}^2L_{\nu_i}(x_i)\cdot\prod_{i=3}^{d}f_{NN,{\nu_i}}^{(Lgd)}({x}_i)\right|\nonumber\\
&\quad+\cdots+\left|\prod_{i=1}^{i'}L_{\nu_i}(x_i)\cdot\prod_{i=i'+1}^{d}f_{NN,{\nu_i}}^{(Lgd)}({x}_i)-\prod_{i=1}^{i'+1}L_{\nu_i}(x_i)\cdot\prod_{i=i'+2}^{d}f_{NN,{\nu_i}}^{(Lgd)}({x}_i)\right|\nonumber\\
&\quad+\cdots+\left|\prod_{i=1}^{d-1}L_{\nu_i}(x_i)\cdot f_{NN,{\nu_d}}^{(Lgd)}({x}_d)-\prod_{i=1}^{d}L_{\nu_i}(x_i)\right|\nonumber\\
&\leq d16^{(d-1)\bar{\nu}}\cdot\frac{15}{2}16^{\bar{\nu}}N^{-5\left\lceil\frac{\bar{\nu}}{k}\right\rceil}\nonumber\\
&=\frac{15}{2}d16^{d\bar{\nu}}N^{-5\left\lceil\frac{\bar{\nu}}{k}\right\rceil},\label{new multivariate Legendre6}
\end{align}
where we apply \eqref{new multivariate Legendre1},\eqref{new multivariate Legendre2},\eqref{new multivariate Legendre4.5} and \eqref{new multivariate Legendre4.75} in the second step. Plugging \eqref{new multivariate Legendre5} and \eqref{new multivariate Legendre6} into \eqref{new multivariate Legendre4}, we finish the proof.

\end{proof}

The network approximating analytic functions is constructed as a linear combination of $\left\{f_{NN,\boldsymbol{\nu}}^{(d-Lgd)}\right\}_{\boldsymbol{\nu}\in \Lambda_\epsilon}$. Here, we employ a ``squashing" technique: rather than computing the entire summation simultaneously within a single wide layer, we perform sequential additions across multiple layers, with each layer only summing $p$ terms of $\left\{f_{NN,\boldsymbol{\nu}}^{(d-Lgd)}\right\}_{\boldsymbol{\nu}\in \Lambda_\epsilon}$. This technique reduces the network width by increasing its depth.
Since this technique is also utilized to prove Lemma \ref{final}, whose construction is more intricate than that of Lemma \ref{new final}, we only present the schematic illustration for Lemma \ref{final} (Figure \ref{figure prop2-2}), while Lemma \ref{new final} shares a similar yet simpler structure.

\begin{lemma}\label{new final}
Let $\lambda\in[0,d]$. For sufficiently large $L$ and $N$, there exists a neural network $f_{NN}:\mathbb{R}^d\to\mathbb{R}$ with width $C(d,\boldsymbol{\rho})L^{\lambda}N(\log N)^{d}$ and depth $C(d,\boldsymbol{\rho})L^{d-\lambda+1}\left(L+\log\log N\right)$ such that for any $\boldsymbol{x}\in\left[-1,1\right]^d$, there holds
\begin{align*}
\left| f(\boldsymbol{x}) - f_{NN}(\boldsymbol{x}) \right|
&\leq C(d,M,\boldsymbol{\rho})N^{-C(d,\boldsymbol{\rho})L}.
\end{align*}
\end{lemma}
\begin{proof}
Let $\left\{f_{NN,\boldsymbol{\nu}}^{(d-Lgd)}\right\}_{{\boldsymbol{\nu}} \in \Lambda_\epsilon\setminus\{\boldsymbol{0}\}}$ be the neural networks contructed in Lemma \ref{new multivariate Legendre} with 
\begin{align*}
\bar{\nu}:=C(d)|\Lambda_{\epsilon}|^{1/d}\geq \|\boldsymbol{\nu}\|_1,\quad \forall{\boldsymbol{\nu}} \in \Lambda_\epsilon,
\end{align*}
where we employ Lemma \ref{m estimate}. We also replace $N$ with $\sqrt{N}$ (assume without loss of generality $\sqrt{N}$ is an integer) in Lemma \ref{new multivariate Legendre}. It follows that the width of $\left\{f_{NN,\boldsymbol{\nu}}^{(d-Lgd)}\right\}_{{\boldsymbol{\nu}} \in \Lambda_\epsilon\setminus\{\boldsymbol{0}\}}$ is
$$\max\left\{5\cdot2^{\lceil\log_2d\rceil}q\sqrt{N},10dk4^{\lceil\log_2k\rceil}\sqrt{N}+4d\right\}$$
and the depth is $$2d\lceil\log_2d\rceil^2\left\lceil\frac{4\bar{\nu}}{\log_2q}\right\rceil +5\left\lceil\frac{\bar{\nu}}{k}\right\rceil\left(\lceil\log_2k\rceil+\left\lceil\frac{\bar{\nu}}{k}\right\rceil-1\right)
\leq C(d)\left(\frac{|\Lambda_{\epsilon}|^{1/d}}{\log_2q}+\frac{|\Lambda_{\epsilon}|^{1/d}}{k}\log_2k+\frac{|\Lambda_{\epsilon}|^{2/d}}{k^2}\right).$$ 
Furthermore, for $\boldsymbol{x}\in\left[-1,1\right]^d$,
\begin{align}
\left|f_{NN,\boldsymbol{\nu}}^{(d-Lgd)}(\boldsymbol{x})-L_{\boldsymbol{\nu}}(\boldsymbol{x})\right|&\leq96N^{-d\lceil\log_2d\rceil\left\lceil\frac{4\bar{\nu}}{\log_2q}\right\rceil }+\frac{15}{2}d16^{d\bar{\nu}}N^{-\frac{1}{2}\left\lceil\frac{\bar{\nu}}{k}\right\rceil}\nonumber\\
&\leq C(d)\left(N^{-C(d)\frac{|\Lambda_{\epsilon}|^{1/d}}{\log_2q} }+16^{C(d)|\Lambda_{\epsilon}|^{1/d}}N^{-C(d)|\Lambda_{\epsilon}|^{1/d}/k}\right).\label{new final0}
\end{align}
For convenience, we rearrange
$
\left\{f_{NN,\boldsymbol{\nu}}^{(d-Lgd)}(\boldsymbol{x})\right\}_{\boldsymbol{\nu}\in\Lambda_\epsilon\setminus\{\boldsymbol{0}\}}$ as
$\left\{f_{NN,j}^{(d-Lgd)}(\boldsymbol{x})\right\}_{j\in\left[\left|\Lambda_\epsilon\setminus\{\boldsymbol{0}\}\right|\right]}$ and define
\begin{align*}
f_{NN,0}^{(d-Lgd)}(\boldsymbol{x}):=1.
\end{align*}
We also rearrange the corresponding Legendre expansion coefficients
$
\left\{l_{\boldsymbol{\nu}}(f)\right\}_{\boldsymbol{\nu}\in\Lambda_\epsilon}$ as $\left\{l_{j}\right\}_{j=0}^{\left|\Lambda_\epsilon\setminus\{\boldsymbol{0}\}\right]}$ with 
$l_{\boldsymbol{0}}(f)$ as $l_{0}$. In the following, we inductively show that for $i\in\left[\left\lceil\frac{\left|\Lambda_\epsilon\setminus\{\boldsymbol{0}\}\right|}{p}\right\rceil\right]$, there exists a neural network $\boldsymbol{\phi}^{(i)}:\mathbb{R}^d\to\mathbb{R}^{p+d+1}$ with width 
$$p\cdot\max\left\{5\cdot2^{\lceil\log_2d\rceil}q\sqrt{N},10dk4^{\lceil\log_2k\rceil}\sqrt{N}+4d\right\}+2(d+1)$$ 
and depth $$i\cdot\left[2d\lceil\log_2d\rceil^2\left\lceil\frac{4\bar{\nu}}{\log_2q}\right\rceil +5\left\lceil\frac{\bar{\nu}}{k}\right\rceil\left(\lceil\log_2k\rceil+\left\lceil\frac{\bar{\nu}}{k}\right\rceil-1\right)\right]$$
such that for any $\boldsymbol{x}\in\left[-1,1\right]^d$, there holds
\begin{align}\label{new final1}
\boldsymbol{\phi}^{(i)}(\boldsymbol{x})=\begin{pmatrix}
\boldsymbol{x}\\
f_{NN,(i-1)p+1}^{(d-Lgd)}(\boldsymbol{x})\\
f_{NN,(i-1)p+2}^{(d-Lgd)}(\boldsymbol{x})\\
\vdots\\
f_{NN,(i-1)p+p}^{(d-Lgd)}(\boldsymbol{x})\\
\sum_{j=0}^{(i-1)p} l_{j} f_{NN,j}^{(d-Lgd)}(\boldsymbol{x})
\end{pmatrix}\in\mathbb{R}^{p+d+1}.
\end{align}
For the case of $i=1$, define directly
\begin{align*}
\boldsymbol{\phi}^{(1)}(\boldsymbol{x}):=\begin{pmatrix}
\boldsymbol{x}\\
f_{NN,1}^{(d-Lgd)}(\boldsymbol{x})\\
f_{NN,2}^{(d-Lgd)}(\boldsymbol{x})\\
\vdots\\
f_{NN,p}^{(d-Lgd)}(\boldsymbol{x})\\
l_0
\end{pmatrix}\in\mathbb{R}^{p+d+1}.
\end{align*}
From its definition, we see that the width of $\boldsymbol{\phi}^{(1)}$ is 
$$p\cdot\max\left\{5\cdot2^{\lceil\log_2d\rceil}q\sqrt{N},10dk4^{\lceil\log_2k\rceil}\sqrt{N}+4d\right\}+2(d+1)$$ 
and the depth is $$2d\lceil\log_2d\rceil^2\left\lceil\frac{4\bar{\nu}}{\log_2q}\right\rceil +5\left\lceil\frac{\bar{\nu}}{k}\right\rceil\left(\lceil\log_2k\rceil+\left\lceil\frac{\bar{\nu}}{k}\right\rceil-1\right).$$
Now, we assume \eqref{new final1} is valid for the case of $i$ and define
\begin{align*}
\boldsymbol{\phi}^{(i+1)}(\boldsymbol{x}):=\begin{pmatrix}
\boldsymbol{\phi}_{1:d}^{(i)}\\
f_{NN,ip+1}^{(d-Lgd)}\left(\boldsymbol{\phi}_{1:d}^{(i)}\right)\\
f_{NN,ip+2}^{(d-Lgd)}\left(\boldsymbol{\phi}_{1:d}^{(i)}\right)\\
\vdots\\
f_{NN,ip+p}^{(d-Lgd)}\left(\boldsymbol{\phi}_{1:d}^{(i)}\right)\\
{\phi}_{p+d+1}^{(i)}+\sum_{j=1}^{p}l_{(i-1)p+j}{\phi}_{d+j}^{(i)}
\end{pmatrix}=\begin{pmatrix}
\boldsymbol{x}\\
f_{NN,ip+1}^{(d-Lgd)}(\boldsymbol{x})\\
f_{NN,ip+2}^{(d-Lgd)}(\boldsymbol{x})\\
\vdots\\
f_{NN,ip+p}^{(d-Lgd)}(\boldsymbol{x})\\
\sum_{j=0}^{ip} l_{j} f_{NN,j}^{(d-Lgd)}(\boldsymbol{x})
\end{pmatrix}\in\mathbb{R}^{p+d+1},
\end{align*}
where the second equality is due to the induction assumption. It also follows from induction assumption that the width of $\boldsymbol{\phi}^{(i+1)}$ is  
\begin{align*}
&p\cdot\max\left\{5\cdot2^{\lceil\log_2d\rceil}q\sqrt{N},10dk4^{\lceil\log_2k\rceil}\sqrt{N}+4d\right\}+2(d+1)
\end{align*}
and the depth is 
\begin{align*}
(i+1)\cdot\left[2d\lceil\log_2d\rceil^2\left\lceil\frac{4\bar{\nu}}{\log_2q}\right\rceil  +5\left\lceil\frac{\bar{\nu}}{k}\right\rceil\left(\lceil\log_2k\rceil+\left\lceil\frac{\bar{\nu}}{k}\right\rceil-1\right)\right].
\end{align*}
Hence, the induction is completed and we derive that
\begin{align*}
\boldsymbol{\phi}^{\left(\left\lceil\frac{\left|\Lambda_\epsilon\setminus\{\boldsymbol{0}\}\right|}{p}\right\rceil\right)}(\boldsymbol{x})=\begin{pmatrix}
\boldsymbol{x}\\
f_{NN,\left(\left\lceil\frac{\left|\Lambda_\epsilon\setminus\{\boldsymbol{0}\}\right|}{p}\right\rceil-1\right)p+1}^{(d-Lgd)}(\boldsymbol{x})\\
f_{NN,\left(\left\lceil\frac{\left|\Lambda_\epsilon\setminus\{\boldsymbol{0}\}\right|}{p}\right\rceil-1\right)p+2}^{(d-Lgd)}(\boldsymbol{x})\\
\vdots\\
f_{NN,\left(\left\lceil\frac{\left|\Lambda_\epsilon\setminus\{\boldsymbol{0}\}\right|}{p}\right\rceil-1\right)p+p}^{(d-Lgd)}(\boldsymbol{x})\\
\sum_{j=0}^{\left(\left\lceil\frac{\left|\Lambda_\epsilon\setminus\{\boldsymbol{0}\}\right|}{p}\right\rceil-1\right)p} l_{j} f_{NN,j}^{(d-Lgd)}(\boldsymbol{x})
\end{pmatrix}\in\mathbb{R}^{p+d+1},
\end{align*}
Furthermore, its width is 
\begin{align*}
&p\cdot\max\left\{5\cdot2^{\lceil\log_2d\rceil}q\sqrt{N},10dk4^{\lceil\log_2k\rceil}\sqrt{N}+4d\right\}+2(d+1)
\end{align*}
and its depth is
\begin{align*}
&\left\lceil\frac{\left|\Lambda_\epsilon\setminus\{\boldsymbol{0}\}\right|}{p}\right\rceil\cdot\left[2d\lceil\log_2d\rceil^2\left\lceil\frac{4\bar{\nu}}{\log_2q}\right\rceil  +\left\lceil\frac{\bar{\nu}}{k}\right\rceil\left(\lceil\log_2k\rceil+\left\lceil\frac{\bar{\nu}}{k}\right\rceil-1\right)\right]\\
&\leq C(d)\frac{\left|\Lambda_\epsilon\right|}{p}\left(\frac{|\Lambda_{\epsilon}|^{1/d}}{\log_2q}+\frac{|\Lambda_{\epsilon}|^{1/d}}{k}\log_2k+\frac{|\Lambda_{\epsilon}|^{2/d}}{k^2}\right).
\end{align*}
Let
\begin{align*}
\boldsymbol{l}:=
\begin{pmatrix}
\boldsymbol{0}_{1\times d}&l_{\left(\left\lceil\frac{\left|\Lambda_\epsilon\setminus\{\boldsymbol{0}\}\right|}{p}\right\rceil-1\right)p+1}&l_{\left(\left\lceil\frac{\left|\Lambda_\epsilon\setminus\{\boldsymbol{0}\}\right|}{p}\right\rceil-1\right)p+2}&\cdots&l_{\left|\Lambda_\epsilon\setminus\{\boldsymbol{0}\}\right|}&\boldsymbol{0}_{1\times\left(\left\lceil\frac{\left|\Lambda_\epsilon\setminus\{\boldsymbol{0}\}\right|}{p}\right\rceil p-\left|\Lambda_\epsilon\setminus\{\boldsymbol{0}\}\right|\right)}&1 
\end{pmatrix}\in\mathbb{R}^{1\times(p+d+1)}
\end{align*}
and define
\begin{align*}
f_{NN}(\boldsymbol{x}):&=
\boldsymbol{l}\boldsymbol{\phi}^{\left(\left\lceil\frac{\left|\Lambda_\epsilon\setminus\{\boldsymbol{0}\}\right|}{p}\right\rceil\right)}(\boldsymbol{x})
=\sum_{j=0}^{\left|\Lambda_\epsilon\setminus\{\boldsymbol{0}\}\right|} l_{j} f_{NN,j}^{(d-Lgd)}(\boldsymbol{x})=\sum_{{\boldsymbol{\nu}} \in \Lambda_\epsilon} l_{\boldsymbol{\nu}}(f)f_{NN,\boldsymbol{\nu}}^{(d-Lgd)}(\boldsymbol{x}).
\end{align*}
By the traingle inequality,
\begin{align*}
\left| f(\boldsymbol{x}) - f_{NN}(\boldsymbol{x}) \right|
&\leq\left| f(\boldsymbol{x}) - \sum_{{\boldsymbol{\nu}} \in \Lambda_\epsilon} l_{\boldsymbol{\nu}}(f) L_{\boldsymbol{\nu}}(\boldsymbol{x}) \right|+\left| \sum_{{\boldsymbol{\nu}} \in \Lambda_\epsilon} l_{\boldsymbol{\nu}}(f) L_{\boldsymbol{\nu}}(\boldsymbol{x}) -f_{NN}(\boldsymbol{x})\right|.
\end{align*}
The first term on the right-hand side is bounded by Proposition \ref{Legendre approximation}:
\begin{align*}
\left| f(\boldsymbol{x}) - \sum_{{\boldsymbol{\nu}} \in \Lambda_\epsilon} l_{\boldsymbol{\nu}}(f) L_{\boldsymbol{\nu}}(\boldsymbol{x}) \right|\leq C(d,M,\boldsymbol{\rho}) e^{-C(d,\boldsymbol{\rho}) |\Lambda_\epsilon|^{1/d}}.
\end{align*}
For the second term,
\begin{align*}
\left| \sum_{{\boldsymbol{\nu}} \in \Lambda_\epsilon} l_{\boldsymbol{\nu}}(f) L_{\boldsymbol{\nu}}(\boldsymbol{x}) -f_{NN}(\boldsymbol{x})\right|
&\leq \sum_{{\boldsymbol{\nu}} \in \Lambda_\epsilon\setminus\{\boldsymbol{0}\}} \left|l_{\boldsymbol{\nu}}(f)\right|\left| L_{\boldsymbol{\nu}}(\boldsymbol{x}) -f_{NN,\boldsymbol{\nu}}^{(d-Lgd)}(\boldsymbol{x})\right|\\
&\leq C(d)\left(N^{-C(d)\frac{|\Lambda_\epsilon|^{1/d}}{\log_2q}}+16^{C(d)|\Lambda_\epsilon|^{1/d}}N^{-C(d)\frac{|\Lambda_\epsilon|^{1/d}}{k}}\right)\sum_{{\boldsymbol{\nu}} \in \Lambda_\epsilon\setminus\{\boldsymbol{0}\}} \left|l_{\boldsymbol{\nu}}(f)\right|\\
&\leq C(d,M,\boldsymbol{\rho})\left(N^{-C(d)\frac{|\Lambda_\epsilon|^{1/d}}{\log_2q}}+16^{C(d)|\Lambda_\epsilon|^{1/d}}N^{-C(d)\frac{|\Lambda_\epsilon|^{1/d}}{k}}\right),
\end{align*}
where we empoly \eqref{new final0} in the second step and Lemma \ref{coefficient2} in the third step. Therefore, 
\begin{align*}
\left| f(\boldsymbol{x}) - f_{NN}(\boldsymbol{x}) \right|
&\leq C(d,M,\boldsymbol{\rho})e^{-C(d,\boldsymbol{\rho}) |\Lambda_\epsilon|^{1/d}}+C(d,M,\boldsymbol{\rho})\left(N^{-C(d)\frac{|\Lambda_\epsilon|^{1/d}}{\log_2q}}+16^{C(d)|\Lambda_\epsilon|^{1/d}}N^{-C(d)\frac{|\Lambda_\epsilon|^{1/d}}{k}}\right).
\end{align*}
We complete the proof by setting
\begin{align*}
|\Lambda_\epsilon|\asymp L^d(\log N)^d,\quad k\asymp\log N,\quad p\asymp L^{\lambda}(\log N)^d,\quad q\asymp \sqrt{N}.
\end{align*}

\end{proof}

We are now able to prove Proposition \ref{I}, which is restated below for convenience.

\begin{reproposition1}[restated]
Suppose $L,N$ are sufficiently large and there exist $\kappa\in\left[0,\frac{d}{2}\right],\beta>0$ such that
\begin{align*}
L^{\kappa+\alpha}\leq N\leq e^{L^{\beta}},
\end{align*}
where $\alpha>0$ can be arbitrarily small; when $\kappa=0$, $\alpha$ can be $0$. There exists a neural network $f_{NN}:\mathbb{R}^d\to\mathbb{R}$ with width $C(d,\boldsymbol{\rho})N$ and depth $C(d,\boldsymbol{\rho},\beta)L$ such that for any $\boldsymbol{x}\in\left[-1,1\right]^d$, there holds
\begin{align*}
\left| f(\boldsymbol{x}) - f_{NN}(\boldsymbol{x}) \right|
&\leq C(d,\boldsymbol{\rho},M)N^{-C(d,\boldsymbol{\rho},\kappa,\alpha)L^{\frac{\kappa+1}{d+2}}}.
\end{align*}
\end{reproposition1}
\begin{proof}
\textbf{Case $\kappa\in\left(0,\frac{d}{2}\right]$.} Choosing $\lambda=\frac{\kappa(d+2)}{\kappa+1}\in(0,d]$ and replacing $L$ with ${L}^{\frac{\kappa+1}{d+2}}$ and $N$ with $N^{\frac{\alpha}{2(\kappa+\alpha)}}$ in Lemma \ref{new final}, we conclude that there exists a neural network $f_{NN}:\mathbb{R}^d\to\mathbb{R}$ such that for any $\boldsymbol{x}\in\left[-1,1\right]^d$, there holds
\begin{align*}
\left| f(\boldsymbol{x}) - f_{NN}(\boldsymbol{x}) \right|
&\leq C(d,\boldsymbol{\rho},M)N^{-C(d,\boldsymbol{\rho},\kappa,\alpha)L^{\frac{\kappa+1}{d+2}}}.
\end{align*}
Since $L^{\kappa+\alpha}\leq N\leq e^{L^{\beta}}$, the width of $f_{NN}$ is bounded as
\begin{align*}
C(d,\boldsymbol{\rho})\left(L^{\frac{\kappa+1}{d+2}}\right)^{\frac{\kappa(d+2)}{\kappa+1}}N^{\frac{\alpha}{2(\kappa+\alpha)}}(\log N)^{d}&\leq C(d,\boldsymbol{\rho})L^{\kappa}N^{\frac{\alpha}{2(\kappa+\alpha)}}N^{\frac{\alpha}{2(\kappa+\alpha)}}\\
&\leq C(d,\boldsymbol{\rho})N^{\frac{\kappa}{\kappa+\alpha}}N^{\frac{\alpha}{2(\kappa+\alpha)}}N^{\frac{\alpha}{2(\kappa+\alpha)}}=C(d,\boldsymbol{\rho})N
\end{align*}
and the depth is bounded as
\begin{align*}
C(d,\boldsymbol{\rho},\beta)\left(L^{\frac{\kappa+1}{d+2}}\right)^{\frac{d-\kappa+1}{\kappa+1}}\left(L^{\frac{\kappa+1}{d+2}}+\log\log N\right)\leq C(d,\boldsymbol{\rho},\beta){L}^{\frac{d-\kappa+1}{d+2}}\left({L}^{\frac{\kappa+1}{d+2}}+\log L\right)\leq C(d,\boldsymbol{\rho},\beta)L.
\end{align*}

\noindent\textbf{Case $\kappa=0$.} Choosing $\lambda=0$ and replacing $L$ with ${L}^{\frac{1}{d+2}}$ and $N$ with $N^{1/2}$ in Lemma \ref{new final}, we conclude that there exists a neural network $f_{NN}:\mathbb{R}^d\to\mathbb{R}$ such that for any $\boldsymbol{x}\in\left[-1,1\right]^d$, there holds
\begin{align*}
\left| f(\boldsymbol{x}) - f_{NN}(\boldsymbol{x}) \right|
&\leq C(d,\boldsymbol{\rho},M)N^{-C(d,\boldsymbol{\rho})L^{\frac{1}{d+2}}}.
\end{align*}
Since $C\leq N\leq e^{L^{\beta}}$, the width of $f_{NN}$ is 
\begin{align*}
C(d,\boldsymbol{\rho})N^{1/2}(\log N)^{d}&\leq C(d,\boldsymbol{\rho})N
\end{align*}
and the depth is
\begin{align*}
C(d,\boldsymbol{\rho}){L}^{\frac{d+1}{d+2}}\left(L^{\frac{1}{d+2}}+\log\log N\right)\leq C(d,\boldsymbol{\rho},\beta){L}^{\frac{d+1}{d+2}}\left({L}^{\frac{1}{d+2}}+\log L\right)\leq C(d,\boldsymbol{\rho},\beta)L.
\end{align*}

\end{proof}

\subsection{Proof of Proposition \ref{II}}\label{proof of II}

Also based on Lemma \ref{power approximation}, here we introduce an alternative approach, distinct from that in Lemma \ref{newLegendre}, to approximate univariate Legendre polynomials. Unlike the network constructed in Lemma \ref{newLegendre}, which can only approximate a single $L_{i}(x)$, the network we construct here is capable of simultaneously approximating $\{L_i(x)\}_{i \in [\nu]}$, thereby achieving higher efficiency. 
An illustrative example is provided in Figure \ref{figure prop2-1} for the case of $\nu=8$, where the powers of $x$ shown in the figure are actually their approximations in the construction process.

\begin{figure}[htbp]
    \centering

    \resizebox{0.7\textwidth}{!}{
        \begin{tikzpicture}[
            line/.style={black, thick}
        ]

        \node (x0) at (0, -0.5) {$x$};

        \node (c1_1) at (2.5, 2.0) {$x$};
        \node (c1_2) at (2.5, 1.1) {$x^2$};
        \node (c1_3) at (2.5, -0.25) {$x^2$};
        \node (c1_4) at (2.5, -3.15) {$x^2$};

        \draw[line] (x0) -- (c1_1);
        \draw[line] (x0) -- (c1_2);
        \draw[line] (x0) -- (c1_3);
        \draw[line] (x0) -- (c1_4);

        \node (c2_1) at (5.0, 2.0) {$x$};
        \node (c2_2) at (5.0, 1.1) {$x^2$};
        \node (c2_3) at (5.0, 0.2) {$x^3$};
        \node (c2_4) at (5.0, -0.7) {$x^4$};
        \node (c2_5) at (5.0, -3.15) {$x^4$};

        \draw[line] (c1_1) -- (c2_1);

        \draw[line] (c1_1) -- (c2_3);

        \draw[line] (c1_2) -- (c2_2);

        \draw[line] (c1_2) -- (c2_4);

        \draw[line] (c1_3) -- (c2_3);
        \draw[line] (c1_3) -- (c2_4);

        \draw[line] (c1_4) -- (c2_5);

        \node (c3_1) at (8.0, 2.0) {$x$};
        \node (c3_2) at (8.0, 1.1) {$x^2$};
        \node (c3_3) at (8.0, 0.2) {$x^3$};
        \node (c3_4) at (8.0, -0.7) {$x^4$};
        \node (c3_5) at (8.0, -1.8) {$x^5$};
        \node (c3_6) at (8.0, -2.7) {$x^6$};
        \node (c3_7) at (8.0, -3.6) {$x^7$};
        \node (c3_8) at (8.0, -4.5) {$x^8$};

        \draw[line] (c2_1) -- (c3_1);
        \draw[line] (c2_2) -- (c3_2);
        \draw[line] (c2_3) -- (c3_3);
        \draw[line] (c2_4) -- (c3_4);

        \draw[line] (c2_1) -- (c3_5);
        \draw[line] (c2_2) -- (c3_6);
        \draw[line] (c2_3) -- (c3_7);
        \draw[line] (c2_4) -- (c3_8);

        \draw[line] (c2_5) -- (c3_5);
        \draw[line] (c2_5) -- (c3_6);
        \draw[line] (c2_5) -- (c3_7);
        \draw[line] (c2_5) -- (c3_8);

        \node (L1) at (12.5, 2.0) {$L_1(x)$};
        \node (L2) at (12.5, 1.1) {$L_2(x)$};
        \node (L3) at (12.5, 0.2) {$L_3(x)$};
        \node (L_dots) at (12.5, -2.0) {$\vdots$};
        \node (L8) at (12.5, -4.5) {$L_8(x)$};

        \draw[line] (c3_1) -- (L1);
        \draw[line] (c3_1) -- (L2);
        \draw[line] (c3_1) -- (L3);

        \draw[line] (c3_2) -- (L2);
        \draw[line] (c3_2) -- (L3);

        \draw[line] (c3_3) -- (L3);

        \draw[line] (c3_1) -- (L8);
        \draw[line] (c3_2) -- (L8);
        \draw[line] (c3_3) -- (L8);
        \draw[line] (c3_4) -- (L8);
        \draw[line] (c3_5) -- (L8);
        \draw[line] (c3_6) -- (L8);
        \draw[line] (c3_7) -- (L8);
        \draw[line] (c3_8) -- (L8);

        \end{tikzpicture}
    }
    \caption{An illustration of the proof of Lemma \ref{Legendre}.}
    \label{figure prop2-1}
\end{figure}
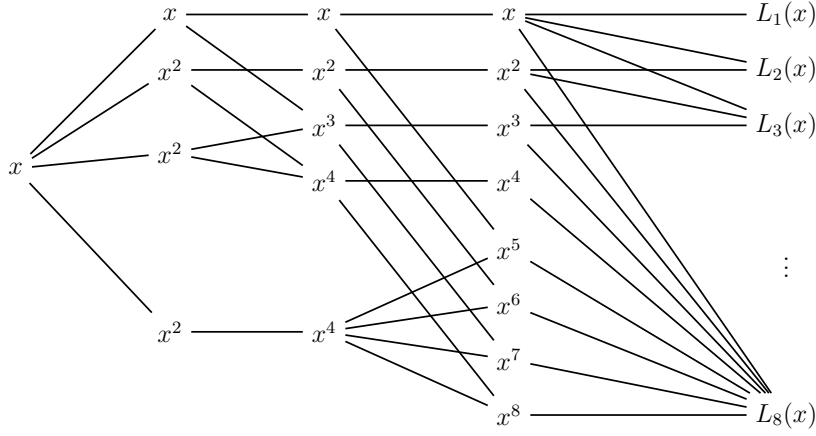

\begin{lemma}\label{Legendre}
Let $\nu\in\mathbb{N}_{\geq1}$. Suppose $N^L\geq\left(300\cdot4^{\left\lceil\log_2\nu\right\rceil-1}\right)^{1/5}$. There exists a neural network $\boldsymbol{f}_{NN,1\to\nu}^{(Lgd)}:\mathbb{R}\to\mathbb{R}^{\nu}$ with width $2^{\left\lceil\log_2\nu\right\rceil}\cdot10N+2$ and depth $5\left\lceil\log_2\nu\right\rceil^2L$ such that for any ${x}\in\left[-1,1\right]$ and any $i\in[\nu]$, there holds
\begin{align*}
\left|f_{NN,i}^{(Lgd)}(x)-L_{i}(x)\right|
\leq\frac{15}{2}\cdot16^{\nu}N^{-5L}.
\end{align*}
Here $f_{NN,i}^{(Lgd)}(x)$ denotes the $i$-th component of $\boldsymbol{f}_{NN,1\to\nu}^{(Lgd)}(x)$. 
\end{lemma}
\begin{proof}

We inductively show that for $i\in[\left\lceil\log_2\nu\right\rceil]$, there exists a neural network $\boldsymbol{\psi}^{(i)}:\mathbb{R}\to\mathbb{R}^{2^{i}}$ with width $2^i\cdot10N+2$ and depth $5i^2L$ such that for any $x\in\left[-1,1\right]$, there holds
\begin{align}\label{Legendre1}
\left|\psi_{j}^{(i)}(x)-x^{j}\right|\leq 30\cdot4^{i-1}N^{-5L},\quad j\in\left[2^{i}\right].
\end{align}
By setting $r=2,n=2$ in Lemma \ref{power approximation}, we find a neural network $g_1:\mathbb{R}\to\mathbb{R}$ with width $20N$ and depth $5L$ such that for any $x\in[-1,1]$, there holds
\begin{align*}
\left|g_1(x)-x^{2}\right|\leq 30N^{-5L}.
\end{align*}
Define
\begin{align*}
\boldsymbol{\psi}^{(1)}(x):=\begin{pmatrix}
x\\
g_1(x)
\end{pmatrix}\in\mathbb{R}^{2}.
\end{align*}
Then $\boldsymbol{\psi}^{(1)}$ can be expressed as a neural network with width $20N+2$ and depth $5L$ and
\begin{align*}
\left|\psi_{j}^{(1)}(x)-x^{j}\right|\leq 30N^{-5L},\quad j\in\left[2\right].
\end{align*}
Next, we assume \eqref{Legendre1} is valid for the case of $i$ and prove the case of $i+1$. By setting $r=2,n=2^{i}$ in Lemma \ref{power approximation}, we find a neural network $g_i:\mathbb{R}\to\mathbb{R}$ with width $10\cdot2^{i}N$ and depth $5i^2L$ such that for any $x\in[0,1]$, there holds
\begin{align}\label{Legendre1.5}
\left|g_i(x)-x^{2^{i}}\right|\leq 30N^{-5L}.
\end{align}
Let $f_{NN,2}^{(mtp)}:\mathbb{R}^2\to\mathbb{R}$ be the neural network in Lemma \ref{multiplication 2} with $b=1.1,a=-1.1$. Then $f_{NN,2}^{(mtp)}$ is of width $10N$ and depth $5L$ and 
\begin{align}\label{Legendre2}
|f_{NN,2}^{(mtp)}(x_1, x_2) - x_1x_2| \leq 30N^{-5L},\quad\forall x_1,x_2\in[-1.1,1.1].
\end{align}
$\boldsymbol{\psi}^{(i+1)}$ is defined as
\begin{align*}
\boldsymbol{\psi}^{(i+1)}(x):=
\begin{pmatrix}
\boldsymbol{\psi}^{(i)}\\
f_{NN,2}^{(mtp)}\left({\psi}_1^{(i)},g_i\right)\\
f_{NN,2}^{(mtp)}\left({\psi}_2^{(i)},g_i\right)\\
\vdots\\
f_{NN,2}^{(mtp)}\left({\psi}_{2^{i}}^{(i)},g_i\right)\\
\end{pmatrix}\in\mathbb{R}^{2^{i+1}}.
\end{align*}
By the induction assumption, the width of $\boldsymbol{\psi}^{(i+1)}$ is 
\begin{align*}
2^{i}\cdot10N+2+2^{i}\cdot10N=2^{i+1}\cdot10N+2
\end{align*}
and the depth is 
\begin{align*}
5i^2L+5L\leq 5(i+1)^2L.
\end{align*}
In the following, we verify $\boldsymbol{\psi}^{(i+1)}$ defined above satisfies \eqref{Legendre1}. For the first $2^i$ components, due to the induction assumption, we have 
\begin{align*}
&\left|\psi_{j}^{(i+1)}(x)-x^{j}\right|=\left|\psi_{j}^{(i)}(x)-x^{j}\right|\leq 30\cdot4^{i-1}N^{-5L}\leq30\cdot4^{i}N^{-5L},\quad j\in\left[2^i\right].
\end{align*}
For the last $2^i$ components, by the triangle inequality,
\begin{align}\label{Legendre2.5}
&\left|\psi_{2^i+j}^{(i+1)}(x)-x^{2^i+j}\right|\leq\left|f_{NN,2}^{(mtp)}\left(\psi_{j}^{(i)}(x),g_i(x)\right)-\psi_{j}^{(i)}(x)g_i(x)\right|+\left|\psi_{j}^{(i)}(x)g_i(x)-x^{2^i+j}\right|
\end{align}
We aim to employ \eqref{Legendre2} to bound the first term on the right hand side. Thereby, we need to check  
\begin{align}\label{Legendre3}
\left|\psi_{j}^{(i)}(x)\right|,\left|g_i(x)\right|\leq1.1.
\end{align}
For $\left|g_i(x)\right|$, by \eqref{Legendre1.5} and the condition $N^L\geq\left(300\cdot4^{\left\lceil\log_2\nu\right\rceil-1}\right)^{1/5}$, we have 
\begin{align*}
\left|g_i(x)\right|\leq\left|g_i(x)-{x}^{2^i}\right|+\left|{x}^{2^i}\right|&\leq 30N^{-5L}+1\leq0.1+1\leq1.1.
\end{align*}
For $\left|\psi_{j}^{(i)}(x)\right|$, by the induction assumption and the condition $N^L\geq\left(300\cdot4^{\left\lceil\log_2\nu\right\rceil-1}\right)^{1/5}$, we have
\begin{align*}
\left|\psi_{j}^{(i)}(x)\right|\leq\left|\psi_{j}^{(i)}(x)-x^{j}\right|+\left|x^{j}\right|\leq 30\cdot4^{i-1}N^{-5L}+1\leq30\cdot4^{\left\lceil\log_2\nu\right\rceil-1}N^{-5L}+1\leq0.1+1=1.1.
\end{align*}
Therefore, \eqref{Legendre2} leads to
\begin{align}\label{Legendre4}
\left|f_{NN,2}^{(mtp)}\left(\psi_{j}^{(i)}(x),g_i(x)\right)-\psi_{j}^{(i)}(x)g_i(x)\right|\leq 30N^{-5L}.
\end{align}
For the second term on the right-hand side of \eqref{Legendre2.5}, 
\begin{align}
\left|\psi_{j}^{(i)}(x)g_i(x)-x^{2^i+j}\right|
&\leq\left|\psi_{j}^{(i)}(x)g_i(x)-x^{j}g_i(x)\right|+\left|x^jg_i(x)-x^{2^i+j}\right|\nonumber\\
&\leq\left|\psi_{j}^{(i)}(x)-x^{j}\right|\left|g_i(x)\right|+\left|g_i(x)-x^{2^i}\right|\nonumber\\
&\leq 1.1\cdot30\cdot4^{i-1}N^{-5L}+30N^{-5L},\label{Legendre5}
\end{align}
where we make use of the induction assumption, \eqref{Legendre3} and \eqref{Legendre1.5} in the third step. Plugging \eqref{Legendre4} and \eqref{Legendre5} into \eqref{Legendre2.5}, we get
\begin{align*}
\left|\psi_{2^i+j}^{(i+1)}(x)-x^{2^i+j}\right|
&\leq30N^{-5L}+1.1\cdot30\cdot4^{i-1}N^{-5L}+30N^{-5L}\\
&=30(1.1\cdot4^{i-1}+2)N^{-5L}\leq30\cdot4^{i}N^{-5L}.
\end{align*}
Hence, the induction is completed and we construct a neural network $\boldsymbol{\psi}^{(\left\lceil\log_2\nu\right\rceil)}(x):\mathbb{R}\to\mathbb{R}^{2^{\left\lceil\log_2\nu\right\rceil}}$ with width $2^{\left\lceil\log_2\nu\right\rceil}\cdot10N+2$ and depth $5\left\lceil\log_2\nu\right\rceil^2L$ such that for any $x\in\left[-1,1\right]$, there holds 
\begin{align*}
\left| {\psi}_{j}^{(\left\lceil\log_2\nu\right\rceil)}(x)-x^{j}\right|\leq 30\cdot 4^{\left\lceil\log_2\nu\right\rceil-1}N^{-5L},\quad j\in\left[2^{\left\lceil\log_2\nu\right\rceil}\right].
\end{align*}
Now, $\boldsymbol{f}_{NN,1\to\nu}^{(Lgd)}$ is defined as
\begin{align*}
\boldsymbol{f}_{NN,1\to\nu}^{(Lgd)}(x):=\boldsymbol{A}\boldsymbol{\psi}^{(\left\lceil\log_2\nu\right\rceil)}(x)+\boldsymbol{b},
\end{align*}
where $\boldsymbol{A}$ is a $\nu\times 2^{\left\lceil\log_2\nu\right\rceil}$ matrix with entries defined as 
\begin{align*}
A_{ij}:=\left\{\begin{matrix}
c_{j}^{(i)},  & \text{if }j\leq i\\
0,  & \text{otherwise}
\end{matrix}\right.,\qquad i\in[\nu],j\in\left[2^{\left\lceil\log_2\nu\right\rceil}\right]
\end{align*}
and $\boldsymbol{b}$ is a $\nu$-dimensional vector with components  
$b_i:=c_0^{(i)},i\in[\nu]$. It follows that for any $i\in[\nu]$,
\begin{align*}
f_{NN,i}^{(Lgd)}(x)=\left[\boldsymbol{f}_{NN,1\to\nu}^{(Lgd)}(x)\right]_i=\sum_{j=1}^{i} c_{j}^{(i)} {\psi}_{j}^{(\left\lceil\log_2\nu\right\rceil)}(x)
\end{align*}
and hence
\begin{align*}
\left|f_{NN,i}^{(Lgd)}(x)-L_\nu(x)\right|
&=\left|\sum_{j=1}^{i} c_{j}^{(i)} {\psi}_{j}^{(\left\lceil\log_2\nu\right\rceil)}(x)-\sum_{j=1}^{i} c_{j}^{(i)} x^{j}\right|\\
&\leq\sum_{j=1}^{i} \left|c_{j}^{(i)}\right|\left| {\psi}_{j}^{(\left\lceil\log_2\nu\right\rceil)}(x)-x^{j}\right|\\
&\leq 30\cdot 4^{\left\lceil\log_2\nu\right\rceil-1}N^{-5L}\sum_{j=1}^{i} \left|c_{j}^{(i)}\right|\\
&\leq30\cdot 4^{\nu+\left\lceil\log_2\nu\right\rceil-1}N^{-5L}\\
&\leq\frac{15}{2}\cdot16^{\nu}N^{-5L},
\end{align*}
where we apply Lemma \ref{coefficient1} in the fourth step.

\end{proof}

We now proceed to construct the network for approximating analytic functions. The construction is divided into two steps. The first step involves constructing the approximation networks for multivariate Legendre polynomials, denoted by $\left\{f_{NN,\boldsymbol{\nu}}^{(d-Lgd)}(\boldsymbol{x})\right\}_{\boldsymbol{\nu}\in\Lambda_\epsilon}$, based on the approximation networks of univariate Legendre polynomials and Lemma \ref{multiplication d}. This procedure is almost identical to the proof of Lemma \ref{new multivariate Legendre}. In the second step, we embed $\left\{f_{NN,\boldsymbol{\nu}}^{(d-Lgd)}(\boldsymbol{x})\right\}_{\boldsymbol{\nu}\in\Lambda_\epsilon}$ into a larger network architecture and employ the ``squashing" technique, which is also utilized in the construction for Lemma \ref{new final}, to implement the linear combination of $\left\{f_{NN,\boldsymbol{\nu}}^{(d-Lgd)}(\boldsymbol{x})\right\}_{\boldsymbol{\nu}\in\Lambda_\epsilon}$.
An illustration of the construction is provided in Figure \ref{figure prop2-2}, where the black lines represent the contruction of the approximants of multivariate Legendre polynomials, while the blue lines denote their linear combination. For simplicity, the depicted case assumes $p$ is divisible by $\left|\Lambda_\epsilon\right|$.

\begin{figure}[htbp]
    \centering

    \resizebox{\textwidth}{!}{
        \begin{tikzpicture}[
            blackline/.style={black, thick},
            redline/.style={blue, thick}
        ]

        \node (src) at (-0.0, -0.8){$\boldsymbol{x}$};

        \node (c1_1) at (3.0, 2.0) {$\boldsymbol{f}_{NN}^{(prl-Lgd)}$};
        \node (c1_2) at (3.0, 0.8) {$f_{NN, 1}^{(d-Lgd)}$};
        \node (c1_3) at (3.0, -0.2) {$f_{NN, 2}^{(d-Lgd)}$};
        \node (c1_dots) at (3.0, -1.0) {$\vdots$};
        \node (c1_4) at (3.0, -2.2) {$f_{NN, p}^{(d-Lgd)}$};
        \node (c1_0) at (3.0, -3.8) {$0$};
        \node (phi1) at (3.0, -5.5) {$\boldsymbol{\phi}^{(1)}(\boldsymbol{x})$};
        \node (c1_5) at (2.5, 1.8) {};

        \draw[blackline] (src) -- (c1_5);
        \draw[blackline] (src) -- (c1_2);
        \draw[blackline] (src) -- (c1_3);
        \draw[blackline] (src) -- (c1_4);
        \draw[blackline] (src) -- (c1_0);

        \node (c2_1) at (7.0, 2.0) {$\boldsymbol{f}_{NN}^{(prl-Lgd)}$};
        \node (c2_2) at (7.0, 0.8) {$f_{NN, p+1}^{(d-Lgd)}$};
        \node (c2_3) at (7.0, -0.2) {$f_{NN, p+2}^{(d-Lgd)}$};
        \node (c2_dots) at (7.0, -1.0) {$\vdots$};
        \node (c2_4) at (7.0, -2.2) {$f_{NN, 2p}^{(d-Lgd)}$};
        \node (c2_sum) at (7.0, -3.8) {$\sum_{j=1}^{p} l_j f_{NN, j}^{(d-Lgd)}$};
        \node (phi2) at (7.0, -5.5) {$\boldsymbol{\phi}^{(2)}(\boldsymbol{x})$};

        \draw[blackline] (c1_1) -- (c2_1);
        \draw[blackline] (c1_1) -- (c2_2);
        \draw[blackline] (c1_1) -- (c2_3);
        \draw[blackline] (c1_1) -- (c2_4);

        \draw[redline] (c1_2) -- (c2_sum);
        \draw[redline] (c1_3) -- (c2_sum);
        \draw[redline] (c1_4) -- (c2_sum);
        \draw[redline] (c1_0) -- (c2_sum);

        \node (dots_center) at (9.5, -0.8) {$\dots$};

        \node (c3_1) at (12.0, 2.0) {$\boldsymbol{f}_{NN}^{(prl-Lgd)}$};
        \node (c3_2) at (12.0, 0.8) {$f_{NN, |\Lambda|-2p+1}^{(d-Lgd)}$};
        \node (c3_3) at (12.0, -0.2) {$f_{NN, |\Lambda|-2p+2}^{(d-Lgd)}$};
        \node (c3_dots) at (12.0, -1.0) {$\vdots$};
        \node (c3_4) at (12.0, -2.2) {$f_{NN, |\Lambda|-p}^{(d-Lgd)}$};
        \node (c3_sum) at (12.0, -3.8) {$\sum_{j=1}^{|\Lambda|-2p} l_j f_{NN, j}^{(d-Lgd)}$};
        \node (phi3) at (12.0, -5.5) {$\boldsymbol{\phi}^{\left(\frac{\Lambda}{p}-1\right)}(\boldsymbol{x})$};

        \node (c4_1) at (16.5, 2.0) {$\boldsymbol{f}_{NN}^{(prl-Lgd)}$};
        \node (c4_2) at (16.5, 0.8) {$f_{NN, |\Lambda|-p+1}^{(d-Lgd)}$};
        \node (c4_3) at (16.5, -0.2) {$f_{NN, |\Lambda|-p+2}^{(d-Lgd)}$};
        \node (c4_dots) at (16.5, -1.0) {$\vdots$};
        \node (c4_4) at (16.5, -2.2) {$f_{NN, |\Lambda|}^{(d-Lgd)}$};
        \node (c4_sum) at (16.5, -3.8) {$\sum_{j=1}^{|\Lambda|-p} l_j f_{NN, j}^{(d-Lgd)}$};
        \node (phi4) at (16.5, -5.5) {$\boldsymbol{\phi}^{\left(\frac{\Lambda}{p}\right)}(\boldsymbol{x})$};

        \draw[blackline] (c3_1) -- (c4_1);
        \draw[blackline] (c3_1) -- (c4_2);
        \draw[blackline] (c3_1) -- (c4_3);
        \draw[blackline] (c3_1) -- (c4_4);

        \draw[redline] (c3_2) -- (c4_sum);
        \draw[redline] (c3_3) -- (c4_sum);
        \draw[redline] (c3_4) -- (c4_sum);
        \draw[redline] (c3_sum) -- (c4_sum);

        \node (L_Lambda) at (19.0, -0.8) {$f_{NN}$};
        
        \draw[redline] (c4_1) -- (L_Lambda);
        \draw[redline] (c4_2) -- (L_Lambda);
        \draw[redline] (c4_3) -- (L_Lambda);
        \draw[redline] (c4_4) -- (L_Lambda);
        \draw[redline] (c4_sum) -- (L_Lambda);

        \end{tikzpicture}
    }

    \caption{An illustration of the proof of Lemma \ref{final}.}
    \label{figure prop2-2}
\end{figure}
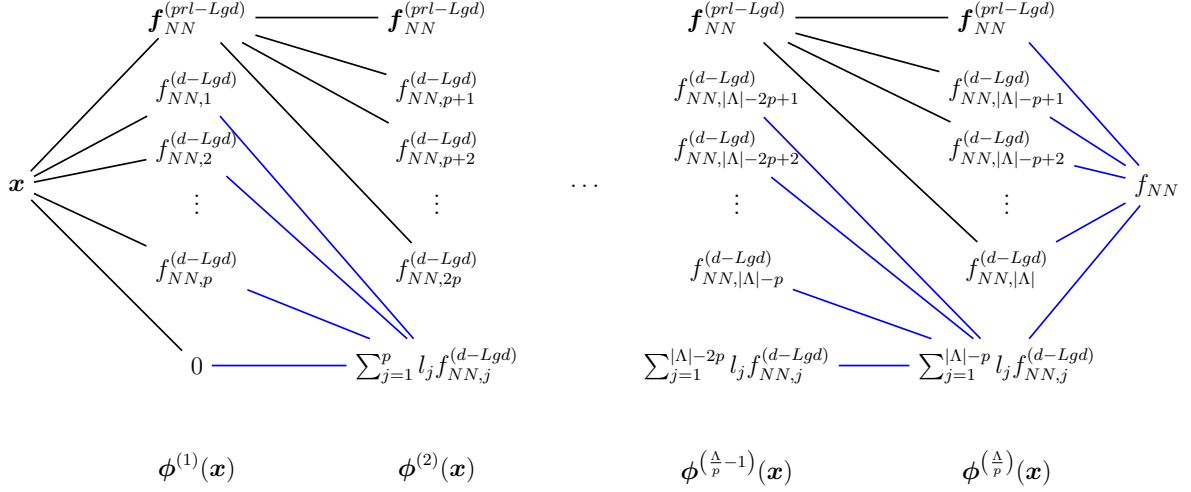

\begin{lemma}\label{final}
Let $\lambda\in[0,d]$. For sufficiently large $L$ and $N$, there exists a neural network $f_{NN}:\mathbb{R}^d\to\mathbb{R}$ with width $C(d,\boldsymbol{\rho})L^{\lambda\vee 1}N(\log N)^d$ and depth $C(d,\boldsymbol{\rho})L^{d-\lambda+1}+C(d,\boldsymbol{\rho})L(\log L+\log\log N)^2$ such that for any $\boldsymbol{x}\in\left[-1,1\right]^d$, there holds
\begin{align*}
\left| f(\boldsymbol{x}) - f_{NN}(\boldsymbol{x}) \right|
&\leq C(d,\boldsymbol{\rho},M)N^{-C(d,\boldsymbol{\rho})L}.
\end{align*}
\end{lemma}
\begin{proof}
The proof is divided into two steps. 

\textbf{Step 1.} In the first step, we construct neural networks $\left\{f_{NN,\boldsymbol{\nu}}^{(d-Lgd)}(\boldsymbol{x})\right\}_{\boldsymbol{\nu}\in\Lambda_\epsilon\setminus\{\boldsymbol{0}\}}$ that approximate multivariate Legendre polynomials based on the obtained univariate Legendre polynomials approximants.

Let $\boldsymbol{f}_{NN,1\to\bar{\nu}}^{(Lgd)}:\mathbb{R}\to\mathbb{R}^{\bar{\nu}}$ be the neural network constructed in Lemma \ref{Legendre} with 
\begin{align*}
\bar{\nu}:=C(d)|\Lambda_{\epsilon}|^{1/d}\geq \|\boldsymbol{\nu}\|_1,\quad \forall{\boldsymbol{\nu}} \in \Lambda_\epsilon,
\end{align*}
where we employ Lemma \ref{m estimate}. We also replace $N$ with $\sqrt{N}$ (assume without loss of generality $\sqrt{N}$ is an integer) in Lemma \ref{Legendre}. Hence, it is of width $2^{\left\lceil\log_2\bar{\nu}\right\rceil}\cdot10\sqrt{N}+2$ and depth $5\left\lceil\log_2\bar{\nu}\right\rceil^2L$. Define $\boldsymbol{f}_{NN}^{(prl-Lgd)}:\mathbb{R}^d\to\mathbb{R}^{d\bar{\nu}+1}$ as the parallelization of $\boldsymbol{f}_{NN,1\to\bar{\nu}}^{(Lgd)}$ with $d$ inputs:
\begin{align*}
\boldsymbol{f}_{NN}^{(prl-Lgd)}(\boldsymbol{x}):=\begin{pmatrix}
\boldsymbol{f}_{NN,1\to\bar{\nu}}^{(Lgd)}(x_1)\\
\boldsymbol{f}_{NN,1\to\bar{\nu}}^{(Lgd)}(x_2)\\
\vdots\\
\boldsymbol{f}_{NN,1\to\bar{\nu}}^{(Lgd)}(x_d)\\
1
\end{pmatrix}
=\begin{pmatrix}
f_{NN,1}^{(Lgd)}(x_1)\\
\vdots\\
f_{NN,\bar{\nu}}^{(Lgd)}(x_1)\\
f_{NN,1}^{(Lgd)}(x_2)\\
\vdots\\
f_{NN,\bar{\nu}}^{(Lgd)}(x_2)\\
\vdots\\
f_{NN,1}^{(Lgd)}(x_d)\\
\vdots\\
f_{NN,\bar{\nu}}^{(Lgd)}(x_d)\\
1
\end{pmatrix}.
\end{align*}
It follows that the width of $\boldsymbol{f}_{NN}^{(prl-Lgd)}$ is $2^{\left\lceil\log_2\bar{\nu}\right\rceil}\cdot10d\sqrt{N}+2d+1$ and the depth is $5\left\lceil\log_2\bar{\nu}\right\rceil^2L$.
For any $\boldsymbol{\nu}\in\Lambda_{\epsilon}$, there exists a linear mapping $\mathcal{H}_{\boldsymbol{\nu}}:\mathbb{R}^{d\bar{\nu}}\to\mathbb{R}^d$ that extracts the components associated with $\boldsymbol{\nu}$ from $\boldsymbol{f}_{NN}^{(prl-Lgd)}$:
\begin{align*}
\mathcal{H}_{\boldsymbol{\nu}}\left(\boldsymbol{f}_{NN}^{(prl-Lgd)}(\boldsymbol{x})\right)=\begin{pmatrix}
f_{NN,\nu_1}^{(Lgd)}(x_1)\\
f_{NN,\nu_2}^{(Lgd)}(x_2)\\
\vdots\\
f_{NN,\nu_d}^{(Lgd)}(x_d)
\end{pmatrix}.
\end{align*}
Note that here we use the notation
\begin{align}\label{multivariate Legendre0.75}
f_{NN,{0}}^{(Lgd)}(x):=L_0(x)=1.
\end{align}
Let $f_{NN,d}^{(mtp)}:\mathbb{R}^{d}\to\mathbb{R}$ be the neural network in Lemma \ref{multiplication d} with $D=q,H=\left\lceil\frac{4\bar{\nu}}{\log_2q}\right\rceil$, where $q\geq2$ is a tunable parameter determined later. We also replace $N$ with $\sqrt{N}$ in Lemma \ref{multiplication d}. Then $f_{NN,d}^{(mtp)}$ is of width $5\cdot2^{\lceil\log_2d\rceil}q\sqrt{N}$ and depth $2d\lceil\log_2d\rceil^2\left\lceil\frac{4\bar{\nu}}{\log_2q}\right\rceil$, and
\begin{align}\label{multivariate Legendre1}
\left|f_{NN,d}^{(mtp)}(\boldsymbol{x})-x_1\cdots x_{d}\right|\leq 96N^{-d\lceil\log_2d\rceil \left\lceil\frac{4\bar{\nu}}{\log_2q}\right\rceil},\quad\forall \boldsymbol{x}\in\left[-q^{\left\lceil\frac{4\bar{\nu}}{\log_2q}\right\rceil},q^{\left\lceil\frac{4\bar{\nu}}{\log_2q}\right\rceil}\right]^d.
\end{align}
$f_{NN,\boldsymbol{\nu}}^{(d-Lgd)}$ is now defined as
\begin{align}\label{d-Lgd}
f_{NN,\boldsymbol{\nu}}^{(d-Lgd)}(\boldsymbol{x}):&=f_{NN,d}^{(mtp)}\left(\mathcal{H}_{\boldsymbol{\nu}}\left(\boldsymbol{f}_{NN}^{(prl-Lgd)}(\boldsymbol{x})\right)\right)\nonumber\\
&=f_{NN,d}^{(mtp)}\left(f_{NN,{\nu_1}}^{(Lgd)}({x}_1),f_{NN,{\nu_2}}^{(Lgd)}({x}_2),\cdots,f_{NN,{\nu_d}}^{(Lgd)}({x}_d)\right).
\end{align}
By the triangle inequality,
\begin{align}\label{multivariate Legendre2}
\left|f_{NN,\boldsymbol{\nu}}^{(d-Lgd)}(\boldsymbol{x})-L_{\boldsymbol{\nu}}(\boldsymbol{x})\right|&\leq\left|f_{NN,\boldsymbol{\nu}}^{(d-Lgd)}(\boldsymbol{x})-\prod_{i=1}^{d}f_{NN,{\nu_i}}^{(Lgd)}({x}_i)\right|+\left|\prod_{i=1}^{d}f_{NN,{\nu_i}}^{(Lgd)}({x}_i)-\prod_{i=1}^dL_{\nu_i}(x_i)\right|.
\end{align}
We aim to employ \eqref{multivariate Legendre1} to bound the first term on the right hand side. To this end, we need to check that for $i\in[d]$,
\begin{align*}
\left|f_{NN,{\nu}_i}^{(Lgd)}(x_i)\right|&\leq q^{\left\lceil\frac{4\bar{\nu}}{\log_2q}\right\rceil}.
\end{align*}
Since $|x_i|\leq1$,
\begin{align}\label{multivariate Legendre2.25}
\left|L_{\nu_i}(x_i)\right|=\left|\sum_{\ell=0}^{\nu_i}c_{\ell}^{(\nu_i)}x_i^{\ell}\right|\leq\sum_{\ell=0}^{\nu_i}\left|c_{\ell}^{(\nu_i)}\right|\leq 4^{\nu_i}.
\end{align}
It follows that
\begin{align}
\left|f_{NN,{\nu}_i}^{(Lgd)}(x_i)\right|&\leq\left|f_{NN,{\nu}_i}^{(Lgd)}(x_i)-L_{\nu_i}(x_i)\right|+\left|L_{\nu_i}(x_i)\right|
\leq\frac{15}{2}\cdot16^{\nu_i}N^{-5L/2}+4^{\nu_i}\nonumber\\
&\leq\frac{1}{2}\cdot16^{\nu_i}+\frac{1}{2}\cdot16^{\nu_i}=16^{\nu_i}\leq16^{\bar{\nu}}=q^{\frac{4\bar{\nu}}{\log_2q}}\leq q^{\left\lceil\frac{4\bar{\nu}}{\log_2q}\right\rceil},\label{multivariate Legendre2.5}
\end{align}
where we employ Lemma \ref{Legendre} and \eqref{multivariate Legendre0.75} in the second step. Therefore, \eqref{multivariate Legendre1} leads to
\begin{align}\label{multivariate Legendre3}
\left|f_{NN,\boldsymbol{\nu}}^{(d-Lgd)}(\boldsymbol{x})-\prod_{i=1}^{d}f_{NN,{\nu_i}}^{(Lgd)}({x}_i)\right|\leq
96N^{-d\lceil\log_2d\rceil\left\lceil\frac{4\bar{\nu}}{\log_2q}\right\rceil }.
\end{align}
For the second term on the right-hand side of \eqref{multivariate Legendre2}, 
\begin{align}
&\left|\prod_{i=1}^{d}f_{NN,{\nu_i}}^{(Lgd)}({x}_i)-\prod_{i=1}^dL_{\nu_i}(x_i)\right|\nonumber\\
&\leq\left|\prod_{i=1}^{d}f_{NN,{\nu_i}}^{(Lgd)}({x}_i)-L_{\nu_1}(x_1)\cdot\prod_{i=2}^{d}f_{NN,{\nu_i}}^{(Lgd)}({x}_i)\right|\nonumber\\
&\quad+\left|L_{\nu_1}(x_1)\cdot\prod_{i=2}^{d}f_{NN,{\nu_i}}^{(Lgd)}({x}_i)-\prod_{i=1}^2L_{\nu_i}(x_i)\cdot\prod_{i=3}^{d}f_{NN,{\nu_i}}^{(Lgd)}({x}_i)\right|\nonumber\\
&\quad+\cdots+\left|\prod_{i=1}^{i'}L_{\nu_i}(x_i)\cdot\prod_{i=i'+1}^{d}f_{NN,{\nu_i}}^{(Lgd)}({x}_i)-\prod_{i=1}^{i'+1}L_{\nu_i}(x_i)\cdot\prod_{i=i'+2}^{d}f_{NN,{\nu_i}}^{(Lgd)}({x}_i)\right|\nonumber\\
&\quad+\cdots+\left|\prod_{i=1}^{d-1}L_{\nu_i}(x_i)\cdot f_{NN,{\nu_d}}^{(Lgd)}({x}_d)-\prod_{i=1}^{d}L_{\nu_i}(x_i)\right|\nonumber\\
&\leq d16^{(d-1)\bar{\nu}}\cdot\frac{15}{2}16^{\bar{\nu}}N^{-5L/2}\nonumber\\
&=\frac{15}{2}d16^{d\bar{\nu}}N^{-5L/2},\label{multivariate Legendre4}
\end{align}
where we apply Lemma \ref{Legendre},\eqref{multivariate Legendre0.75},\eqref{multivariate Legendre2.25} and \eqref{multivariate Legendre2.5} in the second step. Plugging \eqref{multivariate Legendre3} and \eqref{multivariate Legendre4} into \eqref{multivariate Legendre2}, we derive that for any $\boldsymbol{\nu}\in\Lambda_{\epsilon}$,
\begin{align}\label{d-Lgd 2}
\left|f_{NN,\boldsymbol{\nu}}^{(d-Lgd)}(\boldsymbol{x})-L_{\boldsymbol{\nu}}(\boldsymbol{x})\right|&\leq96N^{-d\lceil\log_2d\rceil\left\lceil\frac{4\bar{\nu}}{\log_2q}\right\rceil }+\frac{15}{2}d16^{d\bar{\nu}}N^{-5L/2}.
\end{align}

\textbf{Step 2. }In the second step, we embed the $\left\{f_{NN,\boldsymbol{\nu}}^{(d-Lgd)}(\boldsymbol{x})\right\}_{\boldsymbol{\nu}\in\Lambda_\epsilon}$ constructed in the previous step into a larger network to efficiently compute their weighted sum. Here, a tunable parameter $p$ is introduced to denote the number of $f_{NN,\boldsymbol{\nu}}^{(d-Lgd)}(\boldsymbol{x})$ terms summed within each layer of the larger network.

For convenience, we rearrange
$
\left\{f_{NN,\boldsymbol{\nu}}^{(d-Lgd)}(\boldsymbol{x})\right\}_{\boldsymbol{\nu}\in\Lambda_\epsilon}$ as
$\left\{f_{NN,j}^{(d-Lgd)}(\boldsymbol{x})\right\}_{j\in\left[\left|\Lambda_\epsilon\right|\right]}$. 
We also rearrange the corresponding Legendre expansion coefficients
$
\left\{l_{\boldsymbol{\nu}}(f)\right\}_{\boldsymbol{\nu}\in\Lambda_\epsilon}$ as $\left\{l_{j}\right\}_{j\in\left[\left|\Lambda_\epsilon\right|\right]}$ and the corresponding linear mappings $\left\{\mathcal{H}_{\boldsymbol{\nu}}\right\}_{\boldsymbol{\nu}\in\Lambda_\epsilon}$ as $\left\{\mathcal{H}_j\right\}_{j\in\left[\left|\Lambda_\epsilon\right|\right]}$. In the following, we inductively show that for $i\in\left[\left\lceil\frac{\left|\Lambda_\epsilon\right|}{p}\right\rceil\right]$, there exists a neural network $\boldsymbol{\phi}^{(i)}:\mathbb{R}^d\to\mathbb{R}^{d\bar{\nu}+p+1}$ with width 
\begin{align*}
\max\left\{2^{\left\lceil\log_2\bar{\nu}\right\rceil}\cdot10d\sqrt{N}+2d,5\cdot2^{\lceil\log_2d\rceil}pq\sqrt{N}+2(d\bar{\nu}+1)\right\}
\end{align*}
and depth 
\begin{align*}
5\left\lceil\log_2\bar{\nu}\right\rceil^2L+i\cdot2d\lceil\log_2d\rceil^2\left\lceil\frac{4\bar{\nu}}{\log_2q}\right\rceil. 
\end{align*}
such that for any $\boldsymbol{x}\in\left[-1,1\right]^d$, there holds
\begin{align}\label{final10}
\boldsymbol{\phi}^{(i)}(\boldsymbol{x})=\begin{pmatrix}
\boldsymbol{f}_{NN}^{(prl-Lgd)}(\boldsymbol{x})\\
f_{NN,(i-1)p+1}^{(d-Lgd)}(\boldsymbol{x})\\
f_{NN,(i-1)p+2}^{(d-Lgd)}(\boldsymbol{x})\\
\vdots\\
f_{NN,(i-1)p+p}^{(d-Lgd)}(\boldsymbol{x})\\
\sum_{j=1}^{(i-1)p} l_{j} f_{NN,j}^{(d-Lgd)}(\boldsymbol{x})
\end{pmatrix}\in\mathbb{R}^{d\bar{\nu}+p+1}.
\end{align}
For the case of $i=1$, define a preliminary mapping $\boldsymbol{{\phi}}^{(plm)}:\mathbb{R}^{d\bar{\nu}}\to\mathbb{R}^{d\bar{\nu}+p+1}$ as
\begin{align*}
\boldsymbol{{\phi}}^{(plm)}(\boldsymbol{y}):=
\begin{pmatrix}
\boldsymbol{y}\\
f_{NN,d}^{(mtp)}\left(\mathcal{H}_{1}\boldsymbol{y}\right)\\
f_{NN,d}^{(mtp)}\left(\mathcal{H}_{2}\boldsymbol{y}\right)\\
\vdots\\
f_{NN,d}^{(mtp)}\left(\mathcal{H}_{p}\boldsymbol{y}\right)\\
0
\end{pmatrix},\quad\forall\boldsymbol{y}\in\mathbb{R}^{d\bar{\nu}}.
\end{align*}
Then $\boldsymbol{{\phi}}^{(plm)}$ is of width $5\cdot2^{\lceil\log_2d\rceil}pq\sqrt{N}+2(d\bar{\nu}+1)$ and depth $2d\lceil\log_2d\rceil^2\left\lceil\frac{4\bar{\nu}}{\log_2q}\right\rceil$. Based on $\boldsymbol{{\phi}}^{(plm)}$, define
\begin{align*}
\boldsymbol{\phi}^{(1)}(\boldsymbol{x}):=\boldsymbol{{\phi}}^{(plm)}\left(\boldsymbol{f}_{NN}^{(prl-Lgd)}(\boldsymbol{x})\right)\in\mathbb{R}^{d\bar{\nu}+p+1}.
\end{align*}
Hence $\boldsymbol{\phi}^{(1)}$ satisfies \eqref{final10} due to \eqref{d-Lgd}. From its definition, we see that the width of 
$\boldsymbol{\phi}^{(1)}$ is 
\begin{align*}
&\max\left\{2^{\left\lceil\log_2\bar{\nu}\right\rceil}\cdot10d\sqrt{N}+2d,5\cdot2^{\lceil\log_2d\rceil}pq\sqrt{N}+2(d\bar{\nu}+1)\right\}
\end{align*}
and the depth is 
\begin{align*}
5\left\lceil\log_2\bar{\nu}\right\rceil^2L+2d\lceil\log_2d\rceil^2\left\lceil\frac{4\bar{\nu}}{\log_2q}\right\rceil.   
\end{align*}
Now, we assume \eqref{final10} is valid for the case of $i$ and define
\begin{align*}
\boldsymbol{\phi}^{(i+1)}(\boldsymbol{x}):&=\begin{pmatrix}
\boldsymbol{\phi}_{1:d\bar{\nu}}^{(i)}\\
f_{NN,d}^{(mtp)}\left(\mathcal{H}_{ip+1}\boldsymbol{\phi}_{1:d\bar{\nu}}^{(i)}\right)\\
f_{NN,d}^{(mtp)}\left(\mathcal{H}_{ip+2}\boldsymbol{\phi}_{1:d\bar{\nu}}^{(i)}\right)\\
\vdots\\
f_{NN,d}^{(mtp)}\left(\mathcal{H}_{ip+p}\boldsymbol{\phi}_{1:d\bar{\nu}}^{(i)}\right)\\
{\phi}_{d\bar{\nu}+p+1}^{(i)}+\sum_{j=1}^{p}l_{(i-1)p+j}{\phi}_{d\bar{\nu}+j}^{(i)}
\end{pmatrix}\\
&=\begin{pmatrix}
\boldsymbol{f}_{NN}^{(prl-Lgd)}(\boldsymbol{x})\\
f_{NN,d}^{(mtp)}\left(\mathcal{H}_{ip+1}f_{NN}^{(prl-Lgd)}(\boldsymbol{x})\right)\\
f_{NN,d}^{(mtp)}\left(\mathcal{H}_{ip+2}f_{NN}^{(prl-Lgd)}(\boldsymbol{x})\right)\\
\vdots\\
f_{NN,d}^{(mtp)}\left(\mathcal{H}_{ip+p}f_{NN}^{(prl-Lgd)}(\boldsymbol{x})\right)\\
\sum_{j=1}^{(i-1)p} l_{j} f_{NN,j}^{(Lgd)}(\boldsymbol{x})+\sum_{j=1}^{p}l_{(i-1)p+j}f_{NN,(i-1)p+j}^{(Lgd)}(\boldsymbol{x})
\end{pmatrix}\\
&=\begin{pmatrix}
\boldsymbol{f}_{NN}^{(prl-Lgd)}(\boldsymbol{x})\\
f_{NN,ip+1}^{(d-Lgd)}(\boldsymbol{x})\\
f_{NN,ip+2}^{(d-Lgd)}(\boldsymbol{x})\\
\vdots\\
f_{NN,ip+p}^{(d-Lgd)}(\boldsymbol{x})\\
\sum_{j=1}^{ip} l_{j} f_{NN,j}^{(d-Lgd)}(\boldsymbol{x})
\end{pmatrix}\in\mathbb{R}^{d\bar{\nu}+p+1},
\end{align*}
where the second equality is due to the induction assumption and the third equality is due to \eqref{d-Lgd}. It also follows from induction assumption that the width of $\boldsymbol{\phi}^{(i+1)}$ is  
\begin{align*}
&\max\left\{\max\left\{2^{\left\lceil\log_2\bar{\nu}\right\rceil}\cdot10d\sqrt{N}+2d,5\cdot2^{\lceil\log_2d\rceil}pq\sqrt{N}+2(d\bar{\nu}+1)\right\},5\cdot2^{\lceil\log_2d\rceil}pq\sqrt{N}+2(d\bar{\nu}+1)\right\}\\
&=\max\left\{2^{\left\lceil\log_2\bar{\nu}\right\rceil}\cdot10d\sqrt{N}+2d,5\cdot2^{\lceil\log_2d\rceil}pq\sqrt{N}+2(d\bar{\nu}+1)\right\}
\end{align*}
and the depth is 
\begin{align*}
&5\left\lceil\log_2\bar{\nu}\right\rceil^2L+i\cdot2d\lceil\log_2d\rceil^2\left\lceil\frac{4\bar{\nu}}{\log_2q}\right\rceil+2d\lceil\log_2d\rceil^2\left\lceil\frac{4\bar{\nu}}{\log_2q}\right\rceil\\
&=5\left\lceil\log_2\bar{\nu}\right\rceil^2L+(i+1)\cdot2d\lceil\log_2d\rceil^2\left\lceil\frac{4\bar{\nu}}{\log_2q}\right\rceil.  
\end{align*}
Hence, the induction is completed and we derive that
\begin{align*}
\boldsymbol{\phi}^{\left(\left\lceil\frac{\left|\Lambda_\epsilon\right|}{p}\right\rceil\right)}(\boldsymbol{x})=\begin{pmatrix}
\boldsymbol{f}_{NN}^{(prl-Lgd)}(\boldsymbol{x})\\
f_{NN,\left(\left\lceil\frac{\left|\Lambda_\epsilon\right|}{p}\right\rceil-1\right)p+1}^{(d-Lgd)}(\boldsymbol{x})\\
f_{NN,\left(\left\lceil\frac{\left|\Lambda_\epsilon\right|}{p}\right\rceil-1\right)p+2}^{(d-Lgd)}(\boldsymbol{x})\\
\vdots\\
f_{NN,\left(\left\lceil\frac{\left|\Lambda_\epsilon\right|}{p}\right\rceil-1\right)p+p}^{(d-Lgd)}(\boldsymbol{x})\\
\sum_{j=1}^{\left(\left\lceil\frac{\left|\Lambda_\epsilon\right|}{p}\right\rceil-1\right)p} l_{j} f_{NN,j}^{(d-Lgd)}(\boldsymbol{x})
\end{pmatrix}\in\mathbb{R}^{d\bar{\nu}+p+1},
\end{align*}
Furthermore, its width is 
\begin{align*}
\max\left\{2^{\left\lceil\log_2\bar{\nu}\right\rceil}\cdot10d\sqrt{N}+2d,5\cdot2^{\lceil\log_2d\rceil}pq\sqrt{N}+2(d\bar{\nu}+1)\right\}
\end{align*}
and its depth is
\begin{align*}
5\left\lceil\log_2\bar{\nu}\right\rceil^2L+\left\lceil\frac{\left|\Lambda_\epsilon\right|}{p}\right\rceil\cdot2d\lceil\log_2d\rceil^2\left\lceil\frac{4\bar{\nu}}{\log_2q}\right\rceil. 
\end{align*}
Let
\begin{align*}
\boldsymbol{l}:=
\begin{pmatrix}
\boldsymbol{0}_{1\times d\bar{\nu}}&l_{\left(\left\lceil\frac{\left|\Lambda_\epsilon\right|}{p}\right\rceil-1\right)p+1}&l_{\left(\left\lceil\frac{\left|\Lambda_\epsilon\right|}{p}\right\rceil-1\right)p+2}&\cdots&l_{\left|\Lambda_\epsilon\setminus\{\boldsymbol{0}\}\right|}&\boldsymbol{0}_{1\times\left(\left\lceil\frac{\left|\Lambda_\epsilon\right|}{p}\right\rceil p-\left|\Lambda_\epsilon\right|\right)}&1 
\end{pmatrix}\in\mathbb{R}^{1\times(d\bar{\nu}+p+1)}
\end{align*}
and define
\begin{align*}
f_{NN}(\boldsymbol{x}):&=
\boldsymbol{l}\boldsymbol{\phi}^{\left(\left\lceil\frac{\left|\Lambda_\epsilon\right|}{p}\right\rceil\right)}(\boldsymbol{x})
=\sum_{j=1}^{\left|\Lambda_\epsilon\right|} l_{j} f_{NN,j}^{(d-Lgd)}(\boldsymbol{x})=\sum_{{\boldsymbol{\nu}} \in \Lambda_\epsilon}l_{\boldsymbol{\nu}}(f)f_{NN,\boldsymbol{\nu}}^{(d-Lgd)}(\boldsymbol{x}).
\end{align*}
By the traingle inequality,
\begin{align*}
\left| f(\boldsymbol{x}) - f_{NN}(\boldsymbol{x}) \right|
&\leq\left| f(\boldsymbol{x}) - \sum_{{\boldsymbol{\nu}} \in \Lambda_\epsilon} l_{\boldsymbol{\nu}}(f) L_{\boldsymbol{\nu}}(\boldsymbol{x}) \right|+\left| \sum_{{\boldsymbol{\nu}} \in \Lambda_\epsilon} l_{\boldsymbol{\nu}}(f) L_{\boldsymbol{\nu}}(\boldsymbol{x}) -f_{NN}(\boldsymbol{x})\right|.
\end{align*}
The first term on the right-hand side is bounded by Proposition \ref{Legendre approximation}:
\begin{align*}
\left| f(\boldsymbol{x}) - \sum_{{\boldsymbol{\nu}} \in \Lambda_\epsilon} l_{\boldsymbol{\nu}}(f) L_{\boldsymbol{\nu}}(\boldsymbol{x}) \right|\leq C(d,\boldsymbol{\rho},M) e^{-C(d,\boldsymbol{\rho}) |\Lambda_\epsilon|^{1/d}}.
\end{align*}
For the second term,
\begin{align*}
\left| \sum_{{\boldsymbol{\nu}} \in \Lambda_\epsilon} l_{\boldsymbol{\nu}}(f) L_{\boldsymbol{\nu}}(\boldsymbol{x}) -f_{NN}(\boldsymbol{x})\right|
&\leq \sum_{{\boldsymbol{\nu}} \in \Lambda_\epsilon} \left|l_{\boldsymbol{\nu}}(f)\right|\left| L_{\boldsymbol{\nu}}(\boldsymbol{x}) -f_{NN,\boldsymbol{\nu}}^{(d-Lgd)}(\boldsymbol{x})\right|\\
&\leq C(d)\left(N^{-C(d)\frac{|\Lambda_{\epsilon}|^{1/d}}{\log_2q} }+16^{C(d)|\Lambda_{\epsilon}|^{1/d}}N^{-5L/2}\right)\sum_{{\boldsymbol{\nu}} \in \Lambda_\epsilon\setminus\{\boldsymbol{0}\}} \left|l_{\boldsymbol{\nu}}(f)\right|\\
&\leq C(d,\boldsymbol{\rho},M)\left(N^{-C(d)\frac{|\Lambda_{\epsilon}|^{1/d}}{\log_2q} }+16^{C(d)|\Lambda_{\epsilon}|^{1/d}}N^{-5L/2}\right),
\end{align*}
where we empoly \eqref{d-Lgd 2} in the second step and Lemma \ref{coefficient2} in the third step. Therefore, 
\begin{align*}
\left| f(\boldsymbol{x}) - f_{NN}(\boldsymbol{x}) \right|
&\leq C(d,\boldsymbol{\rho},M)e^{-C(d,\boldsymbol{\rho}) |\Lambda_\epsilon|^{1/d}}+C(d,\boldsymbol{\rho},M)\left(N^{-C(d)\frac{|\Lambda_{\epsilon}|^{1/d}}{\log_2q} }+16^{C(d)|\Lambda_{\epsilon}|^{1/d}}N^{-5L/2}\right).
\end{align*}
We complete the proof by setting
\begin{align*}
|\Lambda_\epsilon|\asymp L^d(\log N)^d,\quad p\asymp L^{\lambda}(\log N)^d,\quad q\asymp \sqrt{N}.
\end{align*}

\end{proof}

We are now able to prove Proposition \ref{II}, which is restated below for convenience.

\begin{reproposition2}[restated]
Suppose $L,N$ are sufficiently large and there exist $\kappa\in\left[\frac{1}{d+1},d\right],\beta>0$ such that
\begin{align*}
L^{\kappa+\alpha}\leq N\leq e^{L^{\beta}},
\end{align*}
where $\alpha>0$ can be arbitrarily small. There exists a neural network $f_{NN}:\mathbb{R}^d\to\mathbb{R}$ with width $C(d,\boldsymbol{\rho})N$ and depth $C(d,\boldsymbol{\rho},\beta)L$ (when $\kappa=d$, a larger depth $C(d,\boldsymbol{\rho},\beta)L(\log L)^2$ is needed) such that for any $\boldsymbol{x}\in\left[-1,1\right]^d$, there holds
\begin{align*}
\left| f(\boldsymbol{x}) - f_{NN}(\boldsymbol{x}) \right|
&\leq\left\{\begin{matrix}
C(d,\boldsymbol{\rho},M)N^{-C(d,\boldsymbol{\rho},\kappa,\alpha)L^{\kappa}},  & \kappa\in\left[\frac{1}{d+1},\frac{1}{d}\right);\\
C(d,\boldsymbol{\rho},M)N^{-C(d,\boldsymbol{\rho},\kappa,\alpha)L^{\frac{\kappa+1}{d+1}}},  & \kappa\in\left[\frac{1}{d},d\right].
\end{matrix}\right.
\end{align*}
\end{reproposition2}
\begin{proof}

\textbf{Case $\kappa\in\left[\frac{1}{d+1},\frac{1}{d}\right)$.} Choosing $\lambda=d+1-\frac{1}{\kappa}<1$ and replacing $L$ with ${L}^{\kappa}$ and $N$ with $N^{\frac{\alpha}{2(\kappa+\alpha)}}$ in Lemma \ref{final}, we conclude that there exists a neural network $f_{NN}:\mathbb{R}^d\to\mathbb{R}$ such that for any $\boldsymbol{x}\in\left[-1,1\right]^d$, there holds
\begin{align*}
\left| f(\boldsymbol{x}) - f_{NN}(\boldsymbol{x}) \right|
&\leq C(d,\boldsymbol{\rho},M)N^{-C(d,\boldsymbol{\rho},\kappa,\alpha)L^{\kappa}}.
\end{align*}
Since $L^{\kappa+\alpha}\leq N\leq e^{L^{\beta}}$, the width of $f_{NN}$ is bounded as
\begin{align*}
C(d,\boldsymbol{\rho})L^{\kappa}N^{\frac{\alpha}{2(\kappa+\alpha)}}(\log N)^{d}
&\leq C(d,\boldsymbol{\rho})N^{\frac{\kappa}{\kappa+\alpha}}N^{\frac{\alpha}{2(\kappa+\alpha)}}N^{\frac{\alpha}{2(\kappa+\alpha)}}=C(d,\boldsymbol{\rho})N
\end{align*}
and the depth is bounded as
\begin{align*}
C(d,\boldsymbol{\rho})\left(L^{\kappa}\right)^{\frac{1}{\kappa}}+C(d,\boldsymbol{\rho},\beta)L^{\kappa}\left(\log L+\log\log N\right)^2\leq C(d,\boldsymbol{\rho},\beta)L.
\end{align*}

\noindent\textbf{Case $\kappa\in\left[\frac{1}{d},d\right]$.} Choosing $\lambda=\frac{\kappa(d+1)}{\kappa+1}\geq1$ and replacing $L$ with ${L}^{\frac{\kappa+1}{d+1}}$ and $N$ with $N^{\frac{\alpha}{2(\kappa+\alpha)}}$ in Lemma \ref{final}, we conclude that there exists a neural network $f_{NN}:\mathbb{R}^d\to\mathbb{R}$ such that for any $\boldsymbol{x}\in\left[-1,1\right]^d$, there holds
\begin{align*}
\left| f(\boldsymbol{x}) - f_{NN}(\boldsymbol{x}) \right|
&\leq C(d,\boldsymbol{\rho},M)N^{-C(d,\boldsymbol{\rho},\kappa,\alpha)L^{\frac{\kappa+1}{d+1}}}.
\end{align*}
Since $L^{\kappa+\alpha}\leq N\leq e^{L^{\beta}}$, the width of $f_{NN}$ is bounded as
\begin{align*}
C(d,\boldsymbol{\rho})\left(L^{\frac{\kappa+1}{d+1}}\right)^{\frac{\kappa(d+1)}{\kappa+1}}N^{\frac{\alpha}{2(\kappa+\alpha)}}(\log N)^{d}&\leq C(d,\boldsymbol{\rho})L^{\kappa}N^{\frac{\alpha}{2(\kappa+\alpha)}}N^{\frac{\alpha}{2(\kappa+\alpha)}}\\
&\leq C(d,\boldsymbol{\rho})N^{\frac{\kappa}{\kappa+\alpha}}N^{\frac{\alpha}{2(\kappa+\alpha)}}N^{\frac{\alpha}{2(\kappa+\alpha)}}=C(d,\boldsymbol{\rho})N
\end{align*}
and the depth is bounded as
\begin{align*}
&C(d,\boldsymbol{\rho})\left(L^{\frac{\kappa+1}{d+1}}\right)^{\frac{d+1}{\kappa+1}}+C(d,\boldsymbol{\rho})L^{\frac{\kappa+1}{d+1}}\left(\log L+\log\log N\right)^2\\
&\leq C(d,\boldsymbol{\rho})L+C(d,\boldsymbol{\rho},\beta)L^{\frac{\kappa+1}{d+1}}\left(\log L\right)^2\leq\left\{\begin{matrix}
C(d,\boldsymbol{\rho})L,  & \kappa\in\left[\frac{1}{d},d\right);\\
C(d,\boldsymbol{\rho},\beta)L(\log L)^2,  & \kappa=d.
\end{matrix}\right.
\end{align*}
\end{proof}

\subsection{Proof of Theorem \ref{app lb}}\label{Proof of lb}

The proof of Theorem \ref{app lb} relies on the following properties of ReLU neural networks. Similar results can also be found in \cite{telgarsky2015representation,serra2018bounding}.
\begin{lemma}[\cite{raghu2017expressive}, Theorem 1 \& Theorem 2]\label{ReLU affine}
For a ReLU neural network function $f_{NN}:\mathbb{R}^d\to\mathbb{R}$ with width $N$ and depth $L$, $f_{NN}$ partitions $\mathbb{R}^d$ into $CN^{dL}$ convex polytopes, on each of which $f_{NN}$ is an affine function.
\end{lemma}

The proof also requires the isodiametric inequality, a classical result in geometric measure theory which states that the Euclidean ball maximizes the volume among all sets of a given diameter.

\begin{proposition}[Isodiametric inequality]\label{Isodiametric inequality}
For any set \( A \subset \mathbb{R}^d \),
\begin{align*}
|A| \leq \alpha_d \left( \frac{\mathrm{diam } (A)}{2} \right)^d.
\end{align*}
where $\alpha_d$ is the Lebesgue measure of $d-$dimensional unit ball and
\begin{align*}
\mathrm{diam } (A):=\sup_{\boldsymbol{x}_1,\boldsymbol{x}_2\in A}\|\boldsymbol{x}_1-\boldsymbol{x}_1\|_2.
\end{align*}
\end{proposition}
\begin{proof}
See, for example, \cite[Theorem 2.4]{evans2015measure}.
\end{proof}

\begin{retheorem}[restated]
For any $N$ and $L$, let $\mathcal{F}_{NN}\left({N},{L}\right)$ be the ReLU neural network function class with width ${N}$ and depth ${L}$. There holds
\begin{align*}
\inf_{f_{NN}\in\mathcal{F}_{NN}(N,L)}\sup_{f\in\mathcal{A}(\boldsymbol{\rho},M)}\|f-f_{NN}\|_{L^{\infty}([-1,1]^d)}\geq C(d,\boldsymbol{\rho},M)N^{-2L}.
\end{align*}
\end{retheorem}
\begin{proof}
The proof basically follows the convexity argument of \cite[Theorem 3]{shapira2023expressivity}. Since $\mathcal{A}(\boldsymbol{\rho},M)$ always contains strongly convex functions (for example, consider a scaled version of $\prod_{i=1}^d(x_i+2)^2$), we can choose $f\in\mathcal{A}(\boldsymbol{\rho},M)$ to be a strongly convex function with parameter $m=m(\boldsymbol{\rho},M)$:
\begin{align}\label{app lb1}
f(\boldsymbol{x}_2)\geq f(\boldsymbol{x}_1)+\nabla f(\boldsymbol{x}_1)^{\top}(\boldsymbol{x}_2-\boldsymbol{x}_1)+\frac{m}{2}\|\boldsymbol{x}_2-\boldsymbol{x}_1\|_2^2,\quad\forall \boldsymbol{x}_1,\boldsymbol{x}_2\in[-1,1]^d.
\end{align}
Let $f_{NN}$ be a ReLU neural network that $\epsilon-$approximates $f$ under the distance $\|\cdot\|_{L^{\infty}([-1,1]^d)}$ and define
\begin{align*}
g(\boldsymbol{x}):=f(\boldsymbol{x})-f_{NN}(\boldsymbol{x}).
\end{align*}
It follows that for any $\boldsymbol{x}\in[-1,1]^d$,
\begin{align}\label{app lb2}
|g(\boldsymbol{x})|\leq\epsilon.
\end{align}
From Lemma \ref{ReLU affine}, we know $[-1,1]^d$ is partitioned by $f_{NN}$ into $CN^{dL}$ convex polytopes, on each of which $f_{NN}$ is an affine function. Consequently, there exists at least one polytope $\Omega\subset[-1,1]^d$ among them such that $|\Omega|\geq CN^{-dL}$, which further implies via Proposition \ref{Isodiametric inequality} that $\mathrm{diam}(\Omega)\geq C(d)|\Omega|^{1/d}=C(d)N^{-L}$. Hence, there exist $\boldsymbol{a},\boldsymbol{b}\in\Omega$ such that
\begin{align}\label{app lb3}
\|\boldsymbol{b}-\boldsymbol{a}\|_2\geq C(d)N^{-L}.
\end{align}
It follows that 
\begin{align}
&g(\boldsymbol{b})+g(\boldsymbol{a})-2g\left(\frac{\boldsymbol{a}+\boldsymbol{b}}{2}\right)\nonumber\\
&=f(\boldsymbol{b})+f(\boldsymbol{a})-2f\left(\frac{\boldsymbol{a}+\boldsymbol{b}}{2}\right)
-f_{NN}(\boldsymbol{b})-f_{NN}(\boldsymbol{a})+2f_{NN}\left(\frac{\boldsymbol{a}+\boldsymbol{b}}{2}\right)\nonumber\\
&=f(\boldsymbol{b})-f\left(\frac{\boldsymbol{a}+\boldsymbol{b}}{2}\right)+f(\boldsymbol{a})-f\left(\frac{\boldsymbol{a}+\boldsymbol{b}}{2}\right)\nonumber\\
&\geq\nabla f\left(\frac{\boldsymbol{a}+\boldsymbol{b}}{2}\right)^{\top}\frac{\boldsymbol{b}-\boldsymbol{a}}{2}+\frac{m}{2}\left\|\frac{\boldsymbol{b}-\boldsymbol{a}}{2}\right\|_2^2+\nabla f\left(\frac{\boldsymbol{a}+\boldsymbol{b}}{2}\right)^{\top}\frac{\boldsymbol{a}-\boldsymbol{b}}{2}+\frac{m}{2}\left\|\frac{\boldsymbol{a}-\boldsymbol{b}}{2}\right\|_2^2\nonumber\\
&=\frac{m}{4}\|\boldsymbol{b}-\boldsymbol{a}\|_2^2\geq C(d,\boldsymbol{\rho},M)N^{-2L},\label{app lb4}
\end{align}
where the second step is due to Lemma \ref{ReLU affine}, the third step is due to \eqref{app lb1} and the fifth step is due to \eqref{app lb3}. Note that \eqref{app lb2} implies
\begin{align}\label{app lb5}
g(\boldsymbol{b})+g(\boldsymbol{a})-2g\left(\frac{\boldsymbol{a}+\boldsymbol{b}}{2}\right)\leq4\epsilon.
\end{align}
The proof is completed by combining \eqref{app lb4}\eqref{app lb5}. 

\end{proof}

\section{Nonparametric Regression}\label{Nonparametric Regression}

Based on the derived approximation results, this section is devoted to the study of neural network estimators in nonparametric regression, which is a classical statistical problem that has witnessed a significant interest in recent machine learning literature \cite{suzukiadaptivity,schmidt2020nonparametric,nakada2020adaptive,kohler2021rate,chen2022nonparametric,jiao2023deepb,yang2025optimal}. Specifically, we consider the following nonparametric regression setting:
\begin{align*}
{Y}_i = f_0(\boldsymbol{X}_i) + \xi_i,\quad i\in[n],    
\end{align*}
where $f_0(\boldsymbol{x})=\mathbb{E}[Y|\boldsymbol{X}=\boldsymbol{x}]:[-1,1]^{d}\to\mathbb{R}$ is the unknown target function in $\mathcal{A}(\boldsymbol{\rho},M)$, $\{(\boldsymbol{X}_i,Y_i)\}_{i=1}^n\subset[-1,1]^{d}\times\mathbb{R}$ are observation pairs, $\{\xi_i\}_{i=1}^n$ are i.i.d. Gaussian noises with $\mathbb{E}\xi_i=0,\mathrm{Var}(\xi_i)=\eta^2$. Let $\mu$ be the marginal measure of $\boldsymbol{X}$. Suppose $\mu$ is absolutely continuous with respect to the Lebesgue measure with density $p_X(\boldsymbol{x})$ satisfying $\underline{p}\leq p_X(\boldsymbol{x})\leq\overline{p}$ for some $\underline{p},\overline{p}\in\mathbb{R}_{>0}$. 
Our goal is to estimate $f_0$ based on the given observation pairs $\{(\boldsymbol{X}_i,Y_i)\}_{i=1}^n$. Let $\mathcal{F}_{NN}\left(\bar{N},\bar{L}\right)$ be the ReLU neural network function class with width $\bar{N}$ and depth $\bar{L}$. Our estimator $\mathcal{T}_{B_n}\circ\widehat{f}_n$ is composed of a truncation function $\mathcal{T}_{B_n}(x):=\max\{\min\{x,B_n\},-B_n\}$ with $B_n>0$ being the truncation parameter and the least square solution 
\begin{align*}
\widehat{f}_n:=\arg\min_{f\in\mathcal{F}_{NN}\left(\bar{N},\bar{L}\right)}\frac{1}{n}\sum_{i=1}^n[f(\boldsymbol{X}_i)-Y_i]^2.
\end{align*}
The difference of the estimator and the target is measured via the so-called excess risk: 
\begin{align*}
\left\|\mathcal{T}_{B_n}\circ\widehat{f}_n-f_{0}\right\|_{L^{2}(\mu)}^2:=\int_{[-1,1]^{d}}\left|\mathcal{T}_{B_n}\circ\widehat{f}_n-f_{0}\right|^2d\mu. 
\end{align*}

A standard approach to bounding the excess risk involves trading off the bias and variance. The following proposition is a typical result that provides a decomposition of the excess risk, where the first and second terms on the right-hand side represent the variance and bias terms, respectively. Here $\mathcal{N}(\epsilon, \mathcal{F},\|\cdot\|_1,n)$ denotes the $\epsilon$-covering number of function class $\mathcal{F}$ under the empirical $L^1$ metric.
\begin{proposition}[\cite{kohler2021rate}, Lemma 18]\label{error decomposition}
Let $B_n\asymp\log n$. There holds
\begin{align*}
&\mathbb{E}_{\{(\boldsymbol{X}_i,Y_i)\}_{i=1}^n}\left\|\mathcal{T}_{B_n}\circ\widehat{f}_n - f_0\right\|_{L^2(\mu)}^2 \\
&\leq\frac{C(M,\eta)(\log n)^2 \log  \mathcal{N} \left( n^{-1}B_n^{-1}, \mathcal{T}_{B_n}\circ \mathcal{F}_{NN}\left(\bar{N},\bar{L}\right), \|\cdot\|_{1},n\right)}{n}
+ 2 \inf_{f \in \mathcal{F}_{NN}\left(\bar{N},\bar{L}\right)}  \left\|{f} - f_0\right\|_{L^2(\mu)}^2. 
\end{align*}
\end{proposition}

The following lemma is a classical result that upper-bounds the covering number of $\mathcal{F}$ in terms of its pseudo-dimension, denoted by $\mathrm{Pdim}(\mathcal{F})$. The readers are also referred to 
\cite{haussler1992decision,haussler1995sphere}.
\begin{lemma}[\cite{anthony2009neural}, Theorem 18.4]\label{covering}
Let $B_{\mathcal{F}}\in\mathbb{R}_{>0}$. Let \( \mathcal{F} \) be a nonempty class of real-valued functions mapping from a domain \( \Omega \) into the interval \(\left[-B_{\mathcal{F}}, B_{\mathcal{F}}\right]\). Then for any \( \epsilon > 0 \),
\[
\mathcal{N}(\epsilon, \mathcal{F},\|\cdot\|_1,n) \leq  e\left(\mathrm{Pdim}(\mathcal{F}) + 1\right) \left( \frac{4eB_{\mathcal{F}}}{\epsilon} \right)^{\mathrm{Pdim}(\mathcal{F})}.
\]
\end{lemma}

The pseudo-dimension of neural network classes has been extensively studied, with \cite{anthony2009neural} serving as a classic reference in this line of work. In this section, we utilize the following result.
\begin{proposition}[\cite{bartlett2019nearly}, Theorem 7]\label{Pdim}
For sufficiently large $\bar{N}$ and $\bar{L}$,
\begin{align*}
\mathrm{Pdim}\left(\mathcal{F}_{NN}\left(\bar{N},\bar{L}\right)\right)\leq C(d)\bar{L}^2\bar{N}^2\log\left(\bar{L}\bar{N}\right).
\end{align*}
\end{proposition}

We also need the following lemma.
\begin{lemma}[\cite{anthony2009neural}, Theorem 11.3]\label{nondecreasing}
Let \( \mathcal{F} \) be a class of real-valued functions and \( \rho : \mathbb{R} \to \mathbb{R} \) be a non-decreasing function. Let \( \rho(\mathcal{F}) \) denote the class \( \{ \rho \circ f : f \in \mathcal{F} \} \). Then \( \mathrm{Pdim}(\rho(\mathcal{F})) \leq \mathrm{Pdim}(\mathcal{F}) \).
\end{lemma}

Now we proceed to present our upper bounds on the excess risk. In the following theorem, we provide three different choices for the network width and depth, arranged in the order of increasing width and decreasing depth.

\begin{theorem}\label{regression ub}
Suppose $n$ is sufficiently large. Let $B_n\asymp\log n$.
\begin{enumerate}[label=(\Roman*)]
\item
If
\begin{align*}
\bar{N}\leq C(d,\boldsymbol{\rho}),\quad \bar{L}\leq C(d,\boldsymbol{\rho})(\log n)^{d+3},
\end{align*}
then
\begin{align*}
\mathbb{E}_{\{(\boldsymbol{X}_i,Y_i)\}_{i=1}^n}\left\|\mathcal{T}_{B_n}\circ\widehat{f}_n - f_0\right\|_{L^2(\mu)}^2\leq C(d,\boldsymbol{\rho},M,\eta)\frac{(\log n)^{2d+9}\log\log n}{n}.
\end{align*}

\item 
If
\begin{align*}
\bar{N}\leq C(d,\boldsymbol{\rho})(\log n)^{p},\quad\bar{L}\leq C(d,\boldsymbol{\rho})\left(\frac{\log n}{\log\log n}\right)^{d+3-p}
\end{align*}
for some $p\in[1,d+1]$, then
\begin{align*}
\mathbb{E}_{\{(\boldsymbol{X}_i,Y_i)\}_{i=1}^n}\left\|\mathcal{T}_{B_n}\circ\widehat{f}_n - f_0\right\|_{L^2(\mu)}^2\leq C(d,\boldsymbol{\rho},M,\eta)\frac{(\log n)^{2d+9}(\log\log n)^{2p-2d-5}}{n}.
\end{align*}

\item 
If
\begin{align*}
\bar{N}\leq C(d,\boldsymbol{\rho})e^{C(\log n)^{p}},\quad\bar{L}\leq C(d,\boldsymbol{\rho})\left(\log n\right)^{2-p}(\log\log n)^2
\end{align*}
for some $p\in(0,1)$, then
\begin{align*}
\mathbb{E}_{\{(\boldsymbol{X}_i,Y_i)\}_{i=1}^n}\left\|\mathcal{T}_{B_n}\circ\widehat{f}_n - f_0\right\|_{L^2(\mu)}^2\leq C(d,\boldsymbol{\rho},M,\eta)\frac{e^{C(\log n)^{p}}(\log n)^{7-p}(\log\log n)^4}{n}.
\end{align*}

\end{enumerate}
\end{theorem}
\begin{proof}
Since the truncation function $\mathcal{T}_{B_n}$ is non-decreasing, we are able to apply Lemma \ref{nondecreasing}  to derive an upper bound for the pseudo-dimension of $\mathcal{T}_{B_n}\circ \mathcal{F}_{NN}$:
\begin{align*}
\mathrm{Pdim}(\mathcal{T}_{B_n}\circ \mathcal{F}_{NN})\leq\mathrm{Pdim}(\mathcal{F}_{NN})\leq C(d)\bar{L}^2\bar{N}^2\log\left(\bar{L}\bar{N}\right),
\end{align*}
where the second step is due to Proposition \ref{Pdim}. With this pseudo-dimension estimate, we can bound the entropy in Proposition \ref{error decomposition} by employing Lemma \ref{covering}: 
\begin{align*}
&\log  \mathcal{N} \left( n^{-1}B_n^{-1}, \mathcal{T}_{B_n}\circ \mathcal{F}_{NN}, \|\cdot\|_{1},n\right)\\
&\leq\log e+\log\left(\mathrm{Pdim}(\mathcal{T}_{B_n}\circ \mathcal{F}_{NN})+1\right)+\mathrm{Pdim}(\mathcal{T}_{B_n}\circ \mathcal{F}_{NN})\log\left( {4neB_n^2}\right)\\
&\leq C(d)\bar{L}^2\bar{N}^2\log\left(\bar{L}\bar{N}\right)\log n.
\end{align*}
Plugging this estimate into Proposition \ref{error decomposition}, we obtain
\begin{align*}
\mathbb{E}_{\{(\boldsymbol{X}_i,Y_i)\}_{i=1}^n}\left\|\mathcal{T}_{B_n}\circ\widehat{f}_n - f_0\right\|_{L^2(\mu)}^2 
\leq \frac{C(d,M,\eta)(\log n)^3 \bar{L}^2\bar{N}^2\log\left(\bar{L}\bar{N}\right)}{n}+2\inf_{f \in \mathcal{F}_{NN}}\left\|{f} - f_0\right\|_{L^2(\mu)}^2.
\end{align*}
If 
\begin{align}\label{regression0}
\inf_{f \in \mathcal{F}_{NN}}  \left\|{f} - f_0\right\|_{L^2(\mu)}^2
&\leq C(d,\boldsymbol{\rho},M)n^{-1},
\end{align}
then we have
\begin{align}\label{regression1}
\mathbb{E}_{\{(\boldsymbol{X}_i,Y_i)\}_{i=1}^n}\left\|\mathcal{T}_{B_n}\circ\widehat{f}_n - f_0\right\|_{L^2(\mu)}^2 
\leq \frac{C(d,M,\eta)(\log n)^3 \bar{L}^2\bar{N}^2\log\left(\bar{L}\bar{N}\right)}{n}+C(d,\boldsymbol{\rho},M)n^{-1}.
\end{align}
In the following, we select three combinations of $(N, L)$ in Theorem \ref{approximation} and prove via Theorem \ref{approximation} that \eqref{regression0} holds under each setting, which correspond respectively to the three cases of this theorem.

\begin{enumerate}[label=(\Roman*)]

\item 
Set 
\begin{align*}
N\asymp C,\quad L\asymp(\log n)^{d+3}
\end{align*}
in Theorem \ref{approximation}, then it falls into the regime of $\kappa=0,\beta=1$ and hence there exists a neural network $f_{NN}:\mathbb{R}^d\to\mathbb{R}$ with width 
\begin{align}\label{regression211}
\bar{N}=C(d,\boldsymbol{\rho})N\leq C(d,\boldsymbol{\rho}) 
\end{align}
and depth 
\begin{align}\label{regression212}
\bar{L}=C(d,\boldsymbol{\rho})L\leq C(d,\boldsymbol{\rho})(\log n)^{d+3}    
\end{align}
such that for any $\boldsymbol{x}\in\left[-1,1\right]^d$, there holds
\begin{align*}
\left|f_{NN}(\boldsymbol{x})-f_0(\boldsymbol{x}) \right|
&\leq C(d,\boldsymbol{\rho},M)N^{-C(d,\boldsymbol{\rho})L^{\frac{1}{d+2}}}\\
&\leq C(d,\boldsymbol{\rho},M)\cdot C^{-C(d,\boldsymbol{\rho})(\log n)^{\frac{d+3}{d+2}}}\leq C(d,\boldsymbol{\rho},M)n^{-1/2},
\end{align*}
which implies \eqref{regression0}. Plugging \eqref{regression211}\eqref{regression212} into \eqref{regression1} yields Case (I).

\item 
Set
\begin{align*}
N\asymp (\log n)^{p},\quad L\asymp\left(\frac{\log n}{\log\log n}\right)^{d+3-p}
\end{align*}
in Theorem \ref{approximation} with $p\in[1,d+1]$, then 
it falls into the regime of 
\begin{align*}
\kappa=\frac{2p-1}{2(d+3-p)}\in\left[\frac{1}{2(d+2)},\frac{2d+1}{4}\right],\quad\beta=1    
\end{align*}
and hence there exists a neural network $f_{NN}:\mathbb{R}^d\to\mathbb{R}$ with width 
\begin{align}\label{regression281}
\bar{N}=C(d,\boldsymbol{\rho})N\leq C(d,\boldsymbol{\rho})(\log n)^{p} 
\end{align}
and depth 
\begin{align}\label{regression282}
\bar{L}=C(d,\boldsymbol{\rho})L\leq C(d,\boldsymbol{\rho})\left(\frac{\log n}{\log\log n}\right)^{d+3-p}
\end{align}
such that for any $\boldsymbol{x}\in\left[-1,1\right]^d$, there holds
\begin{align*}
\left|f_{NN}(\boldsymbol{x})-f_0(\boldsymbol{x}) \right|
&\leq C(d,\boldsymbol{\rho},M)N^{-C(d,\boldsymbol{\rho},p,\alpha)L^{\tau(\kappa)}}
\leq C(d,\boldsymbol{\rho},M)N^{-C(d,\boldsymbol{\rho},p,\alpha)L^{\frac{\kappa+1}{d+2}}}\\
&\leq C(d,\boldsymbol{\rho},M)(\log n)^{-C(d,\boldsymbol{\rho},p,\alpha)\left(\frac{\log n}{\log\log n}\right)^{\frac{2d+5}{2d+4}}}\leq C(d,\boldsymbol{\rho},M)n^{-1/2},
\end{align*}
which implies \eqref{regression0}. Plugging \eqref{regression281}\eqref{regression282} into \eqref{regression1} yields Case (II).

\item 
Set 
\begin{align*}
N\asymp e^{C(\log n)^{p}},\quad L\asymp\left(\log n\right)^{2-p}
\end{align*}
in Theorem \ref{approximation} with $p\in(0,1)$, then 
it falls into the regime of $\kappa=d,\beta=1$ and hence there exists a neural network $f_{NN}:\mathbb{R}^d\to\mathbb{R}$ with width 
\begin{align}\label{regression241}
\bar{N}=C(d,\boldsymbol{\rho})N\leq C(d,\boldsymbol{\rho})e^{C(\log n)^{p}} 
\end{align}
and depth 
\begin{align}\label{regression242}
\bar{L}=C(d,\boldsymbol{\rho})L(\log L)^2\leq C(d,\boldsymbol{\rho})(\log n)^{2-p}(\log\log n)^2    
\end{align}
such that for any $\boldsymbol{x}\in\left[-1,1\right]^d$, there holds
\begin{align*}
\left|f_{NN}(\boldsymbol{x})-f_0(\boldsymbol{x}) \right|
&\leq C(d,\boldsymbol{\rho},M)N^{-C(d,\boldsymbol{\rho},\alpha)L}
\leq C(d,\boldsymbol{\rho},M)e^{-C(d,\boldsymbol{\rho},\alpha)(\log n)^p(\log n)^{2-p}}\\
&=C(d,\boldsymbol{\rho},M)e^{-C(d,\boldsymbol{\rho},\alpha)(\log n)^2}\leq C(d,\boldsymbol{\rho},M)n^{-1/2},
\end{align*}
which implies \eqref{regression0}. Plugging \eqref{regression241}\eqref{regression242} into \eqref{regression1} yields Case (III).

\end{enumerate}

\end{proof}

It can be observed that, despite both employing ReLU FNNs estimators, the convergence rate for regression with analytic targets is much faster than that for function classes with finite smoothness $s$ (such as Hölder, Sobolev, and Besov spaces), whose typical convergence rate is $\mathcal{O}\left(n^{-\frac{2s}{2s+d}}\right)$ \cite{suzukiadaptivity,schmidt2020nonparametric,nakada2020adaptive,kohler2021rate,chen2022nonparametric,jiao2023deepb,yang2025optimal}. This is precisely due to the infinite differentiability of analytic functions. From an architectural perspective, the network size required in our setting is also much smaller than those required in the finite-smoothness setting. Taking \cite{yang2025optimal} as an example, to achieve the $\mathcal{O}\left(n^{-\frac{2s}{2s+d}}\right)$ convergence rate, the product of the depth $\bar{L}$ and width $\bar{N}$ must scale at a polynomial rate  $\mathcal{O}\left(n^{\frac{d}{2d+4s}}\right)$. In contrast, 
Theorem \ref{regression ub} shows that both the depth and width need only scale sub-polynomially with respect to $n$ to guarantee a convergence rate of $\widetilde{\mathcal{O}}\left(n^{-1}\right)$ in our setting.

In nonparametric estimation, to evaluate the sharpness of derived upper bounds, it is standard practice to compare them with the minimax rate \cite{stone1982optimal,ibragimov2013statistical,wainwright2019high}. The minimax rate for nonparametric regression with analytic targets is presented in the following theorem, which demonstrates that our upper bounds established in Theorem \ref{regression ub} are nearly optimal.

\begin{theorem}\label{regression minimax}

Let $\widehat{f}$ be any estimator of $f_0$ based on the samples $\{(\boldsymbol{X}_i,Y_i)\}_{i=1}^n$. There holds
\begin{align*}
\inf_{\widehat{f}}\sup_{f_0\in\mathcal{A}(\boldsymbol{\rho},M)}\mathbb{E}_{\{(\boldsymbol{X}_i,Y_i)\}_{i=1}^n}\left\|\widehat{f}-f_0\right\|_{L^2(\mu)}^2\asymp\frac{(\log n)^d}{n}.
\end{align*}
\end{theorem}
\begin{proof}
The upper bound is derived in \cite{ibragimov1998analytic}. Hence it suffices to present a proof of the lower bound.
To this end, we are going to contruct a $\delta$-packing $\{f_{\boldsymbol{\alpha}}\}_{\boldsymbol{\alpha}\in\Omega}$ in $\mathcal{A}(\boldsymbol{\rho},M)$ and apply Fano's inequality (the readers are referred to \cite{wainwright2019high,Tsybakov} for more details):
\begin{align}\label{lb1}
\inf_{\widehat{f}}\sup_{f_0\in\mathcal{A}(\boldsymbol{\rho},M)}\mathbb{E}_{\{(\boldsymbol{X}_i,Y_i)\}_{i=1}^n}\left\|\widehat{f}-f_0\right\|_{L^2(\mu)}^2\geq\delta^2\left(1-\frac{\frac{1}{\left|\Omega\right|^2}\sum_{\boldsymbol{\alpha},\boldsymbol{\alpha}'\in\Omega}D_{\text{KL}}(\mathbb{P}_{\boldsymbol{\alpha}} \| \mathbb{P}_{\boldsymbol{\alpha}'})+\ln 2}{\ln\left|\Omega\right|}\right).
\end{align}
Define
\begin{align*}
f_{\boldsymbol{\alpha}}(\boldsymbol{x}):=\gamma\sum_{\boldsymbol{\nu}\in\{0,1,\cdots,K-1\}^d}\alpha_{\boldsymbol{\nu}}L_{\boldsymbol{\nu}}(\boldsymbol{x}),
\end{align*}
where parameters $\gamma\in\mathbb{R}_{>0}$ and $K\in\mathbb{N}_{\geq1}$ will be determined later, $\alpha_{\boldsymbol{\nu}}\in\{-1,1\}$. For convenience, we view $\boldsymbol{\alpha}$ as an $ K^d$-dimensional vector. Due to the orthogonality of Legendre polynomials $\{L_{\boldsymbol{\nu}}(\boldsymbol{x})\}$,
\begin{align}\label{lb2}
\left\|f_{\boldsymbol{\alpha}}-f_{\boldsymbol{\alpha}'}\right\|_{L^2}^2
=2^d\gamma^2\sum_{\boldsymbol{\nu}\in\{0,1,\cdots,K-1\}^d}\left|\alpha_{\boldsymbol{\nu}}-\alpha_{\boldsymbol{\nu}}'\right|^2
=2^{d+2}\gamma^2\sum_{\boldsymbol{\nu}\in\{0,1,\cdots,K-1\}^d}\mathbbm{1}_{\{\alpha_{\boldsymbol{\nu}}\neq\alpha_{\boldsymbol{\nu}}'\}},
\end{align}
which implies an upper bound
\begin{align*}
\left\|f_{\boldsymbol{\alpha}}-f_{\boldsymbol{\alpha}'}\right\|_{L^2}^2\leq2^{d+2}\gamma^2K^d.
\end{align*}
It follows that
\begin{align}\label{lb3}
D_{\text{KL}}(\mathbb{P}_{\boldsymbol{\alpha}} \| \mathbb{P}_{\boldsymbol{\alpha}'})
=\frac{n}{2\eta^2}\left\|f_{\boldsymbol{\alpha}}-f_{\boldsymbol{\alpha}'}\right\|_{L^2(\mu)}^2
\leq\frac{n\overline{p}}{2\eta^2}\left\|f_{\boldsymbol{\alpha}}-f_{\boldsymbol{\alpha}'}\right\|_{L^2}^2\leq\frac{2^{d+1}n\overline{p}\gamma^2K^d}{\eta^2}.
\end{align}
From the well-known Varshamov–Gilbert bound (see, for example, \cite[Lemma 2.9]{Tsybakov}), there exists $\Omega\subset\{-1,1\}^{K^d}$ such that 
\begin{align}\label{lb4}
|\Omega|\geq2^{K^d/8}
\end{align}
and for any $\boldsymbol{\alpha},\boldsymbol{\alpha}'\in\Omega$ with $\boldsymbol{\alpha}\neq\boldsymbol{\alpha}'$, there holds
\begin{align*}
\sum_{\boldsymbol{\nu}\in\{0,1,\cdots,K-1\}^d}\mathbbm{1}_{\{\alpha_{\boldsymbol{\nu}}\neq\alpha_{\boldsymbol{\nu}}'\}}\geq\frac{K^d}{8}.
\end{align*}
Therefore, we also obtain a lower bound from \eqref{lb2}:
\begin{align}\label{lb5}
\left\|f_{\boldsymbol{\alpha}}-f_{\boldsymbol{\alpha}'}\right\|_{L^2(\mu)}^2\geq\underline{p}\left\|f_{\boldsymbol{\alpha}}-f_{\boldsymbol{\alpha}'}\right\|_{L^2}^2
\geq2^{d-1}\gamma^2\underline{p}K^d:=\delta^2.
\end{align}
Combining \eqref{lb3}\eqref{lb4} and choosing
\begin{align}\label{lb6}
\gamma\asymp\frac{1}{\sqrt{n}},
\end{align}
we can make the term inside the parentheses in \eqref{lb1} greater than $\frac{1}{2}$. Therefore,
\begin{align*}
\inf_{\widehat{f}}\sup_{f_0\in\mathcal{A}(\boldsymbol{\rho},M)}\mathbb{E}_{\{(\boldsymbol{X}_i,Y_i)\}_{i=1}^n}\left\|\widehat{f}-f_0\right\|_{L^2(\mu)}^2\geq\frac{1}{2}\delta^2.
\end{align*}
From Lemma \ref{coefficient1.5}, we know that to ensure $f_{\boldsymbol{\alpha}}$ lies in $\mathcal{A}(\boldsymbol{\rho},M)$, it suffices to let $\gamma\lesssim {\left(\prod_{i=1}^{d}\rho_i\right)^{-K}}$,
which is satisfied when we choose
\begin{align}\label{lb7}
K\asymp\log n   
\end{align}
in view of \eqref{lb6}. Plugging \eqref{lb6}\eqref{lb7} into \eqref{lb5}, we derive that $\delta\asymp\frac{(\log n)^d}{n}$, hence completing the proof.

\end{proof}

\section{Conclusions}\label{Conclusions}

In this work, we derive upper bounds for the approximation of analytic functions by ReLU networks under the $(N,L)$-characterization. Distinct from previous studies on functions with finite smoothness, our findings reveal that depth plays a more critical role than width in the context of analytic function approximation, further highlighting the advantage of the $(N,L)$-characterization. To trading off the smoothness parameters and the approximation accuracy, we employ refined technical constructions of several ReLU networks to approximate power functions, multivariate multiplication, and polynomials; these intermediary results may be of independent interest. We also establish upper bounds on the convergence rate of a ReLU network estimator in nonparametric regression, which are proven nearly optimal.

In the following, we outline several open problems that warrant future research.
First, the exponent parameter $\tau$ derived in this work is a three-piece piecewise linear function of $\kappa$. In the future, we will attempt to improve it to $\tau(\kappa)=\frac{\kappa+1}{d+1}$ over the entire region, which means that $\tau$ would become a single straight line rather than a piecewise linear function. The key lies in the case of $\lambda<1$ in Lemma \ref{final}: if the quantity $\lambda\vee 1$ in Lemma \ref{final} could be improved to $\lambda$, then Proposition \ref{II} might be extended to hold for the regime $\kappa\in[0,d]$ with $\tau=\frac{\kappa+1}{d+1}$, in which case Proposition \ref{II} would completely cover Proposition \ref{I}.
Second, while we have demonstrated that the approximation upper and lower bounds are nearly aligned in the regime $\kappa=d$, the optimality of the upper bound remains unknown for the regime $\kappa<d$.
Third, as discussed in Section \ref{Main Results}, the applicable region of our results covers almost the entire $L$-$N$ plane, leaving only a minor gap uncovered. To bridge this gap, a promising direction for future work is to investigate the construction of networks with fixed depth and varying width. A related challenge is that a constant network depth forces the width to scale at a polynomial rate of the sample size $n$ in nonparametric estimation. Under our present technical framework, this means that the methodology employed in the current paper cannot yield an optimal minimax convergence rate in this case.

\bibliographystyle{plain}
\bibliography{ref}

\end{document}